\documentclass[journal,12pt,onecolumn,draftclsnofoot,]{IEEEtran}
\usepackage[a4paper, total={8.5in, 11in}, margin = 1in]{geometry}
\usepackage{amsmath}
\usepackage{graphicx}
\usepackage{caption2}
\usepackage{amsthm}
\usepackage{float}
\usepackage{mathrsfs}
\usepackage{verbatim}
\usepackage{epstopdf}
\usepackage{amssymb}
\usepackage{amsfonts}
\usepackage{subfigure}
\usepackage{color}
\usepackage{cite}
\usepackage{cancel}
\usepackage[shortlabels]{enumitem}
\usepackage[breaklinks=true,letterpaper=true,colorlinks=false,bookmarks=false]{hyperref}
\usepackage{algorithm}
\usepackage{algpseudocode}
\usepackage{arydshln}
\usepackage[official]{eurosym}
\usepackage{comment}
\usepackage{amsmath,amssymb,amsthm,mathrsfs,amsfonts,dsfont}
\usepackage[shortlabels]{enumitem}
\DeclareMathOperator*{\argmin}{arg\,min}
\DeclareMathOperator*{\argmax}{arg\,max}

\newtheorem{theorem}{Theorem}
\newtheorem{proposition}{Proposition}
\newtheorem{remark}{Remark}

\newtheorem{lemma}{Lemma}

\newtheorem{definition}{Definition}
\begin{document}
\title{Privacy Against Inference Attacks in Vertical Federated Learning}
\author{Borzoo Rassouli$^1$, Morteza Varasteh$^1$ and Deniz G\"und\"uz$^2$\\
\small{$^1$ School of Computer Science and Electronic Engineering, University of Essex, Colchester CO4 3SQ, UK}\\
\small{$^2$ Department of Electrical and Electronic Engineering, Imperial College London, , London SW7 2AZ, UK}\\
{\tt\small \{b.rassouli,m.varasteh\}@essex.ac.uk}, {\tt\small  d.gunduz@imperial.ac.uk}
% \thanks{
% A conference version of this paper is provided in \cite{RG18}.
% }
}
\maketitle

\begin{abstract}
    Vertical federated learning is considered, where an active party, having access to true class labels, wishes to build a classification model by utilizing more features from a passive party, which has no access to the labels, to improve the model accuracy. In the prediction phase, with logistic regression as the classification model, several inference attack techniques are proposed that the adversary, i.e., the active party, can employ to reconstruct the passive party's features, regarded as sensitive information. These attacks, which are mainly based on a classical notion of the center of a set, i.e., the Chebyshev center, are shown to be superior to those proposed in the literature. Moreover, several theoretical performance guarantees are provided for the aforementioned attacks. Subsequently, we consider the minimum amount of information that the adversary needs to fully reconstruct the passive party's features. In particular, it is shown that when the passive party holds one feature, and the adversary is only aware of the signs of the parameters involved, it can perfectly reconstruct that feature when the number of predictions is large enough. Next, as a defense mechanism, several privacy-preserving schemes are proposed that worsen the adversary's reconstruction attacks, while preserving the benefits that VFL brings to the active party. Finally, experimental results demonstrate the effectiveness of the proposed attacks and the privacy-preserving schemes.
\end{abstract}
\section{introduction}
% Emergence of distributed machine learning (ML) techniques has recently revolutionized the way useful information is extracted from raw data in different areas, such as computer vision, image recognition, financial services and natural language processing. There is a huge request in many sectors for utilizing distributed data to build more accurate and sophisticated ML models. Many distributed ML techniques have been proposed to deal with various concerns, such as centralized data storage, overloaded computations, security and privacy of distributed ML models \cite{Wei_Fed_DP}.
To tackle the concerns in the traditional centralized learning, i.e., privacy, storage, and computational complexity, Federated Learning (FL) has been proposed in \cite{McMahan_2017} where machine learning (ML) models are jointly trained by multiple local data owners (i.e., parties), such as smart phones, data centres, etc., without revealing their private data to each other. This approach has gained interest in many real-life applications, such as health systems \cite{Songtao_health, Wenqi_health}, keyboard prediction \cite{Francoise_Keyboard, Andrew_keyboard_prediction}, and e-commerce \cite{Kai_2019, Wang_2020}.

% Federated learning (FL), has recently been proposed  to enable (centrally coordinated) training of ML models via local devices (smart phones, smart watches) and/or data centres (banks, hospitals) without having to gather local data in a centralized centre \cite{Chuan_2019}.
% This approach has alleviated some of the concerns regarding data storage, computation overload, privacy leakage; and it

Based on how data is partitioned among participating parties, three variants of FL, which are horizontal, vertical and transfer FL, have been considered. Horizontal FL (HFL) refers to the FL among data owners that share different data records/samples with the same set of features \cite{Rong}, and vertical FL (VFL) is the FL in which parties share common data samples with disjoint set of features \cite{Cheng_2020}.

% its applicability is limited due to confidentiality and competition among large data centres \cite{Cheng_2020}. Vertical FL (VFL) however, avoids this issue by enabling model training among non-competing data centres (such as a bank and an insurance company), where they
Figure \ref{fig122} illustrates a digital banking system as an example of a VFL setting, in which two parties are participating, namely, a bank and a FinTech company \cite{Xinjian}. The bank wishes to build a binary classification model to approve/disapprove a user's credit card application by utilizing more features from the Fintech company. In this context, only the bank has access to the class labels in the \textit{training} and \textit{testing} datasets, hence named the \textit{active party}, and the FinTech company that is unaware of the labels is referred to as the \textit{passive party}.

\begin{figure}[ht]
 \centering % centering figure
 \scalebox{1.2} % rescale the figure by a factor of 0.8
 {\includegraphics{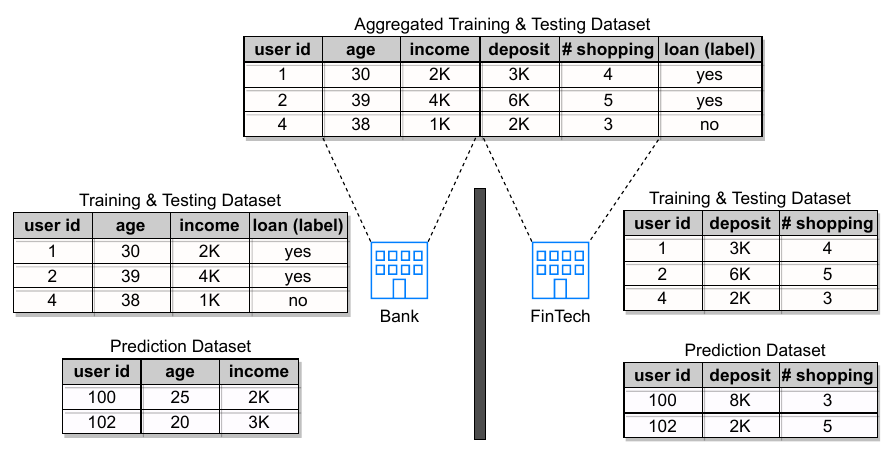}} % importing figure
 \caption{Digital banking as an example of vertical federated learning \cite{Xinjian}.}
 \label{fig122} % labeling to refer it inside the text
\end{figure}

Once the model is trained, it can be used to predict the decision (approve/disapprove) on a new credit card application in the \textit{prediction} dataset. The model outputs, referred to as the prediction outputs or more specifically, the \textit{confidence scores}, are revealed to the active party that is in charge of making decision. The active party maybe aware or unaware of the passive party's model parameters, which are, respectively, referred to as the \textit{white-box} and \textit{black-box} settings.

As stated in \cite{Xinjian}, upon the receipt of the prediction outputs, which generally depend on the passive party's features, a curious active party can perform \textit{reconstruction attacks} to infer the latter, which are regarded as passive the party's sensitive information \footnote{The active party is also referred to as the \textit{adversary} in this paper.}. This privacy leakage in the prediction phase of VFL is the main focus of this paper, and the following contributions are made.
\begin{itemize}
    \item In the white-box setting, several reconstruction attacks are proposed that outperform those given in \cite{Xinjian},\cite{Jiang}. The attacks are motivated by the notion of the Chebyshev center of a convex polytope.
    \item Theorems 1 and 2 provide theoretical bounds as rigorous guarantees for some of these attacks.
    \item In the black-box setting, it is shown that when the passive party holds one feature, and the adversary is aware of the signs of the parameters involved, it can still fully reconstruct the passive party's feature given that the number of predictions is large enough.
    \item Several privacy-preserving schemes are proposed as a defense technique against reconstruction attacks, which have the advantage of not degrading the benefits that VFL brings to the active party.
\end{itemize}

% One party, namely, the \textit{active party} has a database of records with features and their corresponding lables. The other parties, namely, passive parties, have a subset of the same database of records with no lables and with different sets of features.

% This paper is a comprehensive study of the works in \cite{Jiang, Xinjian}, and investigates the privacy leakage during the prediction phase when two parties (and a CA) are involved in VFL. Considering logistic regression (LR) models, we focus on both white box and black box settings. Our contribution is summarized as below:
The organization of the paper is as follows. In section \ref{sm}, an explanation of the system model under consideration is provided. In section \ref{ps}, the elementary steps that pave the way for the adversary's attack are elaborated, and the measure by which the performance of the reconstruction attack is evaluated is provided. The analysis and derivations in this paper need some preliminaries from linear algebra and optimization. To make the text as self contained as possible, these have been provided in section \ref{prel}. The main results of the paper are given in sections \ref{wbs} to \ref{expr}. In section \ref{wbs}, the white-box setting is considered and several attack methods are proposed and evaluated analytically. Section \ref{bbs} deals with the black-box setting and investigates the minimum knowledge the adversary needs to perform a successful attack. In section \ref{PR}, a privacy-preserving scheme is provided that worsens the adversary's attacks, while not altering the confidence scores revealed to it. Section \ref{expr} is devoted to the experimental evaluation of the results of this paper and comparison to those in the literature. Finally, section \ref{conc} concludes the paper. To improve the readability, the notation used in this paper is provided next.

\textbf{Notation.}
%Random variables are denoted by capital letters, their realizations by lower case letters, and their alphabets by capital letters in calligraphic font.
Matrices and vectors\footnote{All the vectors considered in this paper are column vectors.} are denoted by bold capital (e.g. $\mathbf{A,Q}$) and bold lower case letters (e.g. $\mathbf{b,z}$), respectively. Random variables are denoted by capital letters (e.g. $X$), and their realizations by lower case letters (e.g. $x$). \footnote{In order to prevent confusion between the notation of a matrix (bold capital) and a random vector (bold capital), in this paper, letters $x,y$ are not used to denote a matrix, hence, $\mathbf{X,Y}$ are random vectors rather than matrices.}Sets are denoted by capital letters in calligraphic font (e.g. $\mathcal{X},\mathcal{G}$) with the exception of the set of real numbers, i.e., $\mathds{R}$.  The cardinality of the finite set $\mathcal{X}$ is denoted by $|\mathcal{X}|$. For a matrix $\mathbf{A}_{m\times k}$, the null space, rank, and nullity are denoted by $\textnormal{Null}(\mathbf{A})$, $\textnormal{rank}(\mathbf{A})$, and $\textnormal{nul}(\mathbf{A})$, respectively, with $\textnormal{rank}(\mathbf{A})+\textnormal{nul}(\mathbf{A})=k$, i.e., the number of columns. The transpose of $\mathbf{A}$ is denoted by $\mathbf{A}^T$, and when $m=k$, its trace and determinant are denoted by $\textnormal{Tr}(\mathbf{A})$ and $\textnormal{det}(\mathbf{A})$, respectively. For an integer $n\geq 1$, the terms $\mathbf{I}_n$, $\mathbf{1}_n$, and $\mathbf{0}_n$ denote the $n$-by-$n$ identity matrix, the $n$-dimensional all-one, and all-zero column vectors, respectively, and whenever it is clear from the context, their subscripts are dropped. For two vectors $\mathbf{a,b}$, $\mathbf{a}\geq\mathbf{b}$ means that each element of $\mathbf{a}$ is greater than or equal to the corresponding element of $\mathbf{b}$. The notation $\mathbf{A}\succeq 0$ is used to show that $\mathbf{A}$ is positive semi-definite, and  $\mathbf{A}\succeq\mathbf{B}$ is equivalent to $\mathbf{A}-\mathbf{B}\succeq 0$.
For integers $m\leq n$, we have the discrete interval $[m:n]\triangleq\{m, m+1,\ldots,n\}$, and the set $[1:n]$ is written in short as $[n]$.
$F_{X}(\cdot)$ denotes the cumulative distribution function (CDF) of random variable $X$, whose expectation is denoted by $\mathds{E}[X]=\int xdF_{X}(x)$. In this paper, all the (in)equalities that involve a random variable are in the \textit{almost surely} (a.s.) sense, i.e., they happen with probability 1. For $\mathbf{x}\in\mathbb{R}^n$ and $p\in[1,\infty]$, the $L^p$-norm is defined as $\|\mathbf{x}\|_p\triangleq(\sum_{i=1}^n|x_i|^p)^{\frac{1}{p}},p\in[1,\infty)$, and $\|\mathbf{x}\|_\infty\triangleq\max_{i\in[n]}|x_i|$. Throughout the paper, $\|\cdot\|$ (i.e., without subscript) refers to the $L^2$-norm. The nuclear norm of matrix $\mathbf{A}$ is denoted by $\|\mathbf{A}\|_*\triangleq\textnormal{Tr}(\sqrt{\mathbf{A}^T\mathbf{A}})$, which is equal to the sum of its singular values. Let $p,q$ be two arbitrary pmfs on $\mathcal{X}$. The Kullback–Leibler divergence from $q$ to $p$ is defined as\footnote{We assume that $p$ is absolutely continuous with respect to $q$, i.e., $q(x)=0$ implies $p(x)=0$, otherwise, $D(p||q)\triangleq\infty$.} $D(p||q)\triangleq\sum_xp(x)\log_2(\frac{p(x)}{q(x)})$, which is also shown as $D(\mathbf{p}||\mathbf{q})$ with $\mathbf{p},\mathbf{q}$ being the corresponding probability vectors of $p,q$, respectively. Likewise, the cross entropy of $q$ relative to $p$ is given by $H(p,q)\triangleq -\sum_{x}p(x)\log q(x)$, which is also shown as $H(\mathbf{p},\mathbf{q})$. Finally, the total variation distance between two probability vectors is $d_\textnormal{TV}(\mathbf{p},\mathbf{q})\triangleq\frac{1}{2}\|\mathbf{p}-\mathbf{q}\|_1$.
%For $0\leq t\leq 1$, $H_b(t)\triangleq-t\log t-(1-t)\log (1-t)$ denotes the binary entropy function. The unit-step function is denoted by $s(\cdot)$, and $\lfloor\cdot\rfloor$ denotes the floor operator.
% Throughout the paper, for a random variable $Y$ with the corresponding probability vector $\mathbf{p}_Y$, $H(Y)$ and $H(\mathbf{p}_Y)$ are written interchangeably, and so are the quantities $D(p_Y(\cdot)||q_Y(\cdot))$ and $D(\mathbf{p}_Y||\mathbf{q}_Y)$. {\color{blue}All the logarithms in this paper are to the base of 2.} Given two positive integers $a,b$, $a$ modulo $b$ is abbreviated as $a\textnormal{ mod }b$. Finally, $d_{\textnormal{TV}}$, $\lfloor\cdot\rfloor$, and $\lceil\cdot\rceil$ denote the total variation distance, the floor, and the ceiling operators, respectively.
\section{System model}\label{sm}
\subsection{Machine learning (ML)}
An ML model is a function $f_\mathbf{\theta}:\mathcal{X}\to\mathcal{Y}$ parameterized by the vector $\mathbf{\theta}$, where $\mathcal{X}$ and $\mathcal{Y}$ denote the input and output spaces, respectively. Supervised classification is considered in this paper, where a labeled training dataset is used to train the model.

Assume that a training dataset $\mathcal{D}_{\textnormal{train}}\triangleq\{(\mathbf{x}_i,y_i)|i\in[n]\}$ is given, where each $\mathbf{x}_i$ is a $d$-dimensional example/sample and $y_i$ denotes its corresponding label. Learning refers to the process of obtaining the parameter vector $\mathbf{\theta}$ in the minimization of a loss function, i.e.,
\begin{equation}
    \min_{\mathbf{\theta}}\frac{1}{n}\sum_{i=1}^n l(f_\mathbf{\theta}(\mathbf{x}_i),y_i)+\omega(\mathbf{\theta}),
\end{equation}
where $l(\cdot,\cdot)$ measures the loss of predicting $f_\mathbf{\theta}(\mathbf{x}_i)$, while the true label is $y_i, i\in[n]$. A regularization term $\omega(\mathbf{\theta})$ can be added to the optimization to avoid overfitting.

Once the model is trained, i.e., $\mathbf{\theta}$ is obtained, it can be used for the prediction of any new sample. In practice, the prediction is (probability) vector-valued, i.e., it is a vector of confidence scores as $\mathbf{c}=(c_1,c_2,\ldots,c_k)^T$ with $\sum_ic_i=1,c_i\geq 0,i\in[k]$, where $c_i$ denotes the probability that the sample belongs to class $i$, and $k$ denotes the number of classes. Classification can be done by choosing the class that has the highest confidence score.

In this paper, we focus on logistic regression (LR), which can be modelled as
\begin{align}\label{confi}
    \mathbf{c} = \sigma(\mathbf{Wx} + \mathbf{b}),
\end{align}
where $\mathbf{W}$ and $\mathbf{b}$ are the parameters collectively denoted as $\mathbf{\theta}$, and $\sigma(\cdot)$ is the sigmoid or softmax function in the case of binary or multi-class classification, respectively.
%$\mathbf{c}$ is a $k$-dimensional vector representing the confidence score of each class for a particular $\mathbf{x}$

% \subsubsection{Neural Networks}
% An NN can be considered as a composition function with each function representing a layer. Generally, each layer is composed of a nonlinear function (such as tanh, Relu, Pooling) applied on an affine transformation of its input. The overall mapping can be considered as
% \begin{align}\label{NN}
%     \mathbf{c} = \sigma(\mathcal{F}(\mathbf{Wx}+\mathbf{b})),
% \end{align}
% where $\mathcal{F}(\cdot)$ is the overall composition function formed by NN. $\sigma(\cdot)$ is defined similarly as in LR.

\subsection{Vertical Federated Learning}
VFL is a type of ML model training approach in which two or more parties are involved in the training process, such that they hold the same set of samples with disjoint set of features. The main goal in VFL is to train a model in a privacy-preserving manner, i.e., to collaboratively train a model without each party having access to other parties' features. Typically, the training involves a trusted third party known as the coordinator authority (CA), and it is commonly assumed that only one party has access to the label information in the training and testing datasets. This party is named \textit{active} and the remaining parties are called \textit{
passive}. Throughout this paper, we assume that only two parties are involved; one is active and the other is passive. The active party is assumed to be \textit{honest but curious}, i.e., it obeys the protocols exactly, but may try to infer passive party's features based on the information received. As a result, the active party is referred to as the \textit{adversary} in this paper.

In the existing VFL frameworks, CA's main task is to coordinate the learning process once it has been initiated by the active party. During the training,
% each party forward-propagates their corresponding parameters to the last layer of their model, which is an intermediate layer of a larger model. This stage is performed locally at each party. Next,
CA receives the intermediate model updates from each party, and after a set of computations, backpropagates each party's gradient updates, separately and securely. To meet the privacy requirements of parties' datasets, cryptographic techniques such as secure multi-party computation (SMC) \cite{Andrew_SMC} or homomorphic encryption (HE) \cite{HE_IVAN} are used.

Once the global model is trained, upon the request of the active party for a new record prediction, each party computes the results of their model using their own features. CA aggregates these results from all the parties, obtains the prediction (confidence scores), and delivers that to the active party for further action.

As in \cite{Xinjian}, we assume that the active party has no information about the underlying distribution of the passive party's features. However, it is assumed that the knowledge about the name, types and range of the features is available to the active party to decide whether to participate in a VFL or not.

\section{Problem statement}\label{ps}
Let $(\mathbf{Y}^T,\mathbf{X}^T)^T$ denote a random $d_t$-dimensional input sample for prediction, where the $(d_t-d)$-dimensional $\mathbf{Y}$ and the $d$-dimensional $\mathbf{X}$ correspond to the feature values held by the active and passive parties, respectively. The VFL model under consideration is LR, where the confidence score is given by $\mathbf{c}=\sigma(\mathbf{z})$ with $\mathbf{z}=\bold{W}_{act} \mathbf{Y}+\bold{W}_{pas} \mathbf{X}+\mathbf{b}$. Denoting the number of classes in the classification task by $k$, $\bold{W}_{act}$ (with dimension $k\times (d_t-d)$) and $\bold{W}_{pas}$ (with dimension $k\times d$) are the model parameters of the active and passive parties, respectively, and $\mathbf{b}$ is the $k$-dimensional bias vector. From the definition of $\sigma(\cdot)$, we have
\begin{align}\label{qe1}
    \ln \frac{c_{m+1}}{c_m} = z_{m+1}-z_m,\ m\in[k-1],
\end{align}
where $c_m,z_m$ denote the $m$-th element of $\mathbf{c},\mathbf{z}$, respectively. Define $\mathbf{J}$ as
\begin{equation}\label{JJ}
    \mathbf{J}\triangleq \begin{bmatrix}
    -1 & 1 & 0 & 0 & \ldots & 0\\
    0 & -1 & 1 & 0 & \ldots & 0\\
    0 & 0 & -1 & 1 & \ldots & 0\\
    \vdots & \vdots & \vdots & \vdots & \ddots &\vdots \\
    0 & \ldots & \ldots & 0 & -1 & 1
    \end{bmatrix}_{(k-1)\times k},
\end{equation}
whose rows are cyclic permutations of the first row with offset equal to the row index$-1$. By multiplying both sides of $\mathbf{z}=\bold{W}_{act} \bold{Y}+\bold{W}_{pas} \bold{X}+\bold{b}$ with $\mathbf{J}$, and using (\ref{qe1}), we get
% Denote the $m$-th row of $\bold{W}_{pas}$ and $\bold{W}_{act}$ by $\bold{p}_{m}^{T}$ and $\bold{q}_{m}^{T}$, respectively, and construct matrices $\bold{A}$ and $\bold{A}^{'}$ whose $m$-th rows ($m\in[k-1]$) are $\bold{p}^{T}_{m+1}-\bold{p}^{T}_{m}$ and $\bold{q}^{T}_{m+1}-\bold{q}^{T}_{m}$, respectively.
\begin{align}
    \bold{JW}_{pas}\bold{X} &= \mathbf{Jz}-\bold{J}\bold{W}_{act}\mathbf{Y}-\mathbf{Jb}\label{eqeq1}\\
    &=\mathbf{c}^{'}-\bold{J}\bold{W}_{act}\mathbf{Y}-\mathbf{Jb},\label{eq:1}
    \end{align}
where $\mathbf{c}'$ is a $(k-1)$-dimensional vector whose $m$-th element is $\ln \frac{c_{m+1}}{c_m}$. Denoting the RHS of (\ref{eq:1}) by $\bold{b}'$, (\ref{eq:1}) writes in short as $\mathbf{AX}=\mathbf{b}'$, where $\mathbf{A}\triangleq \bold{JW}_{pas}$.

\begin{remark}\label{Rem1}
It is important to note that the way to obtain a system of linear equations is not unique, but all of them are equivalent in the sense that they result in the same solution space. More specifically, let $\mathbf{R}_{(k-1)\times(k-1)}$ be an invertible matrix, and define $\mathbf{A}_{\textnormal{new}}\triangleq\mathbf{RA}$ and $\mathbf{b}'_{\textnormal{new}}\triangleq\mathbf{Rb}'$. We have that both $\mathbf{AX}=\mathbf{b}'$ and $\mathbf{A}_{\textnormal{new}}\mathbf{X}=\mathbf{b}'_{\textnormal{new}}$ are equivalent.
\end{remark}

The white-box setting refers to the scenario where the adversary is aware of $(\mathbf{W}_{act},\mathbf{W}_{pas},\mathbf{b})$ and the black-box setting refers to the context in which the adversary is only aware of $\mathbf{W}_{act}$.
% it is straightforward to verify that $\bold{v}$ is known to the active party.
% Therefore, in the white box setting, where the adversary is aware of $\mathbf{A}$, observing the confidence score in each prediction enables the adversary to construct the satisfiable system of linear equations $\mathbf{Ax}_{pas}=\mathbf{v}$. In \cite{Xinjian} the minimum-norm solution of this system, i.e., $\mathbf{A}^+\mathbf{v}$ is named as ESA. This is the same as $\hat{\mathbf{x}}_\textnormal{LS}$ in this paper.

Since the active party wishes to reconstruct the passive party's features, one measure by which the attack performance can be evaluated is the \textit{mean square error} per feature, i.e.,
\begin{equation}\label{MSE}
    \textnormal{MSE}=\frac{1}{d}\mathds{E}\left[\|\mathbf{X}-\hat{\mathbf{X}}\|^2\right],
\end{equation}
where $\hat{\mathbf{X}}$ is the adversary's estimate. Let $N$ denote the number of predictions. Assuming that these $N$ predictions are carried out in an i.i.d. manner, \textit{Law of Large Numbers} (LLN) allows to approximate MSE by its empirical value $\frac{1}{Nd}\sum_{i=1}^{N}\|\mathbf{X}_i-\hat{\mathbf{X}_i}\|^2$, since the latter converges almost surely to (\ref{MSE}) as $N$ grows.\footnote{It is important to note however that in the case when the adversary's estimates are not independent across the predictions (non-i.i.d. case) the empirical MSE is not necessarily equal to (\ref{MSE}). In such cases, the empirical MSE is taken as the performance metric.} This observation is later used in the experimental results to evaluate the performance of different reconstruction attacks.
\section{Preliminaries}\label{prel}
Throughout this paper, we are interested in solving a satisfiable\footnote{This means that at least one solution exists for this system, which is due to the context in which this problem arises.} system of linear equations, in which the unknowns (features of the passive party) are in the range $[0,1]$. This can be captured by solving for $\mathbf{x}$ in the equation
$\mathbf{Ax}=\mathbf{b}$, where $\mathbf{A}\in\mathds{R}^{m\times d}$, $\mathbf{x}\in[0,1]^d$, and $\mathbf{b}\in\mathds{R}^m$ for some positive integers $m,d$. We are particularly interested in the case when the number of unknowns $d$ is greater than the number of equations $m$. This is a particular case of an \textit{indeterminate/under-determined system}, where $\mathbf{A}$ does not have full column rank and an infinitude of solutions exists for this linear system. Since the system under consideration is satisfiable, any solution can be written as $\mathbf{A}^{+}\mathbf{b}+(\mathbf{I}_d-\mathbf{A}^+\mathbf{A})\mathbf{w}$ for some $\mathbf{w}\in\mathds{R}^d$, where $\mathbf{A}^+$ denotes the pseudoinverse of $\mathbf{A}$ satisfying the Moore-Penrose conditions\cite{penrose}\footnote{When $\mathbf{A}$ has linearly independent rows, we have $\mathbf{A}^+=\mathbf{A}^T(\mathbf{AA}^T)^{-1}$.}. One property of pseudoinverse that is useful in the sequel is that if $\mathbf{A}=\mathbf{U\Sigma V}^T$ is a singular value decomposition (SVD) of $\mathbf{A}$, then $\mathbf{A}^+=\mathbf{V\Sigma}^+\mathbf{U}^T$, in which $\mathbf{\Sigma}^+$ is obtained by taking the reciprocal of each non-zero element on the diagonal of $\mathbf{\Sigma}$, and then transposing the matrix.

For a given pair $(\mathbf{A},\mathbf{b})$, define
\begin{equation}
    \mathcal{S}\triangleq\bigg\{\mathbf{x}\in\mathds{R}^d\bigg|\mathbf{A}\mathbf{x}=\mathbf{b}\bigg\}\ ,\     \mathcal{S_F}\triangleq\bigg\{\mathbf{x}\in\mathcal{S}\cap [0,1]^d\bigg\},
\end{equation}
as the solution space and feasible solution space, respectively. Alternatively, by defining
\begin{equation}\label{defw}
    \mathcal{W}\triangleq\bigg\{\mathbf{w}\in\mathds{R}^d \bigg| -\mathbf{A}^+\mathbf{b}\leq(\mathbf{I}_d-\mathbf{A}^+\mathbf{A})\mathbf{w}\leq\mathbf{1}_d-\mathbf{A}^+\mathbf{b}\bigg\},
\end{equation}
we have
\begin{equation}\label{defs}
    \mathcal{S_F}=\bigg\{(\mathbf{I}_d-\mathbf{A}^+\mathbf{A})\mathbf{w}+\mathbf{A}^+\mathbf{b}\bigg|\mathbf{w}\in\mathcal{W}\bigg\}.
\end{equation}
We have that $\mathcal{W}$ is a closed and bounded convex set defined as an intersection of $2d$ half-spaces. Since $\mathcal{S_F}$ is the image of $\mathcal{W}$ under an affine transformation, it is a closed convex polytope in $[0,1]^d$.

\begin{figure}[ht]
 \centering % centering figure
 \scalebox{0.7} % rescale the figure by a factor of 0.8
 {\includegraphics{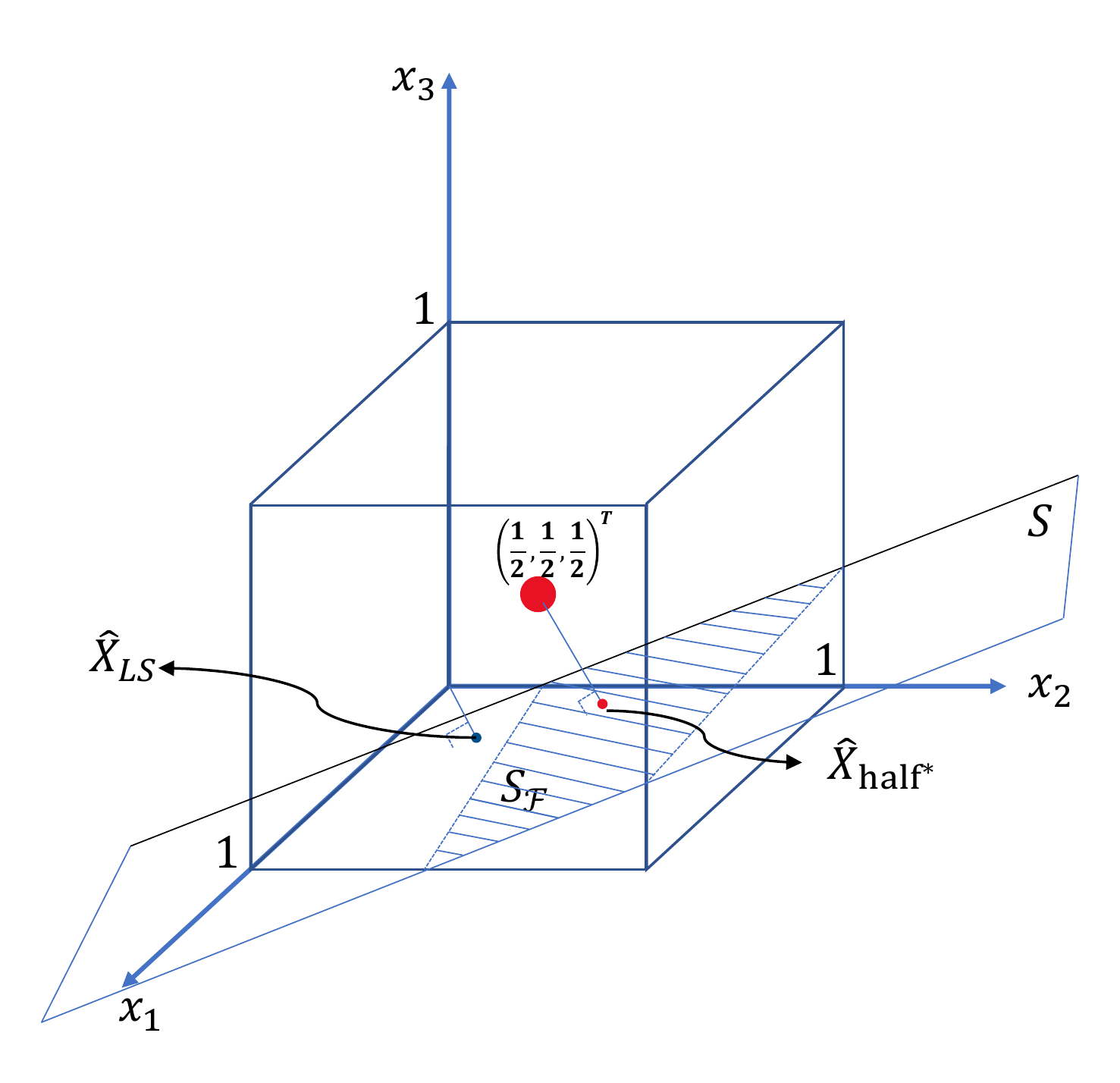}} % importing figure
 \caption{An example with $(m,d)=(1,3)$. The feasible solution space $\mathcal{S_F}$ is the intersection of the solution space $\mathcal{S}$, which is denoted by the plane representing $\mathbf{Ax}=\mathbf{b}$, with the hypercube $[0,1]^3$. In this example, the minimum-norm point on the solution space, i.e., $\hat{X}_\textnormal{LS}$, does not belong to $[0,1]^3$.}
 \label{fig11} % labeling to refer it inside the text
\end{figure}

\begin{figure}[ht]
 \centering % centering figure
 \scalebox{0.4} % rescale the figure by a factor of 0.8
 {\includegraphics{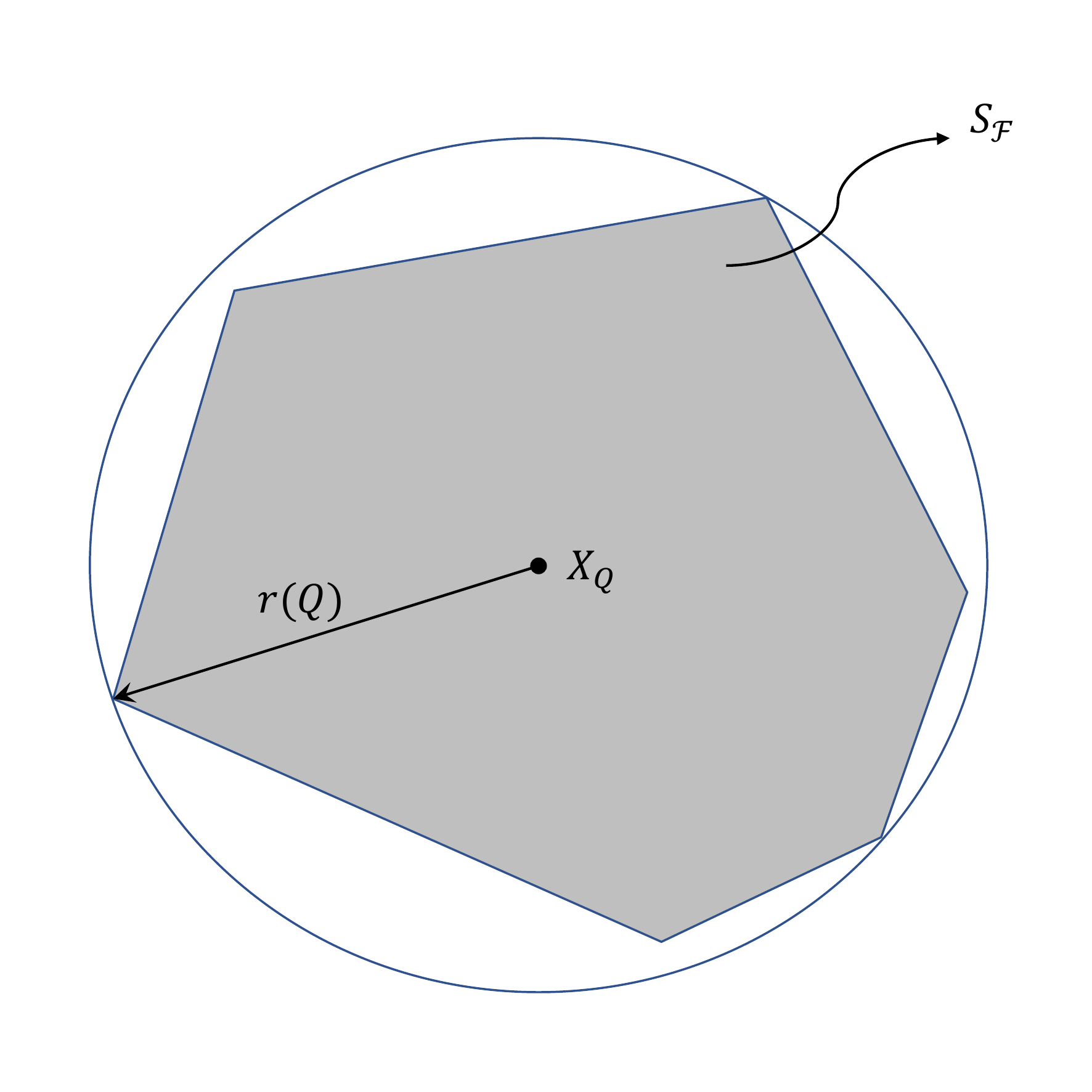}} % importing figure
 \caption{An example of the Chebyshev centre of $\mathcal{Q}$, which is the convex polytope $\mathcal{S_F}$ in this paper.}
 \label{fig123} % labeling to refer it inside the text
\end{figure}

For a general (satisfiable or not) system of linear equations $\mathbf{Cx}=\mathbf{d}$, we have that $\mathbf{C}^+\mathbf{d}=\argmin_{\mathbf{x}}\|\mathbf{Cx}-\mathbf{d}\|_2$. Moreover, if the system is satisfiable, the quantity $\mathbf{C}^+\mathbf{d}$ is the minimum $L^2$-norm solution. Therefore, in our problem, we have $\|\mathbf{x}\|_2\geq\|\mathbf{A}^+\mathbf{b}\|_2$ for all $\mathbf{x}\in\mathcal{S}$. Define $\hat{\mathbf{x}}_\textnormal{LS}\triangleq\mathbf{A}^+\mathbf{b}$, where the subscripts stands for \textit{Least Square}\footnote{Note that this naming is with a slight abuse of convention, as the term least square points to a vector $\mathbf{x}$ that minimizes $\|\mathbf{A}\mathbf{x}-\mathbf{b}\|_2$ in general, rather than a minimum norm solution.}. It is important to note that $\hat{\mathbf{x}}_\textnormal{LS}$ may not necessarily belong to $\mathcal{S_F}$, which is our region of interest. A geometrical representation for the case $(m,d)=(1,3)$ is provided in Figure \ref{fig11}. As a result, one can always consider the constrained optimization $\min_\mathbf{x}\|\mathbf{A}\mathbf{x}-\mathbf{b}\|_2$ with the constraint $\mathbf{0}_d\leq\mathbf{x}\leq\mathbf{1}_d$ in order to find a feasible solution. We denote any solution obtained in this manner by $\hat{\mathbf{x}}_\textnormal{CLS}$, where the subscript stands for \textit{Constrained Least Square}. In an indeterminate system, in contrast to $\hat{\mathbf{x}}_\textnormal{LS}$, $\hat{\mathbf{x}}_\textnormal{CLS}$ is not unique, and any point in $\mathcal{S_F}$ can be a candidate for $\hat{\mathbf{x}}_\textnormal{CLS}$ depending on the initial point of the solver.

Consider this simple example that $x$ is an unknown quantity in the range $[0,1]$ to be estimated and the error in the estimation is measured by the mean square, i.e., $|x-\hat{x}|^2$. Obviously, any point in $[0,1]$ can be proposed as an estimate for $x$. However, without any further knowledge about $x$, one can select the center of $[0,1]$, i.e., $\frac{1}{2}$ as an intuitive estimate. The rationale behind this selection is that the maximum error of the estimate $\hat{x}=\frac{1}{2}$, i.e., $\max_x|x-\frac{1}{2}|^2$ is minimal among all other estimates. In other words, the center minimizes the worst possible estimation error, and hence, it is optimal in the \textit{best-worst} sense. As mentioned earlier, any element of $\mathcal{S_F}$ is a feasible solution of $\mathbf{Ax}=\mathbf{b}$. This calls for a proper definition of the "center" of $\mathcal{S_F}$ as the best-worst solution. This is called the \textit{Chebyshev center} which is introduced in a general topological context as follows.

\begin{definition}
(Chebyshev Center \cite{Narang}) Let $\mathcal{Q}$ be a bounded subset of a metric space $(\mathcal{X},d)$, where $d$ denotes the distance. A Chebyshev center of $\mathcal{Q}$ is the center of minimal closed ball containing $\mathcal{Q}$, i.e., it is an element $x_\mathcal{Q}\in \mathcal{X}$ such that $\sup_{u\in \mathcal{Q}}d(x_\mathcal{Q},u)=\inf_{x\in \mathcal{X}}\sup_{u\in \mathcal{Q}}d(x,u)$. The quantity $r(\mathcal{Q})\triangleq \inf_{x\in \mathcal{X}}\sup_{u\in \mathcal{Q}}d(x,u)$ is the Chebyshev radius of $\mathcal{Q}$.
\end{definition}
In this paper, the metric space under consideration is $(\mathds{R}^d,\|\cdot\|_2)$ for some positive integer $d$, and we have
\begin{equation}\label{Chebcenter}
    \mathbf{x}_\mathcal{Q}\triangleq \argmin_{\hat{\mathbf{x}}\in\mathds{R}^d}\max_{\mathbf{x}\in \mathcal{Q}}\|\mathbf{x}-\hat{\mathbf{x}}\|_2^2.
\end{equation}
For example the Chebyshev center of $[0,1]$ in $(\mathds{R},\|\cdot\|_2)$ is $\frac{1}{2}$, and the Chebyshev center of the ball $\{\mathbf{r}\in\mathds{R}^2| \|\mathbf{r}\|\leq 1\}$ in $(\mathds{R}^2,\|\cdot\|_2)$ is the origin\footnote{Note that the Chebyshev center of the circle $\{\mathbf{r}\in\mathds{R}^2|\|\mathbf{r}\| = 1\}$ in the same metric is still the origin, but obviously it does not belong to the circle, as the circle is not convex in $\mathds{R}^2$}.

In this paper, the subset of interest, i.e., $\mathcal{Q}$ ($=\mathcal{S_F}$), is bounded, closed and convex. In this context, the Chebyshev center of $\mathcal{Q}$ is unique and belongs to $\mathcal{Q}$. Hence, in the argmin in (\ref{Chebcenter}), $\hat{\mathbf{x}}\in\mathds{R}^d$ can be replaced with $\hat{\mathbf{x}}\in\mathcal{Q}$. An example is provided in Figure \ref{fig123}.

Except for simple cases, computing the Chebyshev center is a computationally complex problem due to the non-convex quadratic inner maximization in (\ref{Chebcenter}). When the subset of interest, i.e., $\mathcal{Q}$, can be written as the convex hull of a finite number of points, there are algorithms \cite{Welzl,Botkin} that can find the Chebyshev center. In this paper, $\mathcal{Q}$($=\mathcal{S_F}$) is a convex polytope with a finite number of extreme points (as shown in Figure \ref{fig11}), hence, one can apply these algorithms. However, it is important to note that these extreme points are not given \textit{a priori} and they need to be found in the first place from the equation $\mathbf{Ax}=\mathbf{b}$. Since the procedure of finding the extreme points of $\mathcal{S_F}$ is exponentially complex, it makes sense to seek approximations for the Chebyshev center that can be handled efficiently. Therefore, in this paper, instead of obtaining the exact Chebyshev center of $\mathcal{S_F}$, we rely on its approximations. A nice approximation worth mentioning is given in \cite{Eldar}, which is in the context of signal processing and is explained in the sequel. This approximation is based on replacing
the non-convex inner maximization in (\ref{Chebcenter}) by its semidefinite
relaxation, and then solving the resulting convex-concave minimax problem. A clear explanation of this method, henceforth named Relaxed Chebyshev Center 1 (RCC1), is needed because it is used as one of the adversary's attack methods in this paper. Later, in proposition \ref{prop3}, a second relaxation is proposed, which is denoted as RCC2.

The set $\mathcal{Q}$ in \cite{Eldar} is an intersection of $k$ ellipsoids, i.e.,
\begin{equation}\label{setq}
    \mathcal{Q}\triangleq\{\mathbf{x}|f_i(\mathbf{x})\triangleq \mathbf{x}^T\mathbf{Q}_i\mathbf{x}+2\mathbf{g}_i^T\mathbf{x}+t_i\leq 0, i\in[k]\},
\end{equation}
where $\mathbf{Q}_i\succeq 0, \mathbf{g}_i\in\mathds{R}^d,t_i\in\mathds{R}$, and the optimization problem is given in (\ref{Chebcenter}). Defining $\mathbf{\Delta}\triangleq \mathbf{x}\mathbf{x}^T$, the equivalence holds
\begin{equation}\label{equivalence}
    \max_{\mathbf{x}\in\mathcal{Q}}\|\mathbf{x}-\hat{\mathbf{x}}\|^2 = \max_{(\mathbf{\Delta,x})\in\mathcal{G}}\{\|\hat{\mathbf{x}}\|^2-2\hat{\mathbf{x}}^T\mathbf{x}+\textnormal{Tr}(\mathbf{\Delta})\},
\end{equation}
where $\mathcal{G}=\{(\mathbf{\Delta,x})|\ \textnormal{Tr}(\mathbf{Q}_i\mathbf{\Delta})+2\mathbf{g}_i^T\mathbf{x}+t_i\leq 0, i\in[k],\mathbf{\Delta} =\mathbf{x}\mathbf{x}^T\}$. By focusing on the right hand side (RHS) of (\ref{equivalence}) instead of its left hand side (LHS), we are now dealing with the maximization of a concave (linear) function in $(\mathbf{\Delta,x})$. However, the downside is that $\mathcal{G}$ is not convex, in contrast to $\mathcal{Q}$. Here is where the relaxation is done in \cite{Eldar}, and the optimization is carried out over a relaxed version of $\mathcal{G}$, i.e.,
\begin{equation*}
    \mathcal{T}=\{(\mathbf{\Delta,x})|\ \textnormal{Tr}(\mathbf{Q}_i\mathbf{\Delta})+2\mathbf{g}_i^T\mathbf{x}+t_i\leq 0, i\in[k],\mathbf{\Delta}\succeq \mathbf{x}\mathbf{x}^T\},
\end{equation*}
which is a convex set, and obviously $\mathcal{G}\subset\mathcal{T}$. As a results, RCC1 is the solution to the following minimax problem
\begin{equation}\label{eq2}
    \min_{\hat{\mathbf{x}}}\max_{(\mathbf{\Delta,x})\in\mathcal{T}}\{\|\hat{\mathbf{x}}\|^2-2\hat{\mathbf{x}}^T\mathbf{x}+\textnormal{Tr}(\mathbf{\Delta})\}.
\end{equation}
Since $\mathcal{T}$ is bounded, and the objective in (\ref{eq2}) is convex in $\hat{\mathbf{x}}$ and concave (linear) in $(\mathbf{\Delta,x})$, the order of minimization and maximization can be changed. Knowing that the minimum (over $\hat{\mathbf{x}}$) of the objective function occurs at $\hat{\mathbf{x}}=\mathbf{x}$, (\ref{eq2}) reduces to
\begin{equation*}
    \max_{(\mathbf{\Delta,x})\in\mathcal{T}}\{-\|{\mathbf{x}}\|^2+\textnormal{Tr}(\mathbf{\Delta})\},
\end{equation*}
whose objective is concave and the constraints are linear matrrix inequalities, and RCC1 is the $\mathbf{x}$-part of the solution. Since $\mathcal{G}\subset \mathcal{T}$, the radius of the corresponding ball of RCC1 is an upperbound on $r(\mathcal{Q})$, i.e., the Chebyshev radius of $\mathcal{Q}$.

An explicit representation of $\mathbf{x}_{\textnormal{RCC1}}$ is given in \cite[Theorem III.1]{Eldar}, which is restated here.
\begin{equation}
    \mathbf{x}_{\textnormal{RCC1}}=-\bigg(\sum_{i=1}^k\alpha_i\mathbf{Q}_i\bigg)^{-1}\bigg(\sum_{i=1}^k\alpha_i\mathbf{g}_i\bigg),
\end{equation}
where $(\alpha_1,\ldots,\alpha_k)$ is an optimal solution of the following convex problem
\begin{align}
    &\min_{\alpha_i}\bigg\{\bigg(\sum_{i=1}^k\alpha_i\mathbf{g}_i\bigg)^T\bigg(\sum_{i=1}^k\alpha_i\mathbf{Q}_i\bigg)^{-1}\bigg(\sum_{i=1}^k\alpha_i\mathbf{g}_i\bigg)-\sum_{i=1}^k\alpha_it_i\bigg\}\nonumber\\
    &\ \textnormal{S.t.}\ \sum_{i=1}^k\alpha_i\mathbf{Q}_i\succeq\mathbf{I}, \nonumber\\
    &\ \ \ \ \ \ \ \alpha_i\geq 0, i\in[k],\label{sdp}
\end{align}
which can be cast as a semidefinite program (SDP) and solved by an SDP solver.

It is shown in \cite{Eldar} that similarly to the exact Chebyshev center, $\mathbf{x}_{\textnormal{RCC1}}$ is also unique (due to strict convexity of the $L^2$-norm) and it belongs to $\mathcal{Q}$, where the latter follows from the fact that for any $(\mathbf{\Delta}',\mathbf{x}')\in\mathcal{T}$, we have $\mathbf{x}'\in\mathcal{Q}$, which is due to the positive semidefiniteness of $\mathbf{Q}_i,i\in[k]$.

Finally, suppose that one of the constraints defining the set $\mathcal{Q}$ is a double-sided linear inequality of the form $l\leq\mathbf{a}^T\mathbf{x}\leq u$. We can proceed and write this constraint as two constraints, i.e., $\mathbf{a}^T\mathbf{x}\leq u$ and $-\mathbf{a}^T\mathbf{x}\leq -l$. However, it is shown in \cite{Eldar} that it is better (in the sense of a smaller minimax estimate) to write it in the quadratic form, i.e., $(\mathbf{a}^T\mathbf{x}-l)(\mathbf{a}^T\mathbf{x}-u)\leq 0$. Although the exact Chebyshev center of $\mathcal{Q}$ does not rely on its specific representation, the RCC1 does, as it is the result of a relaxation of $\mathcal{Q}$. Hence, any constraint of the form $l\leq\mathbf{a}^T\mathbf{x}\leq u$ will be replaced by $\mathbf{x}^T\mathbf{Q}\mathbf{x}+2\mathbf{g}^T\mathbf{x}+t\leq 0$, with $\mathbf{Q}=\mathbf{aa}^T, \mathbf{g}=-\frac{u+l}{2}\mathbf{a}$, and $t = ul$.

The final discussion in this section is Von Neumann's trace inequality \cite{marshall}, which is used throughout the paper. It states that for two $n\times n$ (complex) matrices $\mathbf{A,B}$, with singular values $\alpha_1\geq\alpha_2\geq\ldots\geq\alpha_n$ and $\beta_1\geq\beta_2\geq\ldots\geq\beta_n$, respectively, we have
\begin{equation}
    |\textnormal{Tr}(\mathbf{AB})|\leq\sum_{i=1}^n\alpha_i\beta_i.
\end{equation}
If $\mathbf{A,B}$ are symmetric positive semidefinite matrices with eigenvalues $a_1\geq a_2\geq\ldots\geq a_n$ and $b_1\geq b_2\geq\ldots\geq b_n$, respectively, we have
\begin{equation}
    \sum_{i=1}^na_ib_{n-i+1}\leq\textnormal{Tr}(\mathbf{AB})\leq\sum_{i=1}^na_ib_i.
\end{equation}

\section{White-box setting}\label{wbs}
Let $X\in[0,1]$ be a random variable distributed according to an unknown CDF $F_{X}$. The goal is to find an estimate $\hat{X}$. First, we need the following lemma, which states that when there is no side information available to the estimator, there is no loss of optimality in restricting to the set of deterministic estimates.
\begin{lemma}\label{lem1}
Any randomized guess is outperformed by its statistical mean, and the performance improvement is equal to the variance of the random guess.
\end{lemma}
\begin{proof}
Let $\hat{X}$ be a random guess distributed according to a fixed CDF $F_{\hat{X}}$. We have
\begin{align*}
    \mathds{E}[(X-\hat{X})^2]&=\mathds{E}[(X-\mathds{E}[\hat{X}])^2]+\textnormal{Var}(\hat{X})\nonumber\\
    &\geq \mathds{E}[(X-\mathds{E}[\hat{X}])^2].
\end{align*}
Hence, any estimate $\hat{X}\sim F_{\hat{X}}$ is outperformed by the new deterministic estimate $\mathds{E}[\hat{X}]$, whose performance improvement is $\textnormal{Var}(\hat{X})$.
\end{proof}

%  As an example, when there is no side information about $X$, the deterministic estimate $\hat{X}=\frac{1}{3}$ is better than any other random estimate $\hat{X}$ for which $\mathds{E}[\hat{X}]=\frac{1}{3}$.

Since the underlying distribution of $X$ is unknown to the estimator, one conventional approach is to consider the \textit{best-worst} estimator. In other words, the goal of the estimator is to minimize the maximum error, which can be cast as a minimax problem, i.e.,
\begin{equation}
    \min_{\hat{x}}\max_{F_X}\mathds{E}[(X-\hat{x})^2],
\end{equation}
where lemma \ref{lem1} is used in the minimization, i.e., instead of minimizing over $F_{\hat{X}}$, we are minimizing over the singleton $\hat{x}$. Since for any fixed $\hat{x}$, we have $\max_{F_X}\mathds{E}[(X-\hat{x})^2]=\max_{x\in[0,1]}(x-\hat{x})^2$, the best-worst estimation is the solution to
\begin{equation}
   \min_{\hat{x}}\max_{x\in[0,1]}(x-\hat{x})^2,
\end{equation}
which is the Chebyshev center of the interval $[0,1]$ in the space $(\mathds{R},\|\cdot\|_2)$ and it is equal to $\frac{1}{2}$. This implies that with the estimator being blind to the underlying distribution and any possible side information, the best-worst estimate is the Chebyshev center of the support of the random variable, here $[0,1]$.

As one step further, consider that $\mathbf{X}\in[0,1]^d$ is a $d$-dimensional random vector distributed according to an unknown CDF $F_\mathbf{X}$. Although the estimator is still unaware of $F_\mathbf{X}$, this time it has access to the matrix-vector pair $(\mathbf{A},\mathbf{AX})$, and based on this side information, it gives an estimate $\hat{\mathbf{X}}$. This side information refines the prior belief of $\mathbf{X}\in[0,1]^d$ to $\mathbf{X}\in\mathcal{S_F}$. Similarly to the previous discussion, the best-worst estimator gives the Chebyshev center of $\mathcal{S_F}$. As mentioned before, obtaining the exact Chebyshev centre of $\mathcal{S_F}$ is computationally difficult, hence, we focus on its approximations. However, prior to the approximation, we start with simple heuristic estimates that bear an intuitive notion of centeredness.

The first scheme in estimating $\mathbf{X}$, is the naive estimate of $\frac{1}{2}\mathbf{1}_d$, which is the Chebyshev center of $[0,1]^d$. We name this estimate as $\hat{\mathbf{X}}_\textnormal{half}=\frac{1}{2}\mathbf{1}_d$. We already know that when the only information that we have about $\mathbf{X}$ is that it belongs to $[0,1]^d$, then $\hat{\mathbf{X}}_\textnormal{half}$ is optimal in the best-worst sense.

The adversary can perform better when the side information $(\mathbf{AX},\mathbf{A})$ is available. A second scheme can be built on top of the previous scheme as follows. The estimator finds a solution in the solution space, i.e., $\mathcal{S}$, that is closest to $\frac{1}{2}\mathbf{1}_d$, which is shown in Figure \ref{fig11}. In this scheme, the estimate, named $\hat{\mathbf{X}}_{\textnormal{half}^*}$, is given by
\begin{equation}\label{xhalf}
    \hat{\mathbf{X}}_{\textnormal{half}^*}\triangleq\argmin_{{\mathbf{x}}\in\mathcal{S}}\|{\mathbf{x}}-\frac{1}{2}\mathbf{1}_d\|^2,
\end{equation}
whose explicit representation is provided in the following proposition.\footnote{Note that $\hat{\mathbf{X}}_{\textnormal{half}^*}$ may or may not belong to $\mathcal{S_F}$.}
\begin{proposition}
We have
\begin{equation}\label{half*2}
   \hat{\mathbf{X}}_{\textnormal{half}^*}= \mathbf{A}^{+}\mathbf{b}+\frac{1}{2}(\mathbf{I}_d-\mathbf{A}^+\mathbf{A})\mathbf{1}_d.
\end{equation}
\end{proposition}
\begin{proof}
For any $\mathbf{x}\in\mathcal{S}$, we have $\mathbf{x}=\mathbf{A}^{+}\mathbf{b}+(\mathbf{I}_d-\mathbf{A}^+\mathbf{A})\mathbf{w}$ for some $\mathbf{w}\in\mathds{R}^d$. Hence,
\begin{align}
    \min_{\hat{\mathbf{x}}\in\mathcal{S}}\|\hat{\mathbf{x}}-\frac{1}{2}\mathbf{1}_d\|^2&=\min_{\mathbf{w}\in\mathds{R}^d}\|\mathbf{A}^{+}\mathbf{b}+(\mathbf{I}_d-\mathbf{A}^+\mathbf{A})\mathbf{w}-\frac{1}{2}\mathbf{1}_d\|^2\label{e1}\\
    &=\min_{\mathbf{w}\in\mathds{R}^d}\|\mathbf{Cw}-\mathbf{z}\|^2,\label{e2}
\end{align}
where $\mathbf{C}\triangleq\mathbf{I}_d-\mathbf{A}^+\mathbf{A}$ and $\mathbf{z}\triangleq\frac{1}{2}\mathbf{1}_d-\mathbf{A}^{+}\mathbf{b}$. It is already known that the minimizer in (\ref{e2}) is $\mathbf{w}^*=\mathbf{C}^+\mathbf{z}$, which results in
\begin{align}
   \hat{\mathbf{X}}_{\textnormal{half}^*}&=\mathbf{A}^{+}\mathbf{b}+(\mathbf{I}_d-\mathbf{A}^+\mathbf{A})\mathbf{w}^*\nonumber\\ &=\mathbf{A}^{+}\mathbf{b}+(\mathbf{I}_d-\mathbf{A}^+\mathbf{A})(\mathbf{I}_d-\mathbf{A}^+\mathbf{A})^+(\frac{1}{2}\mathbf{1}_d-\mathbf{A}^{+}\mathbf{b})\nonumber\\
   &=\mathbf{A}^{+}\mathbf{b}+(\mathbf{I}_d-\mathbf{A}^+\mathbf{A})^2(\frac{1}{2}\mathbf{1}_d-\mathbf{A}^{+}\mathbf{b})\label{es1}\\
   &=\mathbf{A}^{+}\mathbf{b}+(\mathbf{I}_d-\mathbf{A}^+\mathbf{A})(\frac{1}{2}\mathbf{1}_d-\mathbf{A}^{+}\mathbf{b})\label{es2}\\
   &=\mathbf{A}^{+}\mathbf{b}+\frac{1}{2}(\mathbf{I}_d-\mathbf{A}^+\mathbf{A})\mathbf{1}_d,\label{es3}
\end{align}
where (\ref{es1}) to (\ref{es3}) are justified as follows. Let $\mathbf{A}=\mathbf{U\Sigma V}^T$ be an SVD of $\mathbf{A}$. From $\mathbf{A}^+=\mathbf{V\Sigma}^+\mathbf{U}^T$, we get $\mathbf{A}^+\mathbf{A}=\mathbf{V\Sigma}^+\mathbf{\Sigma V}^T$ and $\mathbf{I}-\mathbf{A}^+\mathbf{A}=\mathbf{V}(\mathbf{I}-\mathbf{\Sigma}^+\mathbf{\Sigma})\mathbf{V}^T$. Knowing that $\mathbf{I}-\mathbf{\Sigma}^+\mathbf{\Sigma}$ is a diagonal matrix with only 0 and 1 on its diagonal, we get $(\mathbf{I}-\mathbf{\Sigma}^+\mathbf{\Sigma})^+=\mathbf{I}-\mathbf{\Sigma}^+\mathbf{\Sigma}$, and therefore, $(\mathbf{I}-\mathbf{A}^+\mathbf{A})^+=\mathbf{I}-\mathbf{A}^+\mathbf{A}$, which results in (\ref{es1}). Noting that $(\mathbf{I}-\mathbf{A}^+\mathbf{A})$ is a projector results in (\ref{es2}). Finally, by noting that $(\mathbf{I}-\mathbf{\Sigma}^+\mathbf{\Sigma})\mathbf{\Sigma}^+\mathbf{\Sigma}=\mathbf{0}$, we get $(\mathbf{I}-\mathbf{A}^+\mathbf{A})\mathbf{A}^+\mathbf{b}=\mathbf{0}$, which results in (\ref{es3}).
\end{proof}
Thus far, we have considered two simple schemes, i.e., $\hat{\mathbf{X}}_{\textnormal{half}}$ and $\hat{\mathbf{X}}_{\textnormal{half}^*}$. In what follows, we investigate two approximations for the Chebyshev center of $\mathcal{S_F}$. The exact Chebyshev center of $\mathcal{S_F}$ is given by
\begin{equation}\label{cheb1}
    \argmin_{\hat{\mathbf{x}}\in\mathcal{S_F}}\max_{\mathbf{x}\in\mathcal{S_F}}\|\mathbf{x}-\hat{\mathbf{x}}\|^2.
\end{equation}
Let $\mathbf{A}=\mathbf{U\Sigma V}^T$ be an SVD of $\mathbf{A}$, where the singular values are arranged in a non-increasing order, i.e., $\sigma_1\geq\sigma_2\geq\sigma_3\ldots$. Let $r\triangleq\textnormal{rank}(\mathbf{A})$. Hence, $\textnormal{Null}(\mathbf{A})=\textnormal{Span}\{\mathbf{v}_{r+1},\mathbf{v}_{r+2},\ldots,\mathbf{v}_d\}$, which is the span of those right singular vectors that correspond to zero singular values. Define $\mathbf{W}\triangleq[\mathbf{v}_{r+1},\mathbf{v}_{r+2},\ldots,\mathbf{v}_d]_{d\times(d-r)}$.

The orthonormal columns of $\mathbf{V}$ can be regarded as a basis for $\mathds{R}^d$. Hence, any vector $\mathbf{w}\in\mathds{R}^d$ can be written as $\mathbf{w}=\mathbf{Wu}+\mathbf{q}$, where $\mathbf{u}\in\mathds{R}^{d-r},\mathbf{q}\in\mathds{R}^d$, and $\mathbf{W}^T\mathbf{q}=\mathbf{0}$. With the definition of $\mathbf{W}$, we have
\begin{align*}
    \mathbf{I}-\mathbf{A}^+\mathbf{A}&=\mathbf{V}(\mathbf{I}-\mathbf{\Sigma}^+\mathbf{\Sigma})\mathbf{V}^T\\
    &=\mathbf{WW}^T.
\end{align*}
Noting that $\mathbf{W}$ has orthonormal columns, we have $\mathbf{W}^T\mathbf{W}=\mathbf{I}_{(d-r)}$. Therefore, $\mathcal{S_F}$ in (\ref{defs}) can be written as
\begin{equation}\label{newdef}
   \mathcal{S_F}=\{\mathbf{A}^+\mathbf{b}+\mathbf{Wu}|\mathbf{u}\in\tilde{\mathcal{S_F}}\},\ \Tilde{\mathcal{S_F}}\triangleq\{\mathbf{u}\in\mathds{R}^{(d-r)}|-\mathbf{A}^+\mathbf{b}\leq\mathbf{Wu}\leq\mathbf{1}-\mathbf{A}^+\mathbf{b}\}.
\end{equation}
% For any $\mathbf{x}\in\mathcal{S_F}$, there exist $\mathbf{w}\in\mathds{R}^d$ such that $\mathbf{x}=\mathbf{A}^+b+(\mathbf{I}-\mathbf{A}^+\mathbf{A})\mathbf{w}$ and $\mathbf{x}\in[0,1]^d$. Since $\mathbf{I}-\mathbf{A}^+\mathbf{A}$ is a projector to the null space of $\mathbf{A}$, we have the following equivalence
% \begin{equation}\label{equival}
%     \mathbf{x}\in\mathcal{S_F}\Longleftrightarrow \mathbf{u}\in\mathds{R}^{(d-r)}:\mathbf{x}=\mathbf{A}^+\mathbf{b}+\mathbf{Wu}\ \&\ -\mathbf{A}^+\mathbf{b}\leq\mathbf{Wu}\leq\mathbf{1}-\mathbf{A}^+\mathbf{b}.
% \end{equation}
% By defining the set $\Tilde{\mathcal{S_F}}\triangleq\{\mathbf{u}\in\mathds{R}^{(d-r)}|-\mathbf{A}^+\mathbf{b}\leq\mathbf{Wu}\leq\mathbf{1}-\mathbf{A}^+\mathbf{b}\}$, (\ref{cheb1}) reduces to
Therefore,
\begin{align}
    \min_{\hat{\mathbf{x}}\in\mathcal{S_F}}\max_{\mathbf{x}\in\mathcal{S_F}}\|\mathbf{x}-\hat{\mathbf{x}}\|^2&=\min_{\hat{\mathbf{u}}\in\tilde{\mathcal{S_F}}}\max_{\mathbf{u}\in\tilde{\mathcal{S_F}}}\|\mathbf{A}^+\mathbf{b}+\mathbf{Wu}-\mathbf{A}^+\mathbf{b}-\mathbf{W\hat{u}}\|^2\nonumber\\
    &=\min_{\hat{\mathbf{u}}\in\tilde{\mathcal{S_F}}}\max_{\mathbf{u}\in\tilde{\mathcal{S_F}}}\|\mathbf{u}-\mathbf{\hat{u}}\|^2.
\end{align}
% The set $\tilde{\mathcal{S_F}}$ consists of $2d$ linear constraints in $\mathbf{u}$. In order to cast $\tilde{\mathcal{S_F}}$ as $\mathcal{Q}$ in (\ref{setq}), i.e., having the constraints in the quadratic form, we simply replace any pair of linear constraints by its quadratic dual, i.e., we replace $l\leq u\leq h$ by $(u-l)(u-h)\leq 0$. Denoting the $i$-th row of $\mathbf{W}$ and the $i$-th element of $\mathbf{A}^+\mathbf{b}$ by $\mathbf{a}_i^T$ and $q_i, i\in[d]$, respectively, $\tilde{\mathcal{S_F}}$ can be written as (\ref{setq}) with $\mathbf{Q}_i=\mathbf{a}_i\mathbf{a}_i^T$, $\mathbf{g}_i=(q_i-\frac{1}{2})\mathbf{a}_i$, and $t_i=-q_i(1-q_i)$. With these replacements, we have the first relaxed Chebyshev center of $\mathcal{S_F}$ given by

Denoting the $i$-th row of $\mathbf{W}$ and the $i$-th element of $\mathbf{A}^+\mathbf{b}$ by $\mathbf{a}_i^T$ and $q_i, i\in[d]$, respectively, the following proposition provides an approximation for the exact Chebyshev center in (\ref{cheb1}).
\begin{proposition}
A relaxed Chebyshev center of $\mathcal{S_F}$ is given by
\begin{equation}\label{Eldarcheb}
    \hat{\mathbf{X}}_\textnormal{RCC1}=\mathbf{A}^+\mathbf{b}-\mathbf{W}\bigg(\sum_{i=1}^d\alpha_i\mathbf{Q}_i\bigg)^{-1}\bigg(\sum_{i=1}^d\alpha_i\mathbf{g}_i\bigg),
\end{equation}
where $\alpha_i$'s are obtained as in (\ref{sdp}) with $\mathbf{Q}_i=\mathbf{a}_i\mathbf{a}_i^T$, $\mathbf{g}_i=(q_i-\frac{1}{2})\mathbf{a}_i$, and $t_i=-q_i(1-q_i)$. Furthermore, $\hat{X}_{\textnormal{RCC1}}$ is unique and it belongs to the set of feasible solution, i.e., $\mathcal{S_F}$.
\end{proposition}
\begin{proof}
The $2d$ linear constraints of $\tilde{\mathcal{S_F}}$ are in the form $-q_i\leq\mathbf{a}^T_i\mathbf{u}\leq1-q_i$ for $i\in[d]$. By writing these constraints as their dual quadratic form $\mathbf{u}^T\mathbf{Q}_i\mathbf{u}+2\mathbf{g}^T_i\mathbf{u_i}+t_i\leq 0$, with $\mathbf{Q}_i=\mathbf{a}_i\mathbf{a}_i^T, \mathbf{g}_i=-\frac{1-2q_i}{2}\mathbf{a}_i$, and $t_i = -q_i(1-q_i),i\in[d]$, and following the approach in \cite{Eldar}, which is explained in section \ref{prel}, $\hat{\mathbf{X}}_{\textnormal{RCC1}}$ is obtained as in (\ref{Eldarcheb}). Finally, the uniqueness and feasibility of $\hat{\mathbf{X}}_{\textnormal{RCC1}}$ follows from the arguments after (\ref{sdp}).
\end{proof}
A second relaxation is provided in the following proposition.
\begin{proposition}\label{prop3}
A relaxed Chebyshev center of $\mathcal{S_F}$ is given by
\begin{equation}\label{cheb22}
    \hat{\mathbf{X}}_\textnormal{RCC2}=\mathbf{A}^+\mathbf{b}+\mathbf{Wu}^*,
\end{equation}
where $\mathbf{u}^*$ is the solution of
\begin{align}
    &\max_{\mathbf{u}}\ \ \mathbf{1}^T\mathbf{Wu}-\|\mathbf{u}\|^2\nonumber\\
    &\ \ \textnormal{S.t.}\ -\mathbf{A}^+\mathbf{b}\leq\mathbf{Wu}\leq\mathbf{1}-\mathbf{A}^+\mathbf{b}.\label{cheb22cons}
\end{align}
Furthermore, $\hat{X}_{\textnormal{RCC2}}$ is unique and it belongs to the set of feasible solution, i.e., $\mathcal{S_F}$.
\end{proposition}
\begin{proof}
The inner maximization in (\ref{cheb1}) is
\begin{equation}
    \max_{\mathbf{x}\in\mathcal{S_F}}\|\mathbf{x}-\hat{\mathbf{x}}\|^2=\max_{\mathbf{x}\in\mathcal{S_F}}\{\|\mathbf{x}\|^2-2\mathbf{x}^T\hat{\mathbf{x}}+\|\hat{\mathbf{x}}\|^2\},
\end{equation}
which is a maximization of a convex objective function. As discussed before, one way of relaxing this problem was studied in \cite{Eldar} where the relaxation was over the search space. Here, we propose to directly relax the objective function by making use of the boundedness of $\mathbf{x}$. In other words, since for any $\mathbf{x}\in\mathcal{S_F}$, $\mathbf{x}\in[0,1]^d$, we have $\|\mathbf{x}\|^2\leq \mathbf{1}^T\mathbf{x}$. Hence, we can write
\begin{align}
    \min_{\hat{\mathbf{x}}\in\mathcal{S_F}}\max_{\mathbf{x}\in\mathcal{S_F}}\|\mathbf{x}-\hat{\mathbf{x}}\|^2&=\min_{\hat{\mathbf{x}}\in\mathcal{S_F}}\max_{\mathbf{x}\in\mathcal{S_F}}\{\|\mathbf{x}\|^2-2\mathbf{x}^T\hat{\mathbf{x}}+\|\hat{\mathbf{x}}\|^2\}\nonumber\\
    &\leq \min_{\hat{\mathbf{x}}\in\mathcal{S_F}}\max_{\mathbf{x}\in\mathcal{S_F}}\{\mathbf{1}^T\mathbf{x}-2\mathbf{x}^T\hat{\mathbf{x}}+\|\hat{\mathbf{x}}\|^2\}\nonumber\\
    &=\max_{\mathbf{x}\in\mathcal{S_F}}\min_{\hat{\mathbf{x}}\in\mathcal{S_F}}\{\mathbf{1}^T\mathbf{x}-2\mathbf{x}^T\hat{\mathbf{x}}+\|\hat{\mathbf{x}}\|^2\}\label{f1}\\
    &=\max_{\mathbf{x}\in\mathcal{S_F}} \{\mathbf{1}^T\mathbf{x}-\|{\mathbf{x}}\|^2\}\label{f2}\\
    &=\max_{\mathbf{u}\in\tilde{\mathcal{S_F}}}\left\{\mathbf{1}^T(\mathbf{A}^+\mathbf{b}+\mathbf{Wu})-\|\mathbf{A}^+\mathbf{b}+\mathbf{Wu}\|^2\right\}\label{f3}\\
    &=\mathbf{1}^T\mathbf{A}^+\mathbf{b}-\|\mathbf{A}^+\mathbf{b}\|^2+\max_{\mathbf{u}\in\tilde{\mathcal{S_F}}}\{\mathbf{1}^T\mathbf{Wu}-\|\mathbf{u}\|^2\}\label{f4}
\end{align}
where (\ref{f1}) follows from i) the boundedness of $\mathcal{S_F}$, and ii) the concavity (linearity) and convexity of the objective in $\mathbf{x}$ and $\hat{\mathbf{x}}$, respectively. (\ref{f2}) follows from the fact that knowing $\mathbf{x}\in\mathcal{S_F}$, $\hat{\mathbf{x}}=\mathbf{x}$ is the minimizer in (\ref{f1}). The RCC2 estimate is the solution of (\ref{f2}). (\ref{f3}) follows from the equivalence given in (\ref{newdef}), and denoting the maximizer of (\ref{f3}) by $\mathbf{u}^*$, we have $\hat{\mathbf{X}}_{\textnormal{RCC2}}=\mathbf{A}^+\mathbf{b}+\mathbf{Wu}^*$. In (\ref{f4}), we have used the fact that $\mathbf{W}^T\mathbf{A}^+\mathbf{A}=0$ and $\mathbf{W}^T\mathbf{W}=\mathbf{I}.$

Finally, since the objective of (\ref{cheb22cons}) is strictly concave, we have that $\mathbf{u}^*$, and hence, $\hat{\mathbf{X}}_\textnormal{RCC2}$ are unique. Moreover, due to the constraint in (\ref{cheb22cons}), we have $\hat{\mathbf{X}}_\textnormal{RCC2}\in\mathcal{S_F}$.
\end{proof}
Denoting the MSE of a certain estimate $\hat{\mathbf{X}}$ by $\textnormal{MSE}(\hat{\mathbf{X}})$, the following theorem provides a relationship between some of the estimates introduced thus far.
\begin{theorem}\label{TH1}
The following inequalities hold.
\begin{equation}
    \textnormal{MSE}(\hat{\mathbf{X}}_\textnormal{RCC2})\leq\textnormal{MSE}(\hat{\mathbf{X}}_{\textnormal{half}^*})\leq\textnormal{MSE}(\hat{\mathbf{X}}_\textnormal{half}).
\end{equation}
\end{theorem}
\begin{proof}
In order to prove the first inequality, we proceed as follows. The derivative of the objective of (\ref{cheb22cons}) with respect to $\mathbf{u}$ is
\begin{equation}
    \frac{d}{d\mathbf{u}}(\mathbf{1}^T\mathbf{Wu}-\|\mathbf{u}\|^2)=\mathbf{W}^T\mathbf{1}-2\mathbf{u}.
\end{equation}
Since the objective in (\ref{cheb22cons}) is (strictly) concave in $\mathbf{u}$, by setting $\frac{d}{d\mathbf{u}}(\cdot)=0$, we obtain $\mathbf{u}'=\frac{1}{2}\mathbf{W}^T\mathbf{1}$ as the maximizer. It is important to note that this $\mathbf{u}'$ is not the solution of (\ref{cheb22cons}), i.e., $\mathbf{u}^*$, in general, as it might not satisfy its constraints. Define $\hat{\mathbf{X}}'\triangleq \mathbf{A}^+\mathbf{b}+\mathbf{Wu}'$. We have
\begin{align}
    \hat{\mathbf{X}}'&=\mathbf{A}^+\mathbf{b}+\frac{1}{2}\mathbf{WW}^T\mathbf{1}\nonumber\\
    &=\mathbf{A}^+\mathbf{b}+\frac{1}{2}(\mathbf{I}-\mathbf{A}^+\mathbf{A})\mathbf{1}\label{ch}\\
    &=\hat{\mathbf{X}}_{\textnormal{half}^*},
\end{align}
where the equality $\mathbf{I}-\mathbf{A}^+\mathbf{A}=\mathbf{WW}^T$ follows from the definition of $\mathbf{W}$.

If $\mathbf{u}'$ satisfies the constraints of (\ref{cheb22cons}), then $\mathbf{u}^*=\mathbf{u}'$, and $\hat{\mathbf{X}}_{\textnormal{RCC2}}=\hat{\mathbf{X}}_{\textnormal{half}^*}$, otherwise, we have that $\mathbf{u}^*$ is the point in $\tilde{\mathcal{S_F}}$ that is closest to $\mathbf{u}'$, and as a result, $\hat{\mathbf{X}}_{\textnormal{RCC2}}$ is the point in $\mathcal{S_F}$ that is closest to $\hat{\mathbf{X}}_{\textnormal{half}^*}$. This is justified as follows.
\begin{align}
    \mathbf{u}^*&=\argmax_{\mathbf{u}\in\tilde{\mathcal{S_F}}}\{\mathbf{1}^T\mathbf{Wu}-\|\mathbf{u}\|^2\}\nonumber\\
    &=\argmin_{\mathbf{u}\in\tilde{\mathcal{S_F}}}\{\|\mathbf{u}\|^2-\mathbf{1}^T\mathbf{Wu}\}\nonumber\\
    &=\argmin_{\mathbf{u}\in\tilde{\mathcal{S_F}}}\{\|\mathbf{u}\|^2-\mathbf{1}^T\mathbf{Wu}+\|\frac{1}{2}\mathbf{W}^T\mathbf{1}\|^2\}\nonumber\\
    &=\argmin_{\mathbf{u}\in\tilde{\mathcal{S_F}}}\|\mathbf{u}-\frac{1}{2}\mathbf{W}^T\mathbf{1}\|^2\nonumber\\
    &=\argmin_{\mathbf{u}\in\tilde{\mathcal{S_F}}}\|\mathbf{u}-\mathbf{u}'\|^2\nonumber,
\end{align}
which results in
% \begin{align}
%     \hat{\mathbf{X}}_{\textnormal{RCC2}}&=\mathbf{A}^+\mathbf{b}+\mathbf{Wu}^*\nonumber\\
%     &=\mathbf{A}^+\mathbf{b}+\mathbf{W}(\argmax_{\mathbf{u}\in\tilde{\mathcal{S_F}}}\{\mathbf{1}^T\mathbf{Wu}-\|\mathbf{u}\|^2\})\nonumber\\
% \end{align}
$\hat{\mathbf{X}}_{\textnormal{RCC2}}=\argmin_{\mathbf{x}\in\mathcal{S_F}}\|\mathbf{x}-\hat{\mathbf{X}}_{\textnormal{half}^*}\|^2$. Hence, we can write\footnote{This follows from the fact that if $\mathcal{C}$ is a nonempty convex subset of $\mathds{R}^d$ and $f:\mathds{R}^d\to\mathds{R}$ a convex and differentiable function, then we have $\mathbf{x}^*=\argmin_{x\in\mathcal{C}}f(\mathbf{x})$ if and only if $(\nabla f(\mathbf{x}^*))^T(\mathbf{x}-\mathbf{x}^*)\geq 0,\forall \mathbf{x}\in\mathcal{C}$. (\ref{inq1}) can be obtained by replacing $\mathcal{C}$ and $\mathbf{x}^*$ with $\mathcal{S_F}$ and $\hat{\mathbf{X}}_{\textnormal{RCC2}}$, respectively, and noting that $f(\mathbf{x})=\|\mathbf{x}-\hat{\mathbf{X}}_{\textnormal{half}^*}\|^2$ and $\nabla f(\mathbf{x}^*)=2(\hat{\mathbf{X}}_{\textnormal{RCC2}}-\hat{\mathbf{X}}_{\textnormal{half}^*})$. }
\begin{equation}\label{inq1}
    (\hat{\mathbf{X}}_\textnormal{RCC2}-\hat{\mathbf{X}}_{\textnormal{half}^*})^T(\mathbf{x}-\hat{\mathbf{X}}_\textnormal{RCC2})\geq 0,\ \forall \mathbf{x}\in\mathcal{S_F},
\end{equation}
which results in the following inequality for $\mathbf{X}\in\mathcal{S_F}$
\begin{equation}
    \|\mathbf{X}-\hat{\mathbf{X}}_{\textnormal{RCC2}}\|^2\leq\|\mathbf{X}-\hat{\mathbf{X}}_{\textnormal{half}^*}\|^2.
\end{equation}
Finally, by taking the expectation of both sides, we obtain
\begin{equation}
    \textnormal{MSE}(\hat{\mathbf{X}}_\textnormal{RCC2})\leq\textnormal{MSE}(\hat{\mathbf{X}}_{\textnormal{half}^*}).
\end{equation}
The proof of the second inequality is straightforward. We have $\mathbf{X}-\hat{\mathbf{X}}_{\textnormal{half}}=\mathbf{X}-\hat{\mathbf{X}}_{\textnormal{half}^*}+\hat{\mathbf{X}}_{\textnormal{half}^*}-\hat{\mathbf{X}}_{\textnormal{half}}$. The proof is concluded by showing that $\mathbf{X}-\hat{\mathbf{X}}_{\textnormal{half}^*}$ is orthogonal to $\hat{\mathbf{X}}_{\textnormal{half}^*}-\hat{\mathbf{X}}_{\textnormal{half}}$, i.e., $(\mathbf{X}-\hat{\mathbf{X}}_{\textnormal{half}^*})^T(\hat{\mathbf{X}}_{\textnormal{half}^*}-\hat{\mathbf{X}}_{\textnormal{half}})=0$. By noting that $\mathbf{X}$ can be written as $\mathbf{A}^+\mathbf{b}+(\mathbf{I}-\mathbf{A}^+\mathbf{A})\mathbf{w}$ for some $\mathbf{w}\in\mathds{R}^d$, we have
\begin{align}
    \mathbf{X}-\hat{\mathbf{X}}_{\textnormal{half}^*}&=(\mathbf{I}-\mathbf{A}^+\mathbf{A})(\mathbf{w}-\frac{1}{2}\mathbf{1})\label{avval}\\
    \hat{\mathbf{X}}_{\textnormal{half}^*}-\hat{\mathbf{X}}_{\textnormal{half}}&=\mathbf{A}^+\mathbf{A}(\mathbf{X}-\frac{1}{2}\mathbf{1}).\label{dovvom}
\end{align}
Knowing that $(\mathbf{I}-\mathbf{A}^+\mathbf{A})\mathbf{A}^+\mathbf{A}=0$ proves the orthogonality of the LHS of (\ref{avval}) and (\ref{dovvom}). Hence, we have
\begin{align}
    \|\mathbf{X}-\hat{\mathbf{X}}_{\textnormal{half}}\|^2&=\|\mathbf{X}-\hat{\mathbf{X}}_{\textnormal{half}^*}\|^2+\|\hat{\mathbf{X}}_{\textnormal{half}^*}-\hat{\mathbf{X}}_{\textnormal{half}}\|^2\nonumber\\
    &\geq \|\mathbf{X}-\hat{\mathbf{X}}_{\textnormal{half}^*}\|^2,
\end{align}
and by taking the expectation, we get
\begin{equation}\label{eq:Half}
    \textnormal{MSE}(\hat{\mathbf{X}}_{\textnormal{half}^*})\leq \textnormal{MSE}(\hat{\mathbf{X}}_{\textnormal{half}}).
\end{equation}
\end{proof}
\begin{remark}
(An alternative characterization of RCC2) In (\ref{xhalf}), $\hat{\mathbf{X}}_{\textnormal{half}^*}$ is defined as the point in the solution space, i.e., $\mathcal{S}$, that is closest to $\frac{1}{2}\mathbf{1}_d$. Interestingly, we observe that $\hat{\mathbf{X}}_{\textnormal{RCC2}}$, which is independently defined as the second relaxation of the Chebyshev center of $\mathcal{S_F}$, can be interpreted in a similar way: it is the point in the feasible solution space $\mathcal{S_F}$ that is closest to $\frac{1}{2}\mathbf{1}_d$, which is justified as follows.
\begin{align*}
    \argmin_{{\mathbf{x}}\in\mathcal{S_F}}\|{\mathbf{x}}-\frac{1}{2}\mathbf{1}_d\|^2&=\mathbf{A}^+\mathbf{b} +\mathbf{W}\left(\argmin_{\mathbf{u}\in\tilde{\mathcal{S_F}}}\|\mathbf{A}^+\mathbf{b}+\mathbf{Wu}-\frac{1}{2}\mathbf{1}_d\|^2\right)\\
    &=\mathbf{A}^+\mathbf{b} +\mathbf{W}\left(\argmin_{\mathbf{u}\in\tilde{\mathcal{S_F}}}\left\{\|\mathbf{Wu}\|^2+2\mathbf{u}^T\mathbf{W}^T(\mathbf{A}^+b-\frac{1}{2}\mathbf{1}_d)+\|\mathbf{A}^+b-\frac{1}{2}\mathbf{1}_d\|^2\right\}\right)\\
    &=\mathbf{A}^+\mathbf{b} +\mathbf{W}\left(\argmin_{\mathbf{u}\in\tilde{\mathcal{S_F}}}\left\{\|\mathbf{u}\|^2-\mathbf{1}^T\mathbf{Wu}\right\}\right)\\
    &=\mathbf{A}^+\mathbf{b} +\mathbf{W}\left(\argmax_{\mathbf{u}\in\tilde{\mathcal{S_F}}}\left\{\mathbf{1}^T\mathbf{Wu}-\|\mathbf{u}\|^2\right\}\right),
\end{align*}
which is the same as (\ref{cheb22}).
\end{remark}
Let $\mathbf{z}$ be an arbitrary vector in $\mathds{R}^d$ and define $\mathbf{K}_\mathbf{z}\triangleq \mathds{E}[(\mathbf{X}-\mathbf{z})(\mathbf{X}-\mathbf{z})^T]$. Define $\mathbf{\mu}\triangleq\mathds{E}[\mathbf{X}]$. In particular, $\mathbf{K}_\mathbf{0}$ and $\mathbf{K}_{\mathbf{\mu}}$ denote the correlation and covariance matrices of $\mathbf{X}$, respectively.
\begin{theorem}\label{TH2}
The following relationships hold.
\begin{align}
    \frac{1}{d}\sum_{i=\textnormal{rank}(\mathbf{A})+1}^{d}\lambda_i\leq\textnormal{MSE}(\hat{\mathbf{X}}_\textnormal{LS})&=\frac{\textnormal{Tr}\left((\mathbf{I}-\mathbf{A}^+\mathbf{A})\mathbf{K}_\mathbf{0}\right)}{d}\leq\frac{1}{d}\sum_{i=1}^{\textnormal{nul}(\mathbf{A})}\lambda_i\label{ineq1}\\
    \frac{1}{d}\sum_{i=\textnormal{rank}(\mathbf{A})+1}^{d}\gamma_i\leq\textnormal{MSE}(\hat{\mathbf{X}}_{\textnormal{half}^*})&=\frac{\textnormal{Tr}\left((\mathbf{I}-\mathbf{A}^+\mathbf{A})\mathbf{K}_{\frac{1}{2}\mathbf{1}}\right)}{d}\leq\frac{1}{d}\sum_{i=1}^{\textnormal{nul}(\mathbf{A})}\gamma_i\label{ineq2}\\
    \textnormal{MSE}(\hat{\mathbf{X}}_\textnormal{LS})\ ,\ \textnormal{MSE}(\hat{\mathbf{X}}_{\textnormal{half}^*})&\geq\frac{\textnormal{Tr}\left((\mathbf{I}-\mathbf{A}^+\mathbf{A})\mathbf{K}_{\mathbf{\mu}}\right)}{d}\label{ineq3},
\end{align}
where $\lambda_1\geq\lambda_2\geq\ldots\geq\lambda_d$ and $\gamma_1\geq\gamma_2\geq\ldots\geq\gamma_d$ are, respectively, the eigenvalues of $\mathbf{K}_\mathbf{0}$ and $\mathbf{K}_{\frac{1}{2}\mathbf{1}}$ arranged in a non-increasing order. We also adopt the convention that $\sum_{i=1}^0(\cdot)_i,\sum_{i=d+1}^d(\cdot)_i\triangleq 0$.
\end{theorem}
\begin{proof}
We have
\begin{align}
    \textnormal{MSE}(\hat{\mathbf{X}}_{\textnormal{LS}})&=\frac{1}{d}\mathds{E}[\|\mathbf{X}-\hat{\mathbf{X}}_\textnormal{LS}\|^2]\nonumber\\
    &=\frac{1}{d}\mathds{E}[\|\mathbf{X}-\mathbf{A}^+\mathbf{b}\|^2]\nonumber\\
    &=\frac{1}{d}\mathds{E}[\|\mathbf{X}-\mathbf{A}^+\mathbf{A}\mathbf{X}\|^2]\nonumber\\
    &=\frac{1}{d}\mathds{E}[\|(\mathbf{I}-\mathbf{A}^+\mathbf{A})\mathbf{X}\|^2]\nonumber\\
    &=\frac{1}{d}\mathds{E}[\textnormal{Tr}\left((\mathbf{I}-\mathbf{A}^+\mathbf{A})\mathbf{XX}^T(\mathbf{I}-\mathbf{A}^+\mathbf{A})^T\right)]\label{toz1}\\
    &=\frac{1}{d}\mathds{E}[\textnormal{Tr}\left((\mathbf{I}-\mathbf{A}^+\mathbf{A})\mathbf{XX}^T\right)]\label{toz2}\\
    &=\frac{1}{d}\textnormal{Tr}\left((\mathbf{I}-\mathbf{A}^+\mathbf{A})\mathbf{K}_\mathbf{0}\right)\label{toz3},
\end{align}
where (\ref{toz1}) follows from having $\|\mathbf{a}\|^2=\textnormal{Tr}(\mathbf{aa}^T)$ for an arbitrary vector $\mathbf{a}$. In (\ref{toz2}), we use the invariance of trace under cyclic permutation (in particular $\textnormal{Tr}(\mathbf{AB})=\textnormal{Tr}(\mathbf{BA})$) and the fact that $\mathbf{I}-\mathbf{A}^+\mathbf{A}$ is an orthogonal projection, i.e., it is symmetric and $(\mathbf{I}-\mathbf{A}^+\mathbf{A})^2=\mathbf{I}-\mathbf{A}^+\mathbf{A}$. By pushing the expectation inside the trace in (\ref{toz2}), which is due to the linearity of the trace operator, (\ref{toz3}) is obtained.

Similarly, we have
\begin{align*}
    \textnormal{MSE}(\hat{\mathbf{X}}_{\textnormal{half}^*})&=\frac{1}{d}\mathds{E}[\|\mathbf{X}-\hat{\mathbf{X}}_{\textnormal{half}^*}\|^2]\nonumber\\
    &=\frac{1}{d}\mathds{E}[\|\mathbf{X}-\mathbf{A}^+\mathbf{b}-\frac{1}{2}(\mathbf{I}-\mathbf{A}^+\mathbf{A})\mathbf{1}\|^2]\nonumber\\
    &=\frac{1}{d}\mathds{E}[\|\mathbf{X}-\mathbf{A}^+\mathbf{A}\mathbf{X}-\frac{1}{2}(\mathbf{I}-\mathbf{A}^+\mathbf{A})\mathbf{1}\|^2]\nonumber\\
    &=\frac{1}{d}\mathds{E}[\|(\mathbf{I}-\mathbf{A}^+\mathbf{A})(\mathbf{X}-\frac{1}{2}\mathbf{1})\|^2]\nonumber\\
    &=\frac{1}{d}\mathds{E}\bigg[\textnormal{Tr}\left((\mathbf{I}-\mathbf{A}^+\mathbf{A})(\mathbf{X}-\frac{1}{2}\mathbf{1})(\mathbf{X}-\frac{1}{2}\mathbf{1})^T(\mathbf{I}-\mathbf{A}^+\mathbf{A})^T\right)\bigg]
\end{align*}
\begin{align*}
    &=\frac{1}{d}\mathds{E}\bigg[\textnormal{Tr}\left((\mathbf{I}-\mathbf{A}^+\mathbf{A})(\mathbf{X}-\frac{1}{2}\mathbf{1})(\mathbf{X}-\frac{1}{2}\mathbf{1})^T\right)\bigg]\\
    &=\frac{1}{d}\textnormal{Tr}\left((\mathbf{I}-\mathbf{A}^+\mathbf{A})\mathbf{K}_{\frac{1}{2}\mathbf{1}}\right).
\end{align*}
Hence, the equalities in (\ref{ineq1}) and (\ref{ineq2}) are proved.

Fix an arbitrary $\mathbf{z}\in\mathds{R}^d$, which results in $\mathbf{K}_\mathbf{z}$ with eigenvalues denoted by $s_1\geq s_2\geq\ldots\geq s_d$. By applying Von Neumann's trace inequality, we have
\begin{equation}\label{von}
    \sum_{i=\textnormal{rank}(\mathbf{A})+1}^ds_i\leq\textnormal{Tr}\left((\mathbf{I}-\mathbf{A}^+\mathbf{A})\mathbf{K}_{\mathbf{z}}\right)\leq\sum_{i=1}^{\textnormal{nul}(\mathbf{A})}s_i,
\end{equation}
which follows from the fact that both $\mathbf{K}_\mathbf{z}$ and $\mathbf{I}-\mathbf{A}^+\mathbf{A}$ are symmetric positive semidefinite matrices and the latter has $\textnormal{nul}(\mathbf{A})$ 1's and $\textnormal{rank}(\mathbf{A})$ 0's as eigenvalues. By replacing $\mathbf{z}$ with $\mathbf{0}$ or $\frac{1}{2}\mathbf{1}$, the upper and lower bounds in (\ref{ineq1}) or (\ref{ineq2}) are obtained.

For a fixed $\mathbf{z}$, we have
\begin{align}
    \textnormal{Tr}\left((\mathbf{I}-\mathbf{A}^+\mathbf{A})\mathbf{K}_{\mathbf{z}}\right)-\textnormal{Tr}\left((\mathbf{I}-\mathbf{A}^+\mathbf{A})\mathbf{K}_{\mathbf{\mu}}\right)&=\textnormal{Tr}\left((\mathbf{I}-\mathbf{A}^+\mathbf{A})(\mathbf{K}_{\mathbf{z}}-\mathbf{K}_{\mathbf{\mu}})\right)\nonumber\\
    &=\textnormal{Tr}\left((\mathbf{I}-\mathbf{A}^+\mathbf{A})(\mathbf{z}-\mathbf{\mu})(\mathbf{z}-\mathbf{\mu})^T\right)\nonumber\\
    &=\textnormal{Tr}\left((\mathbf{z}-\mathbf{\mu})^T(\mathbf{I}-\mathbf{A}^+\mathbf{A})(\mathbf{z}-\mathbf{\mu})\right)\nonumber\\
    &=(\mathbf{z}-\mathbf{\mu})^T(\mathbf{I}-\mathbf{A}^+\mathbf{A})(\mathbf{z}-\mathbf{\mu})\nonumber\\
    &\geq 0\label{psd},
\end{align}
where (\ref{psd}) follows from the positive semidefiniteness of $\mathbf{I}-\mathbf{A}^+\mathbf{A}$. Replacing $\mathbf{z}$ with $\mathbf{0}$ or $\frac{1}{2}\mathbf{1}$ results in (\ref{ineq3}).\footnote{There is a simpler proof for (\ref{ineq3}) by noting that when $\mathbf{A}\succeq\mathbf{B}$ and $\mathbf{Q}\succeq 0$, we have $\textnormal{Tr}(\mathbf{QA})\geq\textnormal{Tr}(\mathbf{QB})$. Knowing that for any $\mathbf{z}$, we have $\mathbf{K}_\mathbf{z}\succeq\mathbf{K}_\mathbf{\mu}$, it then suffices to replace $\mathbf{Q}$ with $\mathbf{I}-\mathbf{A}^+\mathbf{A}$ and $\mathbf{A},\mathbf{B}$ with $\mathbf{K}_\mathbf{z},\mathbf{K}_\mathbf{\mu}$, respectively.}
\end{proof}
\begin{remark}
From Theorem \ref{TH2}, the passive party can obtain the MSE of the attacks $\hat{\mathbf{X}}_{\textnormal{LS}}$ and $\hat{\mathbf{X}}_{\textnormal{half}^*}$ as closed form solutions. It is important to note that in this context, this is still possible although the passive party is unaware of the active party's model parameters and the confidence scores it receives. We also note that according to remark \ref{Rem1}, although the adversary has multiple ways to obtain a system of linear equations, all of them are equivalent. As a result, the passive party can assume that the adversary has obtained this system in a particular way, and obtain the MSE. In other words, regardless of whether the adversary is dealing with $\mathbf{AX}=\mathbf{b}$ or $\mathbf{A}_\textnormal{new}\mathbf{X}=\mathbf{b}_\textnormal{new}$, with $\mathbf{A}_\textnormal{new}=\mathbf{RA}$ and $\mathbf{b}_{new}=\mathbf{Rb}$ for arbitrary invertible $\mathbf{R}$, we have $(\mathbf{I}-\mathbf{A}^+\mathbf{A})=(\mathbf{I}-\mathbf{A}_\textnormal{new}^+\mathbf{A}_\textnormal{new})$, which results from i) $\mathbf{R}$ is invertible and ii) $\mathbf{A}$ has linearly independent rows, and hence $(\mathbf{RA})^+=\mathbf{A}^+\mathbf{R}^+=\mathbf{A}^+\mathbf{R}^{-1}$.
\end{remark}
\begin{remark}
In many practical scenarios, we have that $\mathbf{A}$ is either full column or full row rank, which results in $\textnormal{rank}(\mathbf{A})=\min\{k-1,d\}$ (and hence $\textnormal{nul}(\mathbf{A})=\max\{d-k+1,0\}$), where $k$ denotes the number of classes. In this context, the importance of the lower and upper bounds in (\ref{ineq1}) and (\ref{ineq2}) is that the passive party can calculate them prior to the training, which can be carried out by calculating the eigenvalues of $\mathbf{K}_\mathbf{0}$ and $\mathbf{K}_{\frac{1}{2}\mathbf{1}}$.
\end{remark}
\begin{remark}
The attack schemes proposed in this section are applied per prediction, i.e., an estimate is obtained after the receipt of the confidence scores for each sample in the prediction set. It is easy to verify that these attacks result in the same performance if applied on multiple predictions. More specifically, assuming that there are $N$ predictions, the resulting system of linear equations is $\mathbf{AX}=\mathbf{b}$, in which $\mathbf{X}$ is an $Nd$-dimensional vector obtained as the concatenation of $N$ $d$-dimensional vectors, $\mathbf{A}$ is $N(k-1)\times Nd$-dimensional, and $\mathbf{b}$ is $N(k-1)$-dimensional. Nonetheless, all the attack methods discussed, do not improve by being applied on multiple predictions.
\end{remark}
\section{Black-box setting}\label{bbs}
A relaxed version of the black-box setting is considered in \cite{Jiang}, in which the adversary is aware of some auxiliary data, i.e., the passive party's features for some sample IDs, and based on these auxiliary data, the adversary estimates the passive party's model parameters. Once this estimate is obtained, any reconstruction attack in the white-box setting can be applied by regarding this estimate as the true model parameters. Needless to say that all the proposed attacks in the previous section of this paper can be applied in this way. However, the real black-box setting in which the adversary cannot have access to auxiliary data remains open. In what follows, we investigate this problem under specific circumstances.

Here, we assume that the passive party has only one feature denoted by $X_i\in[0,1], i\in[N]$, corresponding to $N$ predictions. We assume that $(X_1,X_2,\ldots,X_N)$ are i.i.d. according to an unknown CDF $F_X$. In the black-box setting, the adversary observes $v_i=\omega X_i +b, i\in[N]$, where $\omega$ ($\neq 0$) and $b$ are unknown. This is a specific case of (\ref{eq:1}), where $d=1$. A question that arises here is : How is the performance of the adversary affected by the lack of knowledge about $\omega,b$? In other words, what (minimal) knowledge about $\omega,b$ is sufficient for the adversary in order to perform a successful reconstruction attack in estimating $(X_1,X_2,\ldots,X_N)$? In what follows, it is shown that in certain scenarios, this lack of knowledge has a vanishing effect given that $N$ is large enough.
\begin{lemma}\label{lem2}
Assume that $(X_1,X_2,\ldots,X_N)$ are i.i.d. according to an unknown CDF $F_X$, where $X_i\in[0,1],i\in[N]$. Fix an arbitrary $\epsilon\in(0,1)$. We have
\begin{align}
    \lim_{N\to\infty}\textnormal{Pr}\left\{\max_{i\in[N]}X_i\leq 1-\epsilon\right\}&=0\nonumber\\
    \lim_{N\to\infty}\textnormal{Pr}\left\{\min_{i\in[N]}X_i\geq \epsilon\right\}&=0.
\end{align}
In other words $\max_{i\in[N]}X_i$ and $\min_{i\in[N]}X_i$ converge in probability to 1 and 0, respectively.
\end{lemma}
\begin{proof}
We have
\begin{align}
     \lim_{N\to\infty}\textnormal{Pr}\left\{\max_{i\in[N]}X_i\leq 1-\epsilon\right\}&= \lim_{N\to\infty}\textnormal{Pr}\left\{X_i\leq 1-\epsilon,\ \forall i\in[N]\right\}\nonumber\\
    &= \lim_{N\to\infty}\bigg(\textnormal{Pr}\{X_{1}\leq 1-\epsilon\}\bigg)^N\label{wq3}\\
    &= \lim_{N\to\infty}(F_X(1-\epsilon))^N\nonumber\\
    &=0,\label{akh1}
\end{align}
and
\begin{align}
     \lim_{N\to\infty}\textnormal{Pr}\left\{\min_{i\in[N]}X_i\geq \epsilon\right\}&= \lim_{N\to\infty}\textnormal{Pr}\left\{X_i\geq \epsilon,\ \forall i\in[N]\right\}\nonumber\\
    &= \lim_{N\to\infty}\bigg(\textnormal{Pr}\{X_{1}\geq \epsilon\}\bigg)^N\label{wq33}\\
    &= \lim_{N\to\infty}\left(1-\lim_{t\to\epsilon^-}F_X(t)\right)^N\label{wq34}\\
    &=0,\label{akh2}
\end{align}
where (\ref{wq3}) and (\ref{wq33}) are due to the assumption that $X_i$'s are i.i.d., and (\ref{wq34}) results from the identities $\textnormal{Pr}\{X_1\geq \epsilon\}=1-\textnormal{Pr}\{X_1<\epsilon\}$ and $\textnormal{Pr}\{X_1<\epsilon\}=\lim_{t\to\epsilon^-}F_X(t)$.
Since we are assuming that $X_i\in[0,1],\forall i\in[N]$, we have that for any $\alpha\in(0,1)$, $0<F_X(\alpha)<1$, since otherwise, the region $[0,1]$ could have been modified accordingly. This results in (\ref{akh1}) and (\ref{akh2}).
\end{proof}
The adversary observes $v_i$'s and the problem is divided into three cases as follows.
\subsection{Case 1 : $b=0$}
% since $X_i\in[0,1], \forall i\in[N]$, we have that either $b_i\geq 0,\forall i\in[N]$ (corresponding to $\omega>0$) or $b_i\leq 0,\forall i\in[N]$ (corresponding to $\omega<0$). Regardless of these two cases,
In this case, the observations of the adversary are $v_i=\omega X_i,i\in[N]$. The adversary finds the maximum of $|v_i|$ and estimates that the feature in charge of generating this value is 1. In other words, let $M\triangleq \argmax_{i\in[N]}|v_i|$, and the adversary sets $\hat{X}_{M}=1$. \footnote{If there are more than one maximizer, pick one arbitrarily as $M$.} The rationale behind this estimation is that if $N$ is large enough, we are expecting $X_{M}$ to be close to 1 by lemma \ref{lem2}. By design, we have that $\frac{X_i}{X_{M}}=\frac{v_i}{v_{M}}, \forall i\in[N]$. Therefore, it makes sense to set
\begin{equation}
    \hat{X}_i=\frac{v_i}{v_{M}}\hat{X}_{M}=\frac{v_i}{v_{M}}, \forall i\in[N].
\end{equation}
With these estimates, we can write the empirical MSE as
\begin{align}
    \frac{1}{N}\sum_{i=1}^N(X_i-\hat{X}_i)^2&=\frac{1}{N}\sum_{i=1}^N(\frac{v_i}{v_{M}}X_{M}-\frac{v_i}{v_{M}})^2\nonumber\\
    &=\frac{\sum_{i=1}^Nv_i^2}{Nv_{M}^2}(X_{M}-1)^2\nonumber\\
    &\leq (X_{M}-1)^2\label{ak},
\end{align}
where (\ref{ak}) is due to $|v_i|\leq|v_M|,i\in[N]$. Therefore, the empirical MSE is upperbounded by the error in our first estimate, i.e., how close $X_M$ is to 1.

Fix arbitrary $\epsilon>0$. We can write
\begin{align}
    \lim_{N\to\infty}\textnormal{Pr}\{\frac{1}{N}\sum_{i=1}^N(X_i-\hat{X}_i)^2\geq \epsilon\}&\leq \lim_{N\to\infty}\textnormal{Pr}\{(X_{M}-1)^2\geq\epsilon\}\label{wq1}\\
    &=\lim_{N\to\infty}\textnormal{Pr}\{X_{M}\leq 1-\sqrt{\epsilon}\}\label{wq2}\\
    &=\lim_{N\to\infty}\textnormal{Pr}\left\{\max_{i\in[N]}X_{i}\leq 1-\sqrt{\epsilon}\right\}\nonumber\\
    &=0,\label{kh1}
\end{align}
where (\ref{wq1}) follows from (\ref{ak}), and (\ref{wq2}) follows from $X_{M}\in[0,1]$. Finally, (\ref{kh1}) results from lemma \ref{lem2}.
%  Hence, we have that for any $\epsilon>0$,
% \begin{align}
%   \lim_{N\to\infty}\textnormal{Pr}\bigg\{\frac{1}{N}\sum_{i=1}^N(X_i-\hat{X}_i)^2\geq \epsilon\bigg\} &\leq \lim_{N\to\infty}(F_X(1-\sqrt{\epsilon}))^N\nonumber\\
%   & = 0.\nonumber
% \end{align}

Therefore, the empirical MSE of the adversary converges in probability to 0 with the number of predictions $N$. This means that in this context, the lack of knowledge of the parameter has a vanishingly small effect.

\subsection{Case 2: $\omega b>0$}
In this case, the observations of the adversary are $v_i=\omega X_i + b,i\in[N]$, where $\omega$, $b$ and $v_i$'s have the same sign. Let $M\triangleq\argmax_{i\in[N]}|v_i|$ and $m\triangleq\argmin_{i\in[N]}|v_i|$.\footnote{If there are more than one maximizer/minimizer, pick one arbitrarily.} The adversary estimates $\hat{X}_M=1$ and $\hat{X}_m=0$. Let $\mathcal{I}\triangleq\{i\in[N]|v_i\neq v_m\}$, and define
\begin{equation*}
    \alpha_i\triangleq\frac{X_i-X_M}{X_i-X_m},\forall i\in\mathcal{I}.
\end{equation*}
By design, we have
\begin{equation*}
 \alpha_i=\frac{v_i-v_M}{v_i-v_m}\leq 0,\forall i\in\mathcal{I}.
\end{equation*}
Therefore, we have that $X_i=\frac{X_M-\alpha_iX_m}{1-\alpha_i},\forall i\in\mathcal{I}$ and $X_i=X_m,\forall i\in[N]\backslash\mathcal{I}$. The adversary sets $\hat{X}_i=\frac{\hat{X}_M-\alpha_i\hat{X}_m}{1-\alpha_i}=\frac{1}{1-\alpha_i},\forall i\in\mathcal{I}$ and $\hat{X}_i=\hat{X}_m=0,\forall i\in[N]\backslash\mathcal{I}$. With these estimates, the empirical MSE is given by
\begin{align}
    \frac{1}{N}\sum_{i=1}^N(X_i-\hat{X}_i)^2&=\frac{1}{N}\sum_{i\in\mathcal{I}}(X_i-\hat{X}_i)^2+\frac{1}{N}\sum_{i\in[N]\backslash\mathcal{I}}(X_i-\hat{X}_i)^2\nonumber\\
    &=\frac{1}{N}\sum_{i\in\mathcal{I}}\left(\frac{X_M-1-\alpha_iX_m}{1-\alpha_i}\right)^2+\frac{1}{N}\sum_{i\in[N]\backslash\mathcal{I}}X_m^2\label{subs}\\
    &=\frac{\sum_{i\in\mathcal{I}}\frac{1}{(1-\alpha_i)^2}}{N}(X_M-1)^2+\frac{1}{N}\left(N-|\mathcal{I}|+\sum_{i\in\mathcal{I}}(\frac{\alpha_i}{1-\alpha_i})^2\right)X_m^2\nonumber\\
    &\ \ \ +\frac{1}{N}\left(\sum_{i\in\mathcal{I}}\frac{-2\alpha_i}{(1-\alpha_i)^2}\right)(X_M-1)X_m\label{ko1}\\
    &\leq (X_M-1)^2+X_m^2\label{ko2},
\end{align}
where (\ref{ko2}) is justified as follows. Since $\alpha_i\leq 0,\forall i\in\mathcal{I}$, and $|\mathcal{I}|\leq N$, the coefficients of $(X_M-1)^2$ and $X_m^2$ in (\ref{ko1}) are both upper bounded by 1. Moreover, since $(X_M-1)\leq 0$, the third term in (\ref{ko1}) is non-positive, which results in (\ref{ko2}).

Fix arbitrary $\epsilon>0$. We have
\begin{align}
    \lim_{N\to\infty}\textnormal{Pr}\{\frac{1}{N}\sum_{i=1}^N(X_i-\hat{X}_i)^2\geq \epsilon\}&\leq \lim_{N\to\infty}\textnormal{Pr}\{(X_{M}-1)^2+X_m^2\geq\epsilon\}\label{rr1}\\
    &\leq\lim_{N\to\infty}\textnormal{Pr}\{(X_{M}-1)^2\geq\frac{\epsilon}{2}\cup X_m^2\geq\frac{\epsilon}{2}\}\label{rr2}\\
    &\leq \lim_{N\to\infty}\textnormal{Pr}\{(X_{M}-1)^2\geq\frac{\epsilon}{2}\} +\textnormal{Pr}\{X_m^2\geq\frac{\epsilon}{2}\}\label{rr3}\\
    &=\lim_{N\to\infty}\textnormal{Pr}\left\{X_{M}\leq1-\sqrt{\frac{\epsilon}{2}}\right\}+\textnormal{Pr}\left\{X_m\geq\sqrt{\frac{\epsilon}{2}}\right\}\nonumber\\
    &=\lim_{N\to\infty}\textnormal{Pr}\left\{\max_{i\in[N]}X_i\leq1-\sqrt{\frac{\epsilon}{2}}\right\}\nonumber\\
    &\ \ \ +\lim_{N\to\infty}\textnormal{Pr}\left\{\min_{i\in[N]}X_i\geq\sqrt{\frac{\epsilon}{2}}\right\}\nonumber\\
    &=0,\label{rr4}
\end{align}
where (\ref{rr1}) follows from (\ref{ko2}), and (\ref{rr2}) is from the fact that for two random variables $A,B$, the event $\{A+B\geq\epsilon\}$ is a subset of $\{A\geq\frac{\epsilon}{2}\}\cup\{B\geq\frac{\epsilon}{2}\}$. (\ref{rr3}) is the application of Boole's inequality, i.e., the union bound, and finally, (\ref{rr4}) results from lemma \ref{lem2}.

Again, the empirical MSE of the adversary converges in probability to 0, which means that in this context, the lack of knowledge of the parameters has a vanishingly small effect.

\subsection{Case 3: $\omega b<0$}
In this case, the observations of the adversary are $v_i=\omega X_i + b,i\in[N]$, where $\omega$ and $b$ have different signs. This case is more involved and can be divided into two scenarios as follows.
\subsubsection{All the $v_i$'s have the same sign}
In this case, the adversary concludes that the sign of $b$ is the same as that of $v_i$'s, since if $N$ is large enough, for some $i\in[N]$, we have $X_i\approx 0$ and its corresponding $v_i$ is close to $b$. Also, since we have $\omega b<0$, the sign of $\omega$ is inferred. Now that the signs of $\omega$ and $b$ are known to the adversary, following a similar approach as in the previous subsection, it can be shown that MSE converges in probability to 0.
\subsubsection{The $v_i$'s do not have the same sign}
In this case, the adversary cannot decide between $\omega>0,b<0$ and $\omega<0,b>0$. It is, however, easy to show that in one case the adversary's estimates are close to the real values, i.e., $\hat{X}_i\approx X_i,i\in[N]$, and in the other case $\hat{X}_i\approx 1-X_i$. Not knowing which of the two cases is true, one approach is that the adversary can assume $\omega>0,b<0$ for the first $\frac{N}{2}$ predictions and obtain estimates accordingly, and for the second $\frac{N}{2}$ predictions, it assumes $\omega<0,b>0$ and obtain estimates accordingly. The error of the adversary is close to 0 in one of these batches of $\frac{N}{2}$ predictions. However, this approach can be outperformed as follows. The adversary assumes for the whole $N$ predictions that $\omega>0,b<0$ and obtains estimates accordingly. Afterwards, the adversary assumes $\omega<0,b>0$ for the whole $N$ predictions, and obtains a second estimate. Since MSE is a strictly convex function of the estimate, $\hat{X}_i\approx\frac{1}{2}X_i+\frac{1}{2}(1-X_i)=\frac{1}{2}$ outperforms the previous approach, which means that the aforementioned estimation is worse than the naive estimate of $\frac{1}{2}$. Weather the adversary can beat this estimate in this context is left as a problem to be consider in a later study.\footnote{In this context, one possible approach is to use the population statistics publicly available to the active party. For instance, if the unknown feature is the age of each client, the active party can use the population average as an estimate in solving $\omega\frac{\sum_iX_i}{N}+b=\frac{\sum_iv_i}{N}$.}

In conclusion, if the active party is aware of only the signs of $\omega$ and $b$, the attack has an error that vanishes with $N$. If the adversary is only aware of the sign of $\omega b$, the same result holds unless the observations $v_i$'s have different signs.
\section{Privacy-Preserving Scheme (PPS)}\label{PR}
In \cite{Xinjian} and \cite{Jiang}, several defense techniques, such as differentially-private training, processing the confidence scores, etc., have been investigated, where the model accuracy is taken as the utility in a privacy-utility trade-off. Experimental results are provided to compare different techniques. Except for the two techniques, purification and rounding, defense comes with a loss in utility, i.e., the model accuracy is degraded.

This section consists of two subsections. In the first one, we consider the problem of preserving the privacy in the most stringent scenario, i.e., without altering the confidence scores that are revealed to the active party. In the second subsection, this condition is relaxed, and we focus on privacy-preserving schemes that do not degrade the model accuracy.
\subsection{privacy preserving without changing the confidence scores}
In this subsection, the question is: Is it possible to improve the privacy of the passive party, or equivalently worsen the performance of the adversary in doing reconstruction attacks, without altering the confidence scores that the active party receives? This refers to the stringent case where the active party requires the true soft confidence scores for decision making rather than the noisy or hard ones, i.e., class labels. One motivation for this requirement is provided in the following exmple. Consider the binary classification case, in which the active party is a bank that needs to decide whether to approve a credit request or not. Assuming that this party can approve a limited number of requests, it would make sense to receive the soft confidence scores for a better decision making. In other words, if the corresponding confidence scores for two sample IDs are $(0.6,0.4)$ and $(0.9,0.1)$, where each pair refers to the probabilities corresponding to  (Approve, Disapprove) classes, the second sample ID has the priority for being approved. This ability to prioritize the samples would disappear if only a binary score is revealed to the active party. Hence, we wish to design a scheme that worsens the reconstruction attacks, while the disclosed confidence scores remain unaltered.

Before answering this question, we start with a simple example to introduce the main idea, and gradually build upon this. Consider a binary classification task with a logistic regression model. Moreover, assume that the training samples are $2$-dimensional, i.e., $\mathbf{x}_i=(x_{1,i},x_{2,i})^T, i\in[n]$ with $n$ denoting the number of elements in the training dataset $\mathcal{D}_{\textnormal{train}}$. By training the classifier, the model parameters $\mathbf{\omega}_0=(\omega_1^0,\omega_2^0)^T$ and $b_0$ are obtained such that $c_1=\sigma(\mathbf{\omega}_0^T\mathbf{x}+b_0)=\frac{1}{1+e^{-\mathbf{\omega}_0^T\mathbf{x}-b_0}}$ denotes the probability that $\mathbf{x}$ belongs to class 1, and obviously $c_2=1-c_1$.

Now, imagine that this time we train a binary logistic regression model with a new training data set $\mathcal{D}_\textnormal{train}^\textnormal{new}=\{(x_{2,i},x_{1,i},y_i)|(x_{1,i},x_{2,i},y_i)\in\mathcal{D}_\textnormal{train},i\in[n]\}$. In other words, the new training samples are a permuted version of the original ones. The new parameters are denoted by $\mathbf{\omega}_\textnormal{new}$ and $b_\textnormal{new}$. We can expect to have $(\omega_1^\textnormal{new},\omega_2^\textnormal{new})^T=(\omega_2^0,\omega_1^0)^T$ and $b_\textnormal{new}=b_0$ for the obvious reason that given an arbitrary loss function $f:\mathds{R}\to\mathds{R}$, if $(\mathbf{\omega}_0,b_0)$ is a/the minimizer of $f(\mathbf{\omega}^T\mathbf{x}+b)$ over $(\mathbf{\omega},b)$, we have that $(\mathbf{\omega}_\textnormal{new},b_\textnormal{new})$ minimizes $f(\mathbf{\omega}^T\mathbf{x}_{new}+b)$, where $\mathbf{x}_\textnormal{new}$ is the permuted version of $\mathbf{x}$, since we have the identity $\mathbf{\omega}_0^T\mathbf{x}+b_0=\mathbf{\omega}_\textnormal{new}^T\mathbf{x}_\textnormal{new}+b_\textnormal{new}$.

This permutation of the original data can be written as
\begin{equation*}
    \mathbf{x}_{i,\textnormal{new}}=\begin{bmatrix}0&1\\1&0\end{bmatrix}\mathbf{x}_i,\ \ i\in[n],
\end{equation*}
which is a special case of an invertible linear transform, in which $\mathbf{x}_\textnormal{new}=\mathbf{Hx}$ with $\mathbf{H}$ being an invertible matrix.
%\footnote{Note that if an $L^2$-norm regularization term is also added to the objective function, $\mathbf{H}$ needs to be orthonormal, i.e., $\mathbf{H}^{-1}=\mathbf{H}^T$. For simplicity of analysis, here we assume the case with no regularization.}

% With this introduction, we conclude that for a general multi-class logistic regression task we have that if the parameters of the model trained on the original data are denoted by $(\mathbf{W}_0,\mathbf{b}_0)$, we can expect the parameters of the model trained on the new linearly transformed data ($\mathbf{x}_\textnormal{new}=\mathbf{Hx},\ \textnormal{det}(\mathbf{H})\neq 0$) to be $(\mathbf{W}_\textnormal{new},\mathbf{b}_\textnormal{new})=(\mathbf{W}_0\mathbf{H}^{-1},\mathbf{b}_0)$ due to the identity
% \begin{equation*}
%   \mathbf{W}_\textnormal{new}\mathbf{x}_\textnormal{new}=\mathbf{W}_0\mathbf{H}^{-1}\mathbf{x}_\textnormal{new}=\mathbf{W}_0\mathbf{H}^{-1}\mathbf{H}\mathbf{x}=\mathbf{W}_0\mathbf{x},
% \end{equation*}
% which results in not only the preservation of the model accuracy and loss, but also the preservation of the model outputs, i.e., the confidence scores.

The above explanation, being just an introduction to the main idea, is not written rigorously. In what follows, the discussion is provided formally.

Consider the optimization in the multi-class classification logistic regression as
\begin{equation}\label{mlog}
    \min_{\mathbf{W,b}}\left\{\frac{1}{n}\sum_{i=1}^nH(\Bar{\mathbf{y}_i},\mathbf{c}_i)+\lambda[\textnormal{Tr}(\mathbf{WW}^T)+\|\mathbf{b}\|^2]\right\},
\end{equation}
in which $\Bar{\mathbf{y}_i}$ is the one-hot vector of the class label $y_i$ in $\mathcal{D}_\textnormal{train}=\{(\mathbf{x}_i,y_i)|i\in[n]\}$, $\mathbf{c}_i=\sigma(\mathbf{Wx}_i+\mathbf{b})$ is the confidence score as in (\ref{confi}), and $\lambda\geq 0$ is a hyperparameter corresponding to the regularization. Select an invertible $\mathbf{H}$, and construct $\mathcal{D}_\textnormal{train}^\textnormal{new}=\{(\mathbf{Hx}_i,y_i)|(\mathbf{x}_i,y_i)\in\mathcal{D}_\textnormal{train},i\in[n]\}$.
\begin{proposition}\label{LI}
When $\lambda=0$, i.e., no regularization, if $(\mathbf{W}_0,\mathbf{b}_0)$ is a solution of (\ref{mlog}) calculated on $\mathcal{D}_\textnormal{train}$, we have that $(\mathbf{W}_0\mathbf{H}^{-1},\mathbf{b}_0)$ is a solution of (\ref{mlog}) calculated on $\mathcal{D}_\textnormal{train}^\textnormal{new}$. When $\lambda\neq 0$, if $(\mathbf{W}_0,\mathbf{b}_0)$ denotes the solution of (\ref{mlog}) calculated on $\mathcal{D}_\textnormal{train}$, and $\mathbf{H}$ is orthonormal ($\mathbf{H}^T\mathbf{H}=\mathbf{I}$), we have that $(\mathbf{W}_0\mathbf{H}^{-1},\mathbf{b}_0)$ is the solution of (\ref{mlog}) calculated on $\mathcal{D}_\textnormal{train}^\textnormal{new}$.
\end{proposition}
\begin{proof}
When $\lambda=0$, the objective in (\ref{mlog}) is a convex function of $(\mathbf{W},\mathbf{b})$. Therefore, it has a global minimum with infinite minimizers in general. The claim is proved by noting that if the training set is $\mathcal{D}_\textnormal{train}$, and $(\mathbf{W}_0,\mathbf{b}_0)$ is one of the solutions, from the identity $\mathbf{W}_0\mathbf{H}^{-1}\mathbf{Hx}=\mathbf{W}_0\mathbf{x}$, we have that $(\mathbf{W}_0\mathbf{H}^{-1},\mathbf{b}_0)$ is also a solution of (\ref{mlog}) trained over $\mathcal{D}_\textnormal{train}^\textnormal{new}$.

When $\lambda> 0$, the objective in (\ref{mlog}) is a strictly convex function of $(\mathbf{W},\mathbf{b})$ due to the strict convexity of the regularization term, and hence, it has a unique minimizer. Denote it by $(\mathbf{W}_0,\mathbf{b}_0)$. Here, the second term in the objective of (\ref{mlog}) is also preserved if $\mathbf{H}$ is orthonormal. In other words, having $\mathbf{H}^T\mathbf{H}=\mathbf{I}$ results in
\begin{align*}
    \mathbf{W}_0\mathbf{H}^{-1}(\mathbf{W}_0\mathbf{H}^{-1})^T&=\mathbf{W}_0\mathbf{H}^{-1}\mathbf{H}^{-T}\mathbf{W}_0^T\\
    &=\mathbf{W}_0\mathbf{W}_0^T.
\end{align*}
As a result, $(\mathbf{W}_0\mathbf{H}^{-1},\mathbf{b}_0)$ is the solution of (\ref{mlog}) trained over $\mathcal{D}_\textnormal{train}^\textnormal{new}$, which not only preserves the loss and model accuracy, but also it results in the same model outputs.
\end{proof}
% This observation, which is henceforth named as \textit{linear invariance}, can be considered for neural networks, as they can be written as a general function of $\mathbf{Wx}+\mathbf{b}$ as in (\ref{NN}). Therefore, for the two models involved in this paper, i.e., LR and NN, we can linearly transform the data without a change in the model accuracy.

This \textit{linear invariance} observed in proposition \ref{LI} can be used in the design of a privacy-preserving scheme as follows. Consider the VFL discussed in this paper in the context of the white-box setting. Hence, the adversary knows $\mathbf{W}_0$ (corresponding to the passive party's model) and when the number of classes are greater than the number of passive party's features, the latter can be perfectly reconstructed by the adversary resulting in $\textnormal{MSE}=0$, i.e., the maximum privacy leakage. A privacy-preserving method that does not alter the confidence scores is proposed as follows. Select an arbitrary orthonormal matrix $\mathbf{H}_{d\times d}(\neq \mathbf{I}_d)$, and the passive party, instead of performing the training on its original training set $\mathcal{D}_{\textnormal{train}}$, trains the model on $\mathcal{D}_{\textnormal{train}}^\textnormal{new}$ where the new samples are the linear transformation (according to $\mathbf{H}$) of the original samples. Note that the task of the active party in training remains unaltered, i.e., it contributes to the training as before. In the white-box scenario, the adversary is aware of the model parameters, and again with the same assumptions, i.e., when the number of classes are greater than the number of passive party's features, the adversary can perfectly reconstruct $\mathbf{x}_\textnormal{new}(=\mathbf{Hx})$. With this scheme, the adversary's MSE has increased from 0 to $\mathds{E}[\|(\mathbf{I}-\mathbf{H})\mathbf{X}\|^2]$, which answers the question asked in the beginning of this section in the affirmative. What remains is to find an appropriate $\mathbf{H}$. To this end, we propose a heuristic scheme in the sequel.

Any orthonormal $\mathbf{H}(\neq\mathbf{I})$ results in some level of protection for the passive party's features. Therefore, the to propose the heuristic scheme, we first start with maximizing $\mathds{E}[\|(\mathbf{I}-\mathbf{H})\mathbf{X}\|^2]$ over the space of orthonormal matrices. Although the latter is not a convex set, this optimization has a simple solution.
We have
\begin{equation}\label{maxe}
    \argmax_{\substack{\mathbf{H}:\\\mathbf{H}^T\mathbf{H}=\mathbf{I}}}\mathds{E}[\|(\mathbf{I}-\mathbf{H})\mathbf{X}\|^2]=-\mathbf{I},
\end{equation}
and the proof is as follows. Denoting the correlation matrix of $\mathbf{X}$ by $\mathbf{K}_\mathbf{0}$, we can write
\begin{align}
  \max_{\substack{\mathbf{H}:\\ \mathbf{H}^T\mathbf{H}=\mathbf{I}}}\mathds{E}[\|(\mathbf{I}-\mathbf{H})\mathbf{X}\|^2]&= \max_{\substack{\mathbf{H}:\\ \mathbf{H}^T\mathbf{H}=\mathbf{I}}}\textnormal{Tr}\left((\mathbf{I}-\mathbf{H})\mathbf{K}_\mathbf{0}(\mathbf{I}-\mathbf{H})^T\right)\nonumber\\
  &= 2\textnormal{Tr}(\mathbf{K}_\mathbf{0})-2\min_{\substack{\mathbf{H}:\\ \mathbf{H}^T\mathbf{H}=\mathbf{I}}}\textnormal{Tr}(\mathbf{K}_\mathbf{0}\mathbf{H})\label{bne1}\\
  &=4\textnormal{Tr}(\mathbf{K}_\mathbf{0}),\label{bne2}
\end{align}
where in (\ref{bne1}), we have used the arguments i) $\textnormal{Tr}(\mathbf{H}\mathbf{K}_\mathbf{0}\mathbf{H}^T)=\textnormal{Tr}(\mathbf{K}_\mathbf{0}\mathbf{H}^T\mathbf{H})=\textnormal{Tr}(\mathbf{K}_\mathbf{0})$, which follows from the invariance of the trace operator under cyclic permutation and the orthonormality of $\mathbf{H}$, ii) $\textnormal{Tr}(\mathbf{A}^T\mathbf{B})=\textnormal{Tr}(\mathbf{A}\mathbf{B}^T)$ for two $m\times n$ matrices and iii) the symmetry of $\mathbf{K}_\mathbf{0}$, i.e., $\mathbf{K}_\mathbf{0}^T=\mathbf{K}_\mathbf{0}$. To show (\ref{bne2}), denote the singular values of $\mathbf{K}_\mathbf{0}$ and $\mathbf{H}$ by $\alpha_1\geq\alpha_2\geq\ldots\geq\alpha_d$ and $\beta_1\geq\beta_2\geq\ldots\geq\beta_d$, respectively. From Von Neumann's trace inequality, we have
\begin{align}
    |\textnormal{Tr}(\mathbf{K}_\mathbf{0}\mathbf{H})|\leq\sum_{i=1}^d\alpha_i\beta_i   =\sum_{i=1}^d\alpha_i
    =\textnormal{Tr}(\mathbf{K}_\mathbf{0})\label{bne3},
\end{align}
where (\ref{bne3}) follows from the following facts: i) all the singular values of an orthonormal matrix are equal to 1\footnote{This can be proved by noting that the singular values of $\mathbf{H}$ are the absolute value of the square root of the eigenvalues of $\mathbf{H}^T\mathbf{H}$ ($=\mathbf{I}$).}, and ii) since $\mathbf{K}_\mathbf{0}$ is symmetric and positive semidefinite, its singular values and eigenvalues coincide. This shows that $\mathbf{H}=-\mathbf{I}$ is a minimizer in (\ref{bne1}).

The maximization in (\ref{maxe}) is the MSE of the adversary when the number of features is lower than the number of classes. Otherwise, it would be a lower bound on the MSE since the adversary cannot reconstruct $\mathbf{HX}$ perfectly. From Theorem \ref{TH2}, the closed form solution of $\textnormal{MSE}(\hat{\mathbf{X}}_\textnormal{LS})$ is known. Although the passive party is generally unaware of what attack method the adversary employs, in what follows, we analyze the performance of $\hat{\mathbf{X}}_\textnormal{LS}$ after the application of PPS, hence named $\hat{\mathbf{X}}_\textnormal{LS}^\textnormal{PPS}$.

\begin{theorem}
We have
\begin{equation}\label{maxe2}
    \max_{\substack{\mathbf{H}:\\\mathbf{H}^T\mathbf{H}=\mathbf{I}}}\ \ \textnormal{MSE}(\hat{\mathbf{X}}_\textnormal{LS}^\textnormal{PPS})=\textnormal{Tr}((\mathbf{I}+\mathbf{A}^+\mathbf{A})\mathbf{K}_\mathbf{0})+2\|\mathbf{A}^+\mathbf{A}\mathbf{K}_\mathbf{0}\|_*,
\end{equation}
where $\|\cdot\|_*$ denotes the nuclear norm. Let $\mathbf{US}\mathbf{V}^T$ be a singular value decomposition of $\mathbf{A}^+\mathbf{A}\mathbf{K}_\mathbf{0}$. We have that $\mathbf{H}^*=-\mathbf{VU}^T$ is a maximizer in (\ref{maxe2}).
\end{theorem}
\begin{proof}
As already stated, after the application of PPS with the orthonormal matrix $\mathbf{H}$, the new parameters are $\mathbf{W}_0\mathbf{H}^{-1}$ or equivalently $\mathbf{W}_0\mathbf{H}^{T}$. As a result, the matrix $\mathbf{A}$, capturing the coefficients in the system of linear equations, changes to $\mathbf{A}\mathbf{H}^{-1}$. Therefore, we can write
\begin{equation*}
    \hat{\mathbf{X}}_\textnormal{LS}^\textnormal{PPS}=(\mathbf{A}\mathbf{H}^{-1})^+(\mathbf{A}\mathbf{H}^{-1})\mathbf{HX}.
\end{equation*}
It is known that for two matrices $\mathbf{B,C}$, if $\mathbf{C}$ has orthonormal rows, then $(\mathbf{BC})^+=\mathbf{C}^+\mathbf{B}^+$. Since $\mathbf{H}$ is orthonormal, $\mathbf{H}^{-1}$ has orthonormal rows, and being invertible, we have $(\mathbf{H}^{-1})^+=\mathbf{H}$. Therefore, we can write
\begin{equation*}
    \hat{\mathbf{X}}_\textnormal{LS}^\textnormal{PPS}=\mathbf{H}\mathbf{A}^+\mathbf{A}\mathbf{X},
\end{equation*}
which results in
\begin{align}
   \textnormal{MSE}(\hat{\mathbf{X}}_\textnormal{LS}^\textnormal{PPS})&=\mathds{E}[\|\mathbf{X}-\hat{\mathbf{X}}_\textnormal{LS}^\textnormal{PPS}\|^2] \nonumber\\
   &=\mathds{E}[\|(\mathbf{I}-\mathbf{H}\mathbf{A}^+\mathbf{A})\mathbf{X}\|^2]\nonumber\\
   &=\textnormal{Tr}((\mathbf{I}+\mathbf{A}^+\mathbf{A})\mathbf{K}_\mathbf{0})-2\textnormal{Tr}(\mathbf{H}\mathbf{A}^+\mathbf{A}\mathbf{K}_\mathbf{0})\label{kah}.
\end{align}
From (\ref{kah}), we have that maximizing $\textnormal{MSE}(\hat{\mathbf{X}}_\textnormal{LS}^\textnormal{PPS})$ is equivalent to minimizing $\textnormal{Tr}(\mathbf{H}\mathbf{A}^+\mathbf{A}\mathbf{K}_\mathbf{0})$, since the first term in (\ref{kah}) does not depend on $\mathbf{H}$. Let $\mathbf{US}\mathbf{V}^T$ be a singular value decomposition of $\mathbf{A}^+\mathbf{A}\mathbf{K}_\mathbf{0}$. As before, by applying Von Neumann's trace inequality, and noting that all the singular values of $\mathbf{H}$ are 1, we have
\begin{align*}
    \textnormal{Tr}(\mathbf{H}\mathbf{A}^+\mathbf{A}\mathbf{K}_\mathbf{0})&\geq-\sum_{i=1}^ds_i\nonumber\\
    &=-\|\mathbf{A}^+\mathbf{A}\mathbf{K}_\mathbf{0}\|_*
\end{align*}
where $s_i$'s are the singular values of $\mathbf{A}^+\mathbf{A}\mathbf{K}_\mathbf{0}$.
By replacing $\mathbf{H}$ with $-\mathbf{VU}^T$, we have
\begin{align*}
   \textnormal{Tr}(\mathbf{H}\mathbf{A}^+\mathbf{A}\mathbf{K}_\mathbf{0}) &=\textnormal{Tr}(-\mathbf{VU}^T\mathbf{US}\mathbf{V}^T)\nonumber\\
   &=-\textnormal{Tr}(\mathbf{V}\mathbf{S}\mathbf{V}^T)\nonumber\\
   &=-\textnormal{Tr}(\mathbf{S})\nonumber\\
   &=-\|\mathbf{A}^+\mathbf{A}\mathbf{K}_\mathbf{0}\|_*.
\end{align*}
Finally, by noting that $-\mathbf{VU}^T$ is orthonormal, the proof is complete.
\end{proof}
Note that when the number of features of the passive party is lower than the number of classes, we have $\mathbf{A}^+\mathbf{A}=\mathbf{I}$. Therefore, $\mathbf{A}^+\mathbf{A}\mathbf{K}_\mathbf{0}=\mathbf{K}_\mathbf{0}=\mathbf{Q}\mathbf{\Lambda}\mathbf{Q}^T$, and $\mathbf{H}^*=-\mathbf{QQ}^T=-\mathbf{I}$, which is in line with (\ref{maxe}).

In summary, if the training is with regularization, we are sure that the new model parameters are a linear transform of the original ones, which is due to the strict convexity mentioned in proposition \ref{LI}. As a result, we can worsen the adversary's performance without altering the confidence scores that are revealed. However, if the training is without regularization, there is no guarantee that the new model parameters are a linear transform of the original ones\footnote{unless we set the initial point of the solver close to the new model parameters, which is not practical as it requires to train twice.}, which is due to the possibility of having multiple solutions stated in proposition \ref{LI}. In this case, although the confidence scores no longer remain the same, the new setting has the same empirical cost. In other words, the average cross entropy between the one-hot vector of the class labels and the confidence scores does not change, which is a relaxed version of our initial lossless utility.

\begin{remark}
The main idea in this subsection was to train the VFL model on a linearly transformed data to worsen the adversary's performance and preserve the fidelity in reporting confidence scores. This process can be viewed/implemented in a different way as follows. Consider that the training phase is on the original data, and the model parameters are obtained. Let $\mathbf{W}_0$ denote the model parameters of the passive party. Instead of revealing $\mathbf{W}_0$ to the active party, a manipulated version of it, i.e., $\mathbf{W}_0\mathbf{H}^{-1}$ is disclosed. The adversary regards this as the true model parameters, and all the attacks are performed accordingly. In other words, the same preservation of privacy has been switched from training on the linearly transformed data to revealing the linearly transformed parameters. In this context, the constraints that were initially imposed on $\mathbf{H}$ can be viewed as a measure which ensures that the new revealed model parameters are not "far" from the original ones and have a one-to-one correspondence due to the invertibility of $\mathbf{H}$. This different but equivalent view of privacy enhancement can be regarded as a bridge between the white-box (revealing $\mathbf{W}_0$) and the black-box (revealing no model parameters) setting. Finally, this manipulation of the model parameters could also be done in an additive way, which is left as a problem to be investigated in future.
\end{remark}
\subsection{privacy preserving without changing the model accuracy}
In this subsection, we relax the requirement of the previous subsection, and consider privacy-preserving schemes that change the confidence scores without changing the model accuracy. We focus on adding noise to the intermediate results as follows. The confidence score that the coordinator reveals to the adversary is given by $\mathbf{c}=\sigma(\mathbf{z})$ with $\mathbf{z}=\bold{W}_{act} \mathbf{Y}+\bold{W}_{pas} \mathbf{X}+\mathbf{b}$. In order to preserve the privacy, we assume that the coordinator adds some noise to the intermediate results, i.e., $\mathbf{z}$, before the application of softmax. In other words, the new confidence scores that are revealed to the adversary are
\begin{equation}\label{anoise}
    \tilde{\mathbf{c}}=\sigma(\tilde{\mathbf{z}})=\sigma(\mathbf{z}+\mathbf{n}).
\end{equation}
Let $\mathbf{S}\triangleq\mathds{E}[\mathbf{nn}^T]$ denote the correlation matrix of $\mathbf{n}$, and $\textnormal{Tr}(\mathbf{S})$ denotes the noise budget. In what follows, we obtain the MSE of $\hat{\mathbf{X}}_\textnormal{LS}$, which sheds light on how to generate the additive noise $\mathbf{n}$.

The noisy system of linear equations that the adversary constructs is a modified version of (\ref{eqeq1}), which is obtained by replacing $\mathbf{z}$ with $\tilde{\mathbf{z}}$ as
\begin{align}
    \bold{JW}_{pas}\bold{X} &= \mathbf{J}\tilde{\mathbf{z}}-\bold{J}\bold{W}_{act}\mathbf{Y}-\mathbf{Jb}\nonumber\\
    &= \mathbf{Jz}-\bold{J}\bold{W}_{act}\mathbf{Y}-\mathbf{Jb}+\mathbf{Jn},
\end{align}
where $\mathbf{J}$ is given in (\ref{JJ}). As a result, instead of solving the correct system $\mathbf{AX}=\mathbf{b}'$, the adversary tries to solve $\mathbf{AX}=\tilde{\mathbf{b}'}(=\mathbf{b}'+\mathbf{Jn})$. Therefore, we have
\begin{equation}
    \hat{\mathbf{X}}_\textnormal{LS}=\mathbf{A}^+\tilde{\mathbf{b}'}=\mathbf{A}^+(\mathbf{AX}+\mathbf{Jn}),
\end{equation}
and
\begin{align}
    \textnormal{MSE}(\hat{\mathbf{X}}_\textnormal{LS})&=\frac{1}{d}\mathds{E}[\|\mathbf{X}-\hat{\mathbf{X}}_\textnormal{LS}\|^2]\nonumber\\
    &=\frac{1}{d}\mathds{E}[\|\mathbf{X}-\mathbf{A}^+(\mathbf{AX}+\mathbf{Jn})\|^2]\nonumber\\
    &=\frac{1}{d}\mathds{E}[\|(\mathbf{I}-\mathbf{A}^+\mathbf{A})\mathbf{X}-\mathbf{A}^+\mathbf{Jn}\|^2]\nonumber\\
    &=\frac{1}{d}\textnormal{Tr}\left((\mathbf{I}-\mathbf{A}^+\mathbf{A})\mathbf{K}_\mathbf{0}\right)+\frac{1}{d}\textnormal{Tr}(\mathbf{A}^+\mathbf{JSJ}^T{\mathbf{A}^+}^T)-\frac{2}{d}\mathds{E}[\mathbf{n}^T\mathbf{J}^T{\mathbf{A}^+}^T(\mathbf{I}-\mathbf{A}^+\mathbf{A})\mathbf{X}]\nonumber\\
    &=\frac{1}{d}\textnormal{Tr}\left((\mathbf{I}-\mathbf{A}^+\mathbf{A})\mathbf{K}_\mathbf{0}\right)+\frac{1}{d}\textnormal{Tr}(\mathbf{A}^+\mathbf{JSJ}^T{\mathbf{A}^+}^T)\label{akh0},
\end{align}
where (\ref{akh0}) follows from having ${\mathbf{A}^+}^T(\mathbf{I}-\mathbf{A}^+\mathbf{A})=\mathbf{0}$. By comparing (\ref{akh0}) to (\ref{ineq1}), we observe that the second term in (\ref{akh0}), which is non-negative\footnote{This follows the positive semidefiniteness of $\mathbf{A}^+\mathbf{JSJ}^T{\mathbf{A}^+}^T$.}, represents the performance degradation due to the receipt of noisy confidence scores. Furthermore, this performance degradation depends on the additive noise only through its correlation matrix $\mathbf{S}$.

It makes sense to maximize the MSE of the adversary subject to a limited noise budget, i.e.,
\begin{equation}\label{optz}
    \max_{\mathbf{S}:\textnormal{Tr}(\mathbf{S})\leq\alpha}\textnormal{Tr}(\mathbf{A}^+\mathbf{JSJ}^T{\mathbf{A}^+}^T)
\end{equation}
for some $\alpha\geq 0$. However, we first need to show that the objective of this optimization does not depend on a specific choice of $\mathbf{J}$. This is crucial since the coordinator is unaware how the adversary constructs the system of linear equations. We already know that $\mathbf{A}=\mathbf{JW}_{pas}$, and for simplicity, we drop the subscript in the sequel, and use $\mathbf{A}=\mathbf{JW}$ instead. The following lemma shows that replacing $\mathbf{J}$ with $\mathbf{RJ}$, in which $\mathbf{R}$ is invertible, does not change the objective in (\ref{optz}). As a result, the coordinator can assume that the system of linear equations has been obtained by $\mathbf{J}$ and perform the optimization in (\ref{optz}).
\begin{lemma}
For an invertible $\mathbf{R}_{(k-1)\times(k-1)}$, we have
\begin{equation*}
    \textnormal{Tr}\left((\mathbf{RJW})^+\mathbf{RJS}(\mathbf{RJ})^T{(\mathbf{RJW})^+}^T\right)=\textnormal{Tr}\left((\mathbf{JW})^+\mathbf{JS}\mathbf{J}^T{(\mathbf{JW})^+}^T\right).
\end{equation*}
\end{lemma}
\begin{proof}
Since $\mathbf{R}$ is invertible, it has linearly independent columns. Moreover, since $\mathbf{JW}$ has linearly independent rows, we have $(\mathbf{RJW})^+=(\mathbf{JW})^+\mathbf{R}^+=(\mathbf{JW})^+\mathbf{R}^{-1}$. Using this and the fact that $(\mathbf{AB})^T=\mathbf{B}^T\mathbf{A}^T$ concludes the proof.
\end{proof}
Let $\mathbf{A}^+\mathbf{J}=\mathbf{U\Sigma V}^T$ be a singular value decomposition in which the singular values are arranged in a non-increasing order, i.e., $\sigma_1\geq\sigma_2\geq\ldots$ .
\begin{theorem}\label{th4}
We have
\begin{equation}
    \max_{\mathbf{S}:\textnormal{Tr}(\mathbf{S})\leq\alpha}\textnormal{Tr}(\mathbf{A}^+\mathbf{JSJ}^T{\mathbf{A}^+}^T)=\sigma_1^2\alpha,
\end{equation}
where $\sigma_1$ is the maximum singular value of $\mathbf{A}^+\mathbf{J}$, and $\mathbf{S}^*=\alpha \mathbf{v}_1\mathbf{v}_1^T$, where $\mathbf{v}_1$ denotes the right singular vector corresponding to $\sigma_1$.
\end{theorem}
\begin{proof}
Denoting the eigenvalues of $\mathbf{S}$ by $\lambda_1\geq\lambda_2\geq\ldots\geq\lambda_k\geq 0$, we have
\begin{align}
    \max_{\mathbf{S}:\textnormal{Tr}(\mathbf{S})\leq\alpha}\textnormal{Tr}(\mathbf{A}^+\mathbf{JSJ}^T{\mathbf{A}^+}^T)&=\max_{\mathbf{S}:\textnormal{Tr}(\mathbf{S})\leq\alpha}\textnormal{Tr}\left(\mathbf{S}(\mathbf{A}^+\mathbf{J})^T\mathbf{A}^+\mathbf{J}\right)\nonumber\\
    &\leq\max_{\substack{\lambda_i\geq 0:\\\sum\lambda_i\leq\alpha}}\sum_{i=1}^k\sigma_i^2\lambda_i\label{fonn}\\
    &=\sigma_1^2\alpha,
\end{align}
where (\ref{fonn}) follows from the application of Von Neumann's trace inequality, and the upper bound is achieved by $\mathbf{S}^*=\alpha \mathbf{v}_1\mathbf{v}_1^T$.
\end{proof}
Thus far, the coordinator knows that the additive noise $\mathbf{n}$ in (\ref{anoise}) should have a correlation matrix equal to $\alpha \mathbf{v}_1\mathbf{v}_1^T$.
Before further analysis, we need to be aware of two points. The first one is that the MSE in (\ref{akh0}) does not assume that $\hat{\mathbf{X}}_\textnormal{LS}$ has been clamped to $[0,1]^d$. As a result, it can grow unboundedly with $\alpha$, as concluded from theorem \ref{th4}. However, in practice, the adversary can simply truncate its estimate to the region of interest, and hence, the results of theorem \ref{th4} can be used as a hint on how to add noise to the intermediate results. The second point that is worth mentioning here is that we are interested in adding noise in a way that it does not degrade the model accuracy. This requires that the same entry that is maximum in $\mathbf{z}$ should also be the maximum in $\tilde{\mathbf{z}}$.

In what follows, taking into account the aforementioned points, we propose two privacy-preserving schemes that preserve the model accuracy while worsening the adversary's performance. The first scheme is as follows. In each prediction, the coordinator finds the index of the entry that is maximum in $\mathbf{z}$, i.e.,
\begin{equation}\label{e91}
    i^*\triangleq\argmax_{i\in[k]}z_i.
\end{equation}
Denoting the elements of $\mathbf{v}_1$ by $v_i,i\in[k]$, define $\tilde{\mathbf{n}}$ element-wise as
\begin{equation}\label{o.w.}
    \tilde{n}_i\triangleq\left\{\begin{array}{ccc}
v_i&i\neq i^*\\\max_j v_j &\textnormal{o.w.}
\end{array}\right.
\end{equation}
Finally, the coordinator sets
\begin{equation}\label{e93}
    \mathbf{n}\triangleq \sqrt{\alpha}\frac{\tilde{\mathbf{n}}}{\|\tilde{\mathbf{n}}\|},
\end{equation}
and reveals $\tilde{\mathbf{c}}=\sigma(\mathbf{z}+\mathbf{n})$ to the adversary.

The second scheme is as follows. The coordinator sets $\mathbf{z}'\triangleq \mathbf{z}+\sqrt{\alpha}\mathbf{v}_1$, and reveals $\tilde{\mathbf{c}}=\sigma(\tilde{\mathbf{z}})$ to the adversary in which $\tilde{\mathbf{z}}$ is defined element-wise as\footnote{If there are more than one maximizer in $\mathbf{z}$, the aforementioned schemes undergo a slight modification: $i^*$ denotes the set of indices of the maximizers of $\mathbf{z}$, and "$\neq$" in (\ref{o.w.}) and (\ref{o.w.2}) is replaced with "$\not\in$".}
\begin{equation}\label{o.w.2}
    \tilde{z}_i\triangleq\left\{\begin{array}{ccc}
z'_i&i\neq i^*\\\max_j z'_j &\textnormal{o.w.}
\end{array}\right. .
\end{equation}
The above schemes can be viewed alternatively as follows. In each prediction, the adversary has the random error of
\begin{align*}
   \frac{1}{d}\|(\mathbf{I}-\mathbf{A}^+\mathbf{A})\mathbf{X}-\mathbf{A}^+\mathbf{Jn}\|^2&=\frac{1}{d}\|(\mathbf{I}-\mathbf{A}^+\mathbf{A})\mathbf{X}\|^2+\frac{1}{d}\|\mathbf{A}^+\mathbf{Jn}\|^2.
\end{align*}
Obviously, the maximizer of the second term over all the unit-norm $\mathbf{n}$'s is the right singular vector corresponding to the maximum singular value of $\mathbf{A}^+\mathbf{J}$, i.e., $\mathbf{v}_1$. However, in order to preserve the model accuracy, the two schemes modify the input of softmax such that $\argmax\tilde{\mathbf{c}}=\argmax\mathbf{c}$.

\section{Experimental results}\label{expr}
In this section, the performance of the proposed reconstruction attacks are evaluated on real data. These results are also compared with the previously known techniques in the literature (\cite{Jiang, Xinjian}).
% In the following, first we review the models and the datasets we have used in out numerical results. Next, we will review the baselines based on which we have compared our results. Then we illustrate our numerical results, where the attack performance of the active party using the proposed methods here is studied and compared with the baselines.

\textbf{Datasets.} We use both real-world and synthetic data for the evaluations. For the former, three widely-used public datasets (Bank, Robot and Satellite) are used for binary and multi-class classification tasks, which are obtained from the Machine Learning Repository website in \cite{MLR}. The synthetic dataset is generated via make$\_$classification in sklearn$.$dataset package. Table \ref{table_dataset} outlines the details of these datasets.
% \footnote{{\color{red}In \cite{Jiang, Xinjian}, another dataset, namely, Drive signal (with $48$ features and $11$ classes) has been used. However, LR and NN performs poorly on this dataset, whereas using Random Forest, it achieves above 99\% accuracy \cite{Tobias_Drive_Signal}. Accordingly, due to the type of dataset not being suitable for LR and NN, we have not included this dataset in out experimental results.}}
\begin{table}[ht]
\caption{Details of the datasets} % title of Table
\centering % used for centering table
\begin{tabular}{c c c c c} % centered columns (4 columns)
\hline\hline %inserts double horizontal lines
Dataset & \#Feature & \#Class & \#Records \\ [0.5ex] % inserts table
%heading
\hline % inserts single horizontal line
Bank & 19 & 2 & 41188 \\ % inserting body of the table
Robot & 24 & 4 & 5456 \\
Satellite & 36 & 6 & 6430 \\
Synthetic & 10 & 2 & 50000 \\[1ex] % [1ex] adds vertical space
\hline %inserts single line
\end{tabular}
\label{table_dataset} % is used to refer this table in the text
\end{table}

There are 20 features in the Bank dataset. As mentioned in the description file of the dataset in \cite{MLR}, the 11-th feature "highly affects the output target...[it] should be discarded if the intention is to have a realistic predictive model." Accordingly, we have eliminated this feature and the training is based on 19 features as shown in Table \ref{table_dataset}. Moreover, this dataset has 10 categorical features, which can be handled by a number of techniques with models like LR or NN. These techniques include one-hot encoding of categorical features, mapping ordinal values to each category, mapping categorical values to theirs statistics, and so on. In this paper, we have considered the latter where each category of a categorical feature is mapped to its average in each class.
% \footnote{{\color{red}In \cite{Jiang, Xinjian}, the 11-th feature has not been removed from training. Additionally, it is not clear how the categorical features have been handled for training the models.}}.

\textbf{Model.}
LR is the model considered in this paper, where each party holds their parameters corresponding to their local features. The VFL model is trained in a centralized manner, which is a reasonable assumption according to \cite{Xinjian}, since it is assumed that no intermediate information is revealed during the training phase, and only the final model is disclosed.
% {\color{red}Note that as long as the output of nonlinear activation units are separable with respect to each party's features, considering two separate NN architectures, one for each party, does not inject any training/performance loss compared to a case where parties train a model jointly. Therefore, without loss of generality, we train the model using a centralized training. }

% The NN model is composed of three hidden layers with 8 units in each layer, where the final layer is considered as sigmoid or softmax for binary or multi-class classification tasks, respectively. This model is considered in two cases: case 1 with $\text{tanh}(\cdot)$ and case 2 with $\text{Relu}(\cdot)$ as the activation units.

As in \cite{Jiang, Xinjian}, the feature values in each dataset are normalized into $[0,1]$. We note that normalizing the dataset (both the training and test data) as a whole could potentially result in an optimistic model accuracy, which is known as \textit{data snooping} and should be avoided for a very noisy dataset \cite{MacKinlay_Data_snooping}. This effect has been neglected here since the datasets under consideration are not very noisy.

Each dataset has been divided into 80\% training data and 20\% test data. This is done using train$\_$test$\_$split in the sklearn package. In training LR, we apply early stopping, and the training is done without considering any regularization, unless specified otherwise. ADAM optimization is used for training, and the codes, which are in PyTorch, are available online in our GitHub repository \cite{mrtzvrst}.

% Model accuracy and learning rates of the models under consideration are provided in table \ref{table_2}.

% \begin{table}[ht]
% \caption{Accuracy and learning rate in LR and NN models} % title of Table
% \centering % used for centering table
% \begin{tabular}{c c c c c c } % centered columns (4 columns)
% \hline\hline %inserts double horizontal lines
% Dataset & LR Acc. & LR lr & mb size \\ [0.5ex] % inserts table
% \hline % inserts single horizontal line
% Bank & & & &\\ % inserting body of the table
% Robot & & & &\\
% Satellite & & & &\\[1ex] % [1ex] adds vertical space
% \hline %inserts single line
% \end{tabular}
% \label{table_2} % is used to refer this table in the text
% \end{table}

\textbf{Baselines.}
Equation solving attack (ESA) in \cite{Xinjian} and Gradient inversion attack (GIA) in \cite{Jiang} are the baselines and briefly explained in the sequel. ESA was proposed in \cite{Xinjian} mainly for LR that is equal to $\hat{X}_\textnormal{LS}$ in this paper.

GIA is proposed in \cite{Jiang} as a model agnostic reconstruction attack, which can be applied to LR or Neural Networks (NN). The idea is to search for an estimate $\hat{\mathbf{x}}_{pas}\in[0,1]^d$ whose corresponding confidence score, denoted by $\hat{\mathbf{c}}$ is close to $\mathbf{c}$, where the closeness can be measured according to $D(\hat{\mathbf{c}}||{\mathbf{c}})$ and the optimization can be carried out by a gradient-based optimizer with zero initial values.
% We have
% \begin{align}
%     \hat{\bold{x}}_{pas} = \argmin_{\mathcal{X}_{pas}} D(\bold{c}, \hat{\bold{c}})
% \end{align}
% where $D(\cdot, \cdot)$ is a metric to measure the distance of $\bold{c}$ and $\hat{\bold{c}}$.

% Two metrics for the distance function $D(\cdot, \cdot)$ is used, namely, mean squared error (MSE)
% \begin{align}\label{eq:2}
%     D_{MSE}(\hat{\bold{c}}, \bold{c}) = \frac{1}{k}\sum_{m=1}^{m=k}(\hat{c}_m-c_m)^2,
% \end{align}
% and KL divergence\footnote{We note that in \cite{Jiang}, the metric corresponding to KL divergence is not clear whether it is $D_{KLD}(\hat{\bold{c}}, \bold{c})$ or $D_{KLD}(\bold{c}, \hat{\bold{c}})$. We have therefore, implemented both in our experiments. In the paper, we only include $D_{KLD}(\hat{\bold{c}}, \bold{c})$, however, the results for $D_{KLD}(\bold{c}, \hat{\bold{c}})$ can be found in our GitHub repository \cite{mrtzvrst}. Regardless, the corresponding results obtained from both metrics are almost on top of each other.}.
% \begin{align}\label{eq:3}
%     D_{KLD}(\hat{\bold{c}}, \bold{c}) = \frac{1}{k}\sum_{m=1}^{m=k}\hat{c}_m \log\frac{\hat{c}_m}{c_m}.
% \end{align}

% Once the active party has a confidence score $\bold{c}$, a gradient based approach is applied\footnote{Note that the attack is implemented in the context of white box, where active party has access to the model parameters.} with zero initial values for the unknown passive party features until the corresponding metric approaches its minimum.
\begin{remark}
As stated earlier, although ESA (i.e., $\hat{\mathbf{x}}_\textnormal{LS}$) belongs to the solution space $\mathcal{S}$, it does not necessarily belong to the set of feasible solutions, i.e., $\mathcal{S_F}$. As a consequence, we might have $\hat{\mathbf{x}}_{\textnormal{LS}}\not\in[0,1]^d$. On another note, since $D(\hat{\mathbf{c}}||\mathbf{c})$ is a convex function of $\hat{\mathbf{x}}$ and the optimization is restricted to $[0,1]^d$, GIA results in an estimate in the set of feasible solutions $\mathcal{S_F}$. Therefore, the performance improvement of GIA (over ESA) observed in \cite{Jiang} is mainly due to this fact. One trivial approach to improve the performance of ESA is to at least truncate/clamp it when it falls out of $[0,1]^d$, but this has not been considered in \cite{Xinjian}.
% We note that applying ESA on (\ref{eq:1}), does not by itself guarantee that the obtained estimation of passive party feature vector $\hat{\bold{x}}_{pas}$ is in the feasible region, i.e., the estimated vector is in the hypercube  $[0,1]^d$. As it will be clarified more later, this is the main reason of an optimistic performance improvement of GIA attack over ESA in \cite{Jiang}.
\end{remark}
\subsection{Evaluation of the inference attacks in the white-box setting}
The performance of inference attacks are evaluated according to the MSE per feature in (\ref{MSE}), which can be estimated empirically by $\frac{1}{Nd}\sum_{i=1}^N\|\mathbf{X}_i-\hat{\mathbf{X}}_i\|^2$ with $N$ denoting the number of predictions and $d$ denoting the number of passive party's features. We set $N=1000$, and name this averaging over $N$ as \textit{average over time}. This is to distinguish from another type of averaging, namely, \textit{average over space}, which is explained via an example as follows. Assume that the Bank data, which has 19 features, is considered. Also, consider the case that we are interested in obtaining the MSE when the active and passive parties have 14 and 5 features, respectively. Since these 19 features are not i.i.d., the MSE depends on which 5 (out of 19) features are allocated to the passive party. In order to resolve this issue, we average the MSE over some different possibilities of allocating 5 features to the passive party. More specifically, we average the MSE over a moving window of size 5 featues, i.e., MSE is obtained for 19 scenarios where the feature indices of the passive party are $[1:5], [2:6],[3:7],\ldots,\{19\}\cup[1:4]$. Afterwards, these 19 MSE's are summed and divided by 19, which denotes the MSE when $d=5$.

The performance of the following attacks are compared: $\hat{\mathbf{X}}_\textnormal{LS}$(denoted by ESA in \cite{Xinjian}), $\hat{\mathbf{X}}_{\textnormal{CLS}}$, $\hat{\mathbf{X}}_\textnormal{half}$, $\hat{\mathbf{X}}_{\textnormal{half}^*}$, $\hat{\mathbf{X}}_{\textnormal{RCC1}}$, $\hat{\mathbf{X}}_\textnormal{RCC2}$ and GIA (\cite{Jiang}). Moreover, since $\hat{\mathbf{X}}_\textnormal{LS}$ may not belong to $[0,1]^d$, we also consider "Clamped LS" that is $\hat{\mathbf{X}}_\textnormal{LS}$ clamped to $[0,1]^d$, i.e., any values lower than 0 or greater than 1 are replaced with 0 and 1, respectively. Finally, by RG (Random Guess), we are referring to the random variable generated according to the uniform distribution over $[0,1]^d$, and Zero represents the estimate $\hat{\mathbf{X}}=\mathbf{0}$.

The optimizations involved in $\hat{\mathbf{X}}_\textnormal{CLS},\hat{\mathbf{X}}_\textnormal{RCC1}$ and $\hat{\mathbf{X}}_\textnormal{RCC2}$ is carried out using the Cvxpy package in Python \cite{Boyd_cvxpy}. The maximum iteration of the problem solver in Cvxpy for all of the three algorithms is set to $100000$.

%Additionally, the optimizer is the default choice of cvxpy (this changes with respect to each problem).

%For CLS, RCC1 and RCC, empirical averaging (over many samples and different unknown feature set of size $d$) is required to smooth out the results. To that end, for a given $d$, we have averaged the MSE over 1000 data point. We also consider a moving window of size $d$ in order to average over different combinations of successive $d$ features. For more information on how the averaging is performed refer to our repository on GitHub \cite{mrtzvrst}.
\begin{figure}[ht]
 \centering % centering figure
 \scalebox{0.4} % rescale the figure by a factor of 0.8
 {\includegraphics{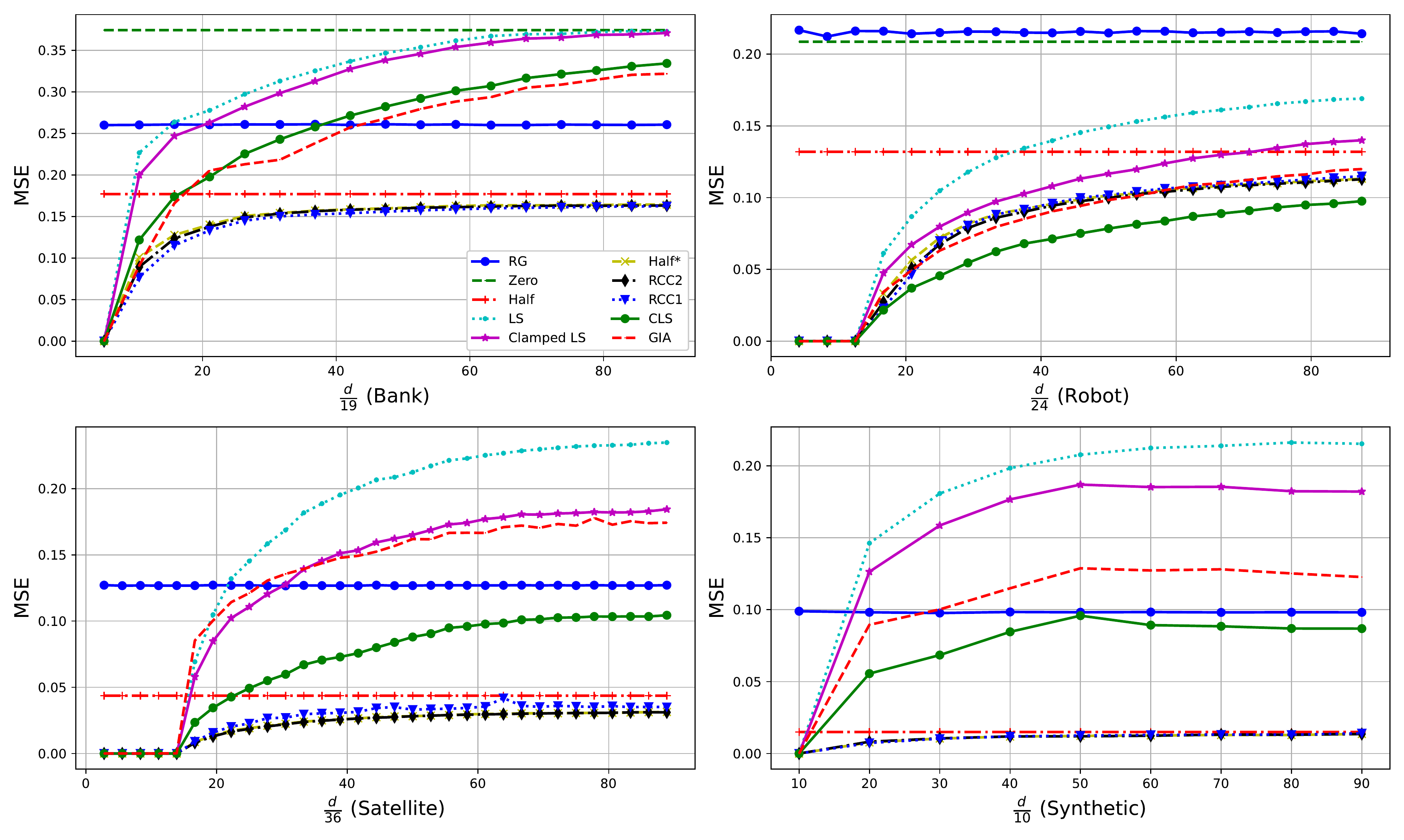}} % importing figure
 \caption{MSE per feature on Bank, Robot, Satellite and Synthetic datasets for different reconstruction attacks.}
 \label{fig1} % labeling to refer it inside the text
\end{figure}
Figure \ref{fig1} illustrates the performance of the aforementioned attacks versus the ratio of passive party's features, i.e., $\frac{d}{d_t}$. The fact that when $d<k$, we have MSE$=0$ is due to the simple fact that the system of linear equations under consideration has a unique solution, i.e., $\mathbf{A}^+\mathbf{b}$. The first observation is provided as a remark
\begin{remark}
It is observed that the scheme Half is always better than RG, which is consistent with theory (lemma \ref{lem1}), that states that any randomized guess is outperformed by its statistical mean. In Figure \ref{fig1}, the gap between the MSE corresponding to Half and RG is always $0.0833$, which is equal to $\textnormal{Var}(\hat{X})=\frac{1}{12}$ (with $\hat{X}$ being uniformly distributed on $[0,1]$) as expected from lemma \ref{lem1}. Therefore, as long as MSE is concerned, RG (used in \cite{Xinjian,Jiang}) is not a proper baseline to choose. Moreover, it is observed that in the Bank, Satellite, and synthetic datasets, the simple scheme of Half outperforms LS (ESA) and GIA for the most part. The importance of this latter observation is more emphasized by noting that Half i) does not rely on any side information (either the receipt of confidence scores or the knowledge of passive party's parameters.), and ii) does not need any optimization or matrix computation.
\end{remark}
The second observation is the confirmation of Theorem \ref{TH1}, i.e., RCC2 performs better than Half$^*$, which itself is better than Half. This rigorous guarantee that RCC2 and Half$^*$ are always better than Half is crucial, since for the attacks CLS, LS (ESA) and GIA, there is no guarantee that they perform better than the trivial blind estimate of Half, which can be verified in Figure \ref{fig1}. The third observation is that Clamped LS performs better than LS (ESA), which is discussed before. The fourth observation is that the performance of RCC1 and RCC2 are close in these datasets, while no general claim can be made as to which is better than the other. The fifth observation is that either CLS or the simple Half$^*$ outperform the results in the literature. We give the two final observations as remarks as follows.\footnote{Note that in Figure \ref{fig1}, the MSE corresponding to the estimate Zero is not shown for the Satellite and synthetic datasets due to the fact that being relatively large compared to the MSE of other attacks, it results in the condensation of other curves, and hence poor legibility.}
\begin{remark}
It is important to note that in general, the performance of the attack methods are data-dependent. In other words, one can synthesize a dataset, with an appropriate choice of the mean and variance, such that a particular attack exhibits promising performance. Nevertheless, in the context where the attacker is unaware of the underlying distribution of the passive party's features, a reasonable method is that of estimation based on the Chebyshev center of the feasible solution space, and consequently, its approximations, i.e., RCC1 and RCC2. Although, there is no claim of universal optimality of the Chebyshev center (i.e., optimal for every dataset), it is in line with the intuition that in the worst possible case, it has the best performance, i.e., optimal in the best-worst sense.
\end{remark}
\begin{remark}
Consider a VFL model with two cases, where in the first one, features with the index set $[2]$ are given to the passive party and in the second case, passive party holds features indexed by $[4]$. Fix a specific attack method. It is obvious that MSE in the second case is greater than (or equal to) the MSE in the first case. But, how do the MSE per feature compare in these cases? This depends on the underlying distribution of the data. More specifically, the MSE per feature in the second case could be greater than, equal to, or even lower than that in the first case. This is because in MSE per feature, we have a ratio whose both numerator and denominator increase with the number of features; However, whether the overall ratio is increasing or decreasing depends on the rate of increase in the numerator which depends on the data. As an example, consider a binary classification VFL with $\hat{\mathbf{X}}_\textnormal{LS}$ as the attack method. First, assume that features indexed by $[2]$ are allocated to the passive party. In this case, the adversary has a non-zero MSE per feature, which is denoted by $\zeta$. Now, imagine that this time the passive party holds features indexed by $[d]$ for some $d>2$ to be obtained later. From (\ref{ineq1}), we have that the MSE per feature is upper bounded by $\frac{1}{d}\textnormal{Tr}(\mathbf{K}_\mathbf{0})$, where $\mathbf{K}_\mathbf{0}$ denotes the correlation matrix of the $d$ features. Assume that the distribution of the data is such that we have $\mathds{E}[X_i^2]=\frac{1}{i^2}$ for $i\in[d]$. As a result, we have that $\lim_{d\to\infty}\frac{1}{d}\textnormal{Tr}(\mathbf{K}_\mathbf{0})=0$, which means that we can select a $d>2$, such that the MSE per feature becomes smaller than $\zeta$, i.e., the case with $d=2$. Hence, in general, no claim can be made on the increasing/decreasing trend of the MSE per feature.
\end{remark}

\subsection{Verification of Theorem \ref{TH2}}
In Theorem \ref{TH2} closed form expressions for the MSE of the attacks $\hat{\mathbf{X}}_{\textnormal{LS}}$ and $\hat{\mathbf{X}}_{\textnormal{half}^*}$ along with their corresponding upper and lower bounds are provided. Figure \ref{fig7} provides the verification of this Theorem on four datasets. The empirical values, denoted by LS empirical and Half$^*$ empirical, are the same curves as in the previous sebsection shown in Figure \ref{fig1}, i.e., by averaging over $N=1000$ predictions. The curves denoted by LS and Half$^*$ denote the closed form expressions in (\ref{ineq1}) and (\ref{ineq2}), respectively, where the matrices $\mathbf{K}_\mathbf{0},\mathbf{K}_{\frac{1}{2}\mathbf{1}_d}$ and $\mathbf{K}_\mathbf{\mu}$ are empirically obtained over the whole datasets (including the training and test sets). First, we observe that LS and Half$^*$ coincide with their corresponding empirical values. Moreover, we observe that the upper and lower bounds in (\ref{ineq1}) and (\ref{ineq2}), denoted by LS UB, Half$^*$ UB, LS LB, and Half$^*$ LB are valid. Note that the lower bound in (\ref{ineq3}) is denoted by LB. In particular, we observe that in the Bank dataset, the LS UB is tight when $d$ is large enough. Furthermore, it is observed that in the Satellite and synthetic datasets, the LB is tight for the attack Half$^*$, which is justified by the fact that for these datasets, all the feature values are close to $\frac{1}{2}$ (see Figure \ref{Fig_fig2}) resulting in $\mathbf{K}_{\mathbf{\mu}}\approx\mathbf{K}_{\frac{1}{2}\mathbf{1}_d}$.
% In Figure \ref{fig7}, these results are evaluated for the Bank, Satellite and Robot datasets.
% we have illustrated the upper bounds and lower bound in (\ref{ineq1}), (\ref{ineq2}) and (\ref{ineq3}) denoted as UB1, UB2 and LB, respectively. This is along with LS and Half$^*$ achievable MSE for all the four datasets. Two points are noticed. First, for datasets with number of features relatively larger than the number of classes, i.e., $k<<d_t$, the LS upper bound is tight for larger values of $d$. This is verified in Figure \ref{fig7} for the bank dataset. Similar observation is valid for Half$^*$ upper bound, although with some residual gap. Second, LB in all cases is reasonably close to the actual MSE achieved via Half$^*$ (note that LB is also valid for LS).
\begin{figure}[ht]
 \centering % centering figure
 \scalebox{0.4} % rescale the figure by a factor of 0.8
 {\includegraphics{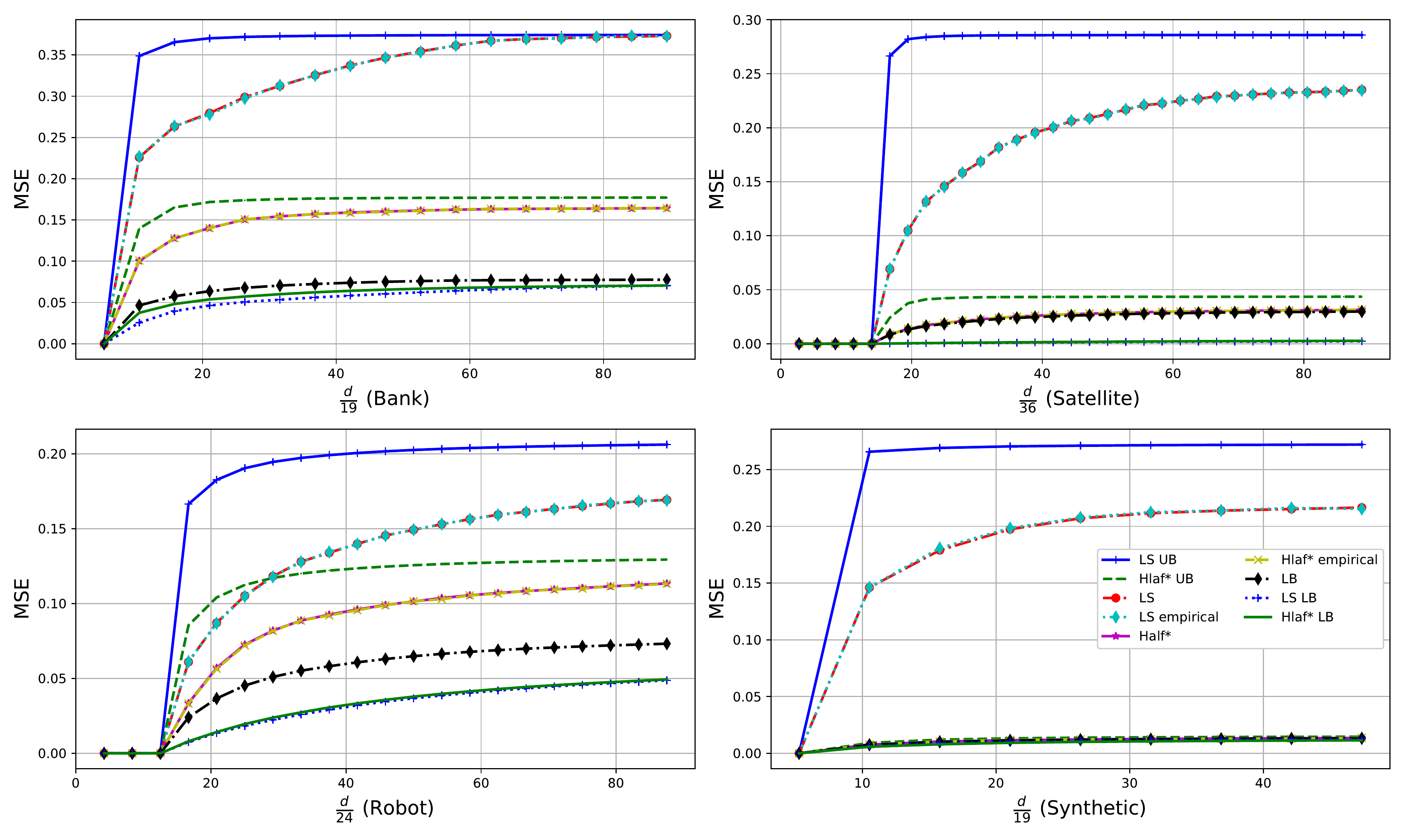}} % importing figure
 \caption{{Illustration of the results of Theorem \ref{TH2}}}
 \label{fig7} % labeling to refer it inside the text
\end{figure}

% each of the lower bounds in RHS of (\ref{ineq1}), (\ref{ineq2}) include two terms inside their min operators. In Figure \ref{fig8} we have plotted these terms for the satellite dataset (the results are similar for the other datasets, therefore removed for brevity), where UB11, UB12 and UB21, UB22 denote the  first and second terms in UB1 and UB2, respectively. It is noted from the Figure that for lower values of $d$, UB11 (UB21) are active, whereas for high values of $d$, UB12 (UB22) are active in yielding UB1 (UB2).
\begin{figure}[ht]
 \centering % centering figure
 \scalebox{0.4} % rescale the figure by a factor of 0.8
 {\includegraphics{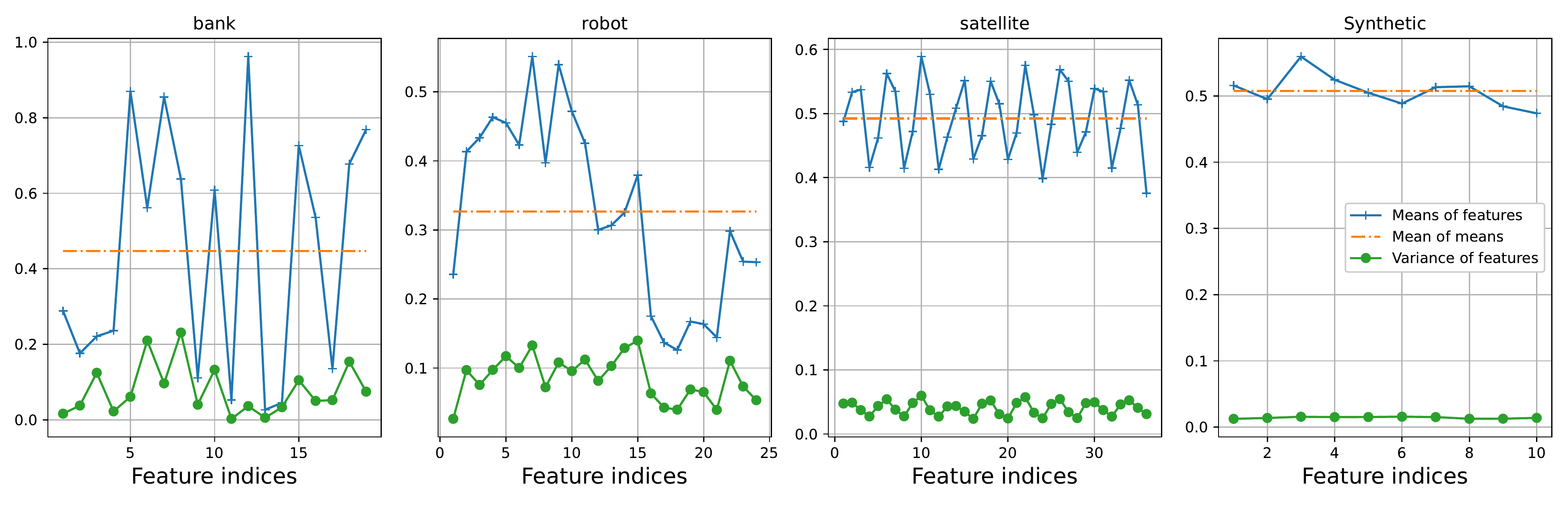}} % importing figure
 \caption{Empirical mean, variance, and mean of the means of features.}
 \label{Fig_fig2} % labeling to refer it inside the text
\end{figure}
\subsection{Evaluation of the inference attack in the black-box setting}
To evaluate the results in section \ref{bbs}, we assume that the 14th feature of the Bank data set, which denotes the employment variation rate in \cite{MLR}, is unknown. In other words, the passive party is allocated feature 14, and the active party has the remaining 18 features of this dataset. For this setting, the corresponding $\omega$ and $b$, given in section \ref{bbs}, are $\omega=0.98916, b=3.048751$, and as a result, we have $\omega b>0$. The adversary is unaware of the exact values of $\omega$ and $b$ and their signs. What it is informed of is only the fact that $\omega$ and $b$ have the same sign, and nothing else. According to the attack method in case 2 of section \ref{bbs}, the MSE converges in probability to 0 as the number of predictions $N$ grows. In Figure \ref{fig12}, the empirical expectation of $\frac{1}{N}\sum_{i=1}^N(X_i-\hat{X}_i)^2$ is plotted for $N$ ranging from 1 to 100. For each $N$, the empirical expectation is obtained with averaging over 100 instances. As it can be verified in Figure \ref{fig12}, this expectation converges to 0, which by using Markov's inequality, confirms that the empirical MSE converges in probability to 0.

\begin{figure}[ht]
 \centering % centering figure
 \scalebox{0.5} % rescale the figure by a factor of 0.8
 {\includegraphics{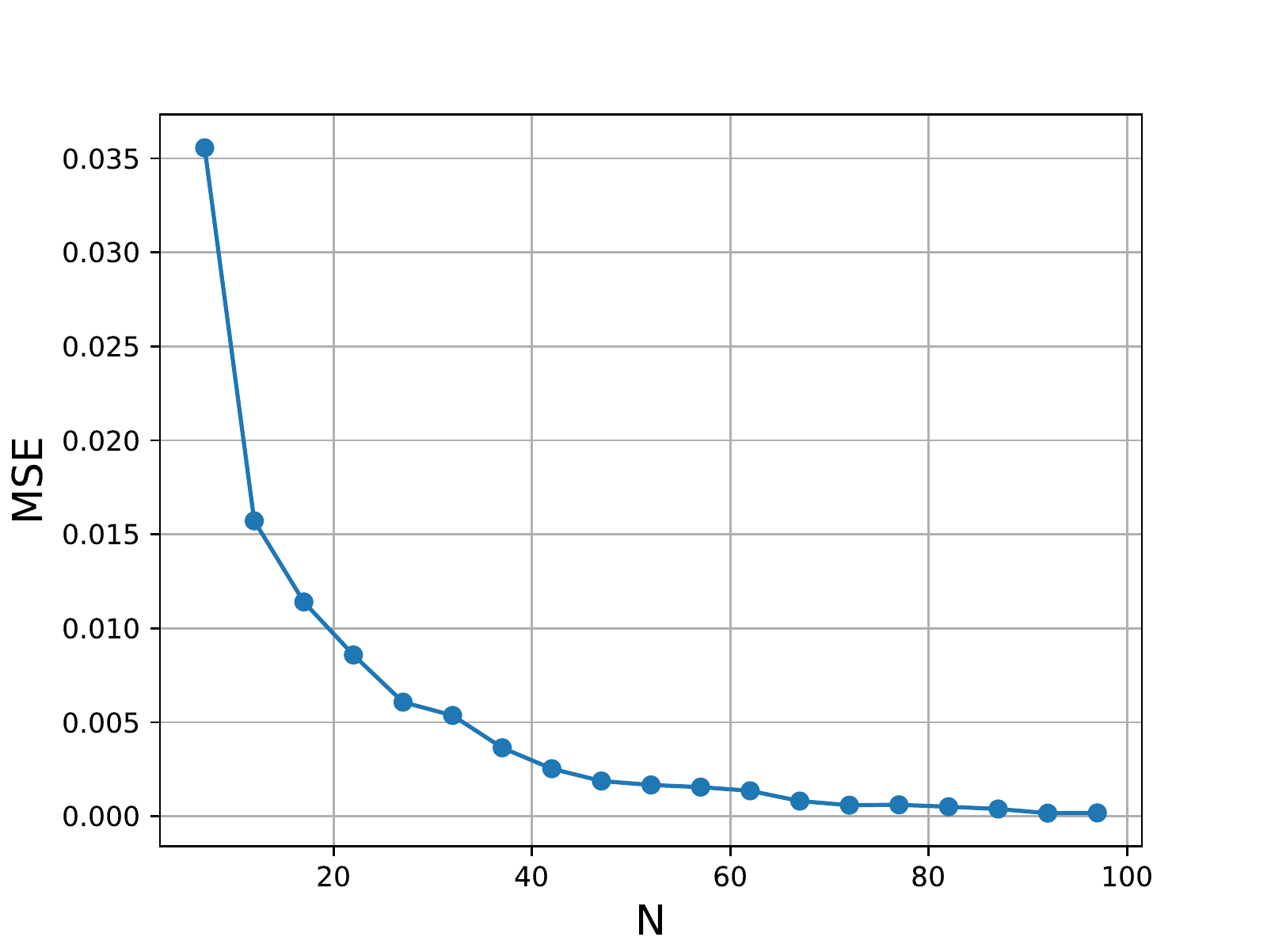}} % importing figure
 \caption{Inference attack in the black-box setting. The passive party holds one feature and $N$ denotes the number of predictions.}
 \label{fig12} % labeling to refer it inside the text
\end{figure}
\subsection{Privacy-preserving scheme}
In this section, the performance of the proposed privacy-preserving schemes in the two subsections of section \ref{PR} is evaluated on real-world and synthetic data.

To evaluate the scheme in the first subsection, first, the LR is trained (with regularization) in the context of VFL without the privacy-preserving scheme and the performance of several attacks (LS, Clamped LS, CLS, Half$^*$, RCC1 and RCC2) are obtained, as illustrated in Figures \ref{fig6pps} to \ref{fig9pps} with circle-blue lines.
\begin{figure}[ht]
 \centering % centering figure
 \scalebox{0.4} % rescale the figure by a factor of 0.8
 {\includegraphics{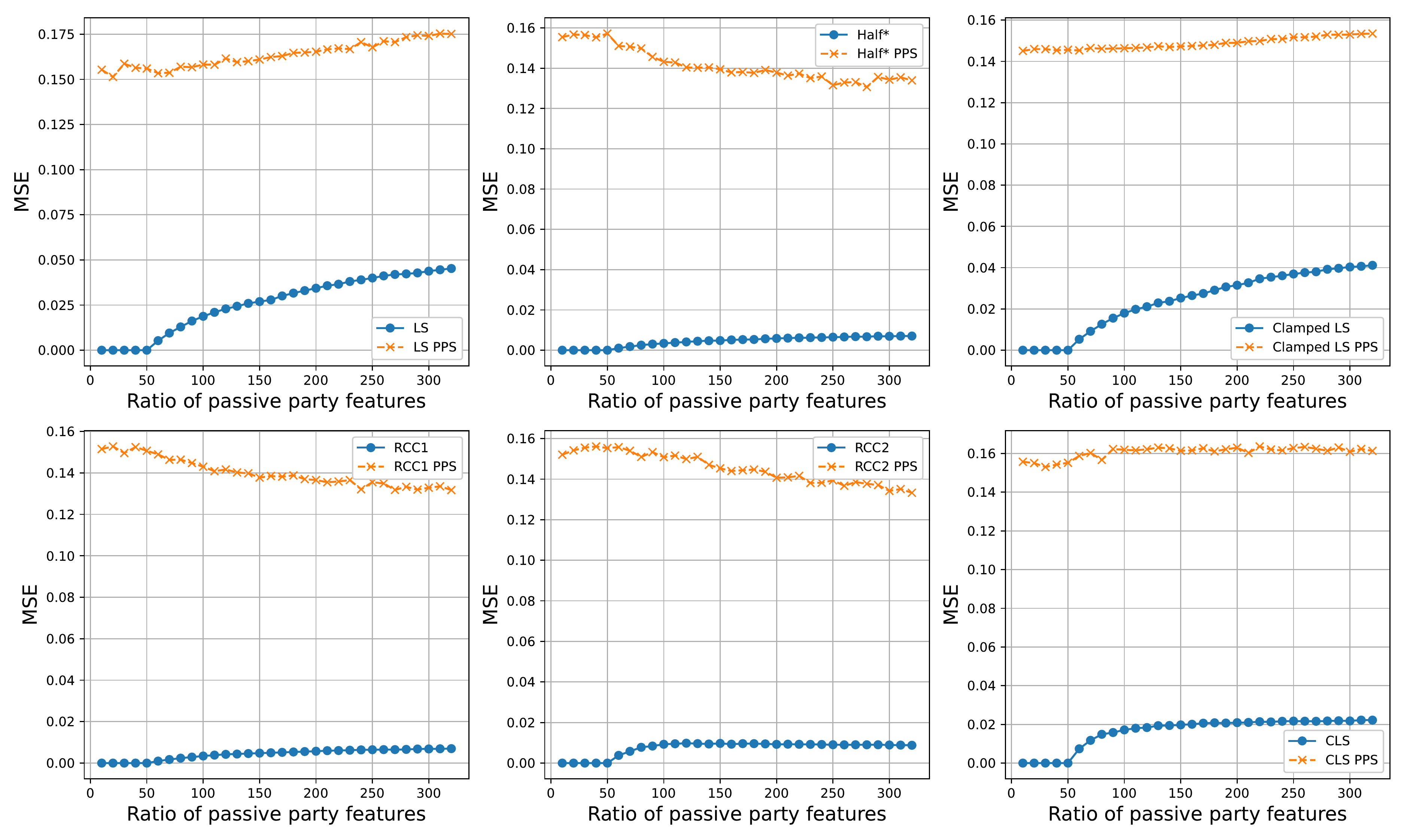}} % importing figure
 \caption{MSE per feature before and after applying PPS (Satellite).}
 \label{fig6pps} % labeling to refer it inside the text
\end{figure}
\begin{figure}[ht]
 \centering % centering figure
 \scalebox{0.4} % rescale the figure by a factor of 0.8
 {\includegraphics{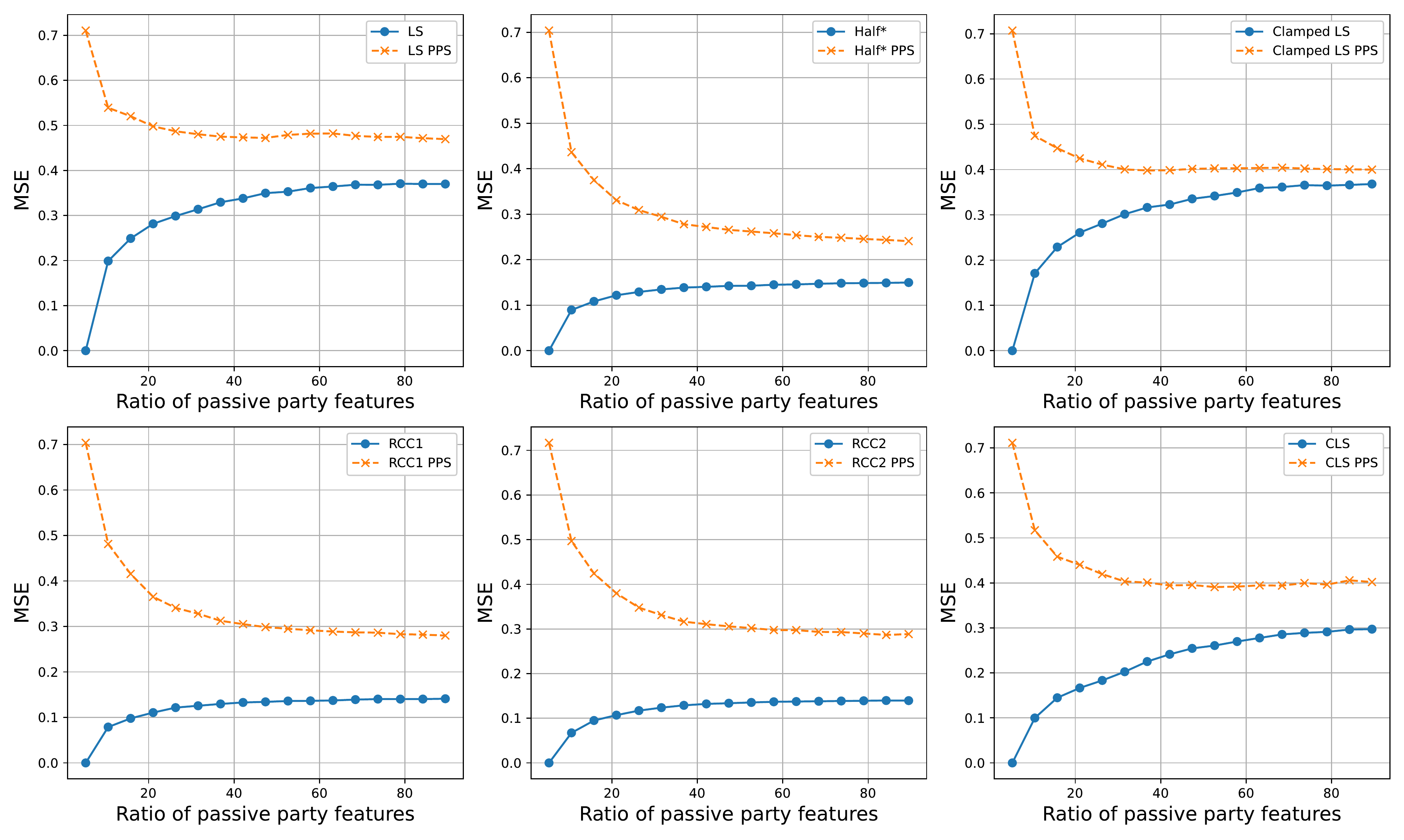}} % importing figure
 \caption{MSE per feature before and after applying PPS (Bank).}
 \label{fig7pps} % labeling to refer it inside the text
\end{figure}
\begin{figure}[ht]
 \centering % centering figure
 \scalebox{0.4} % rescale the figure by a factor of 0.8
 {\includegraphics{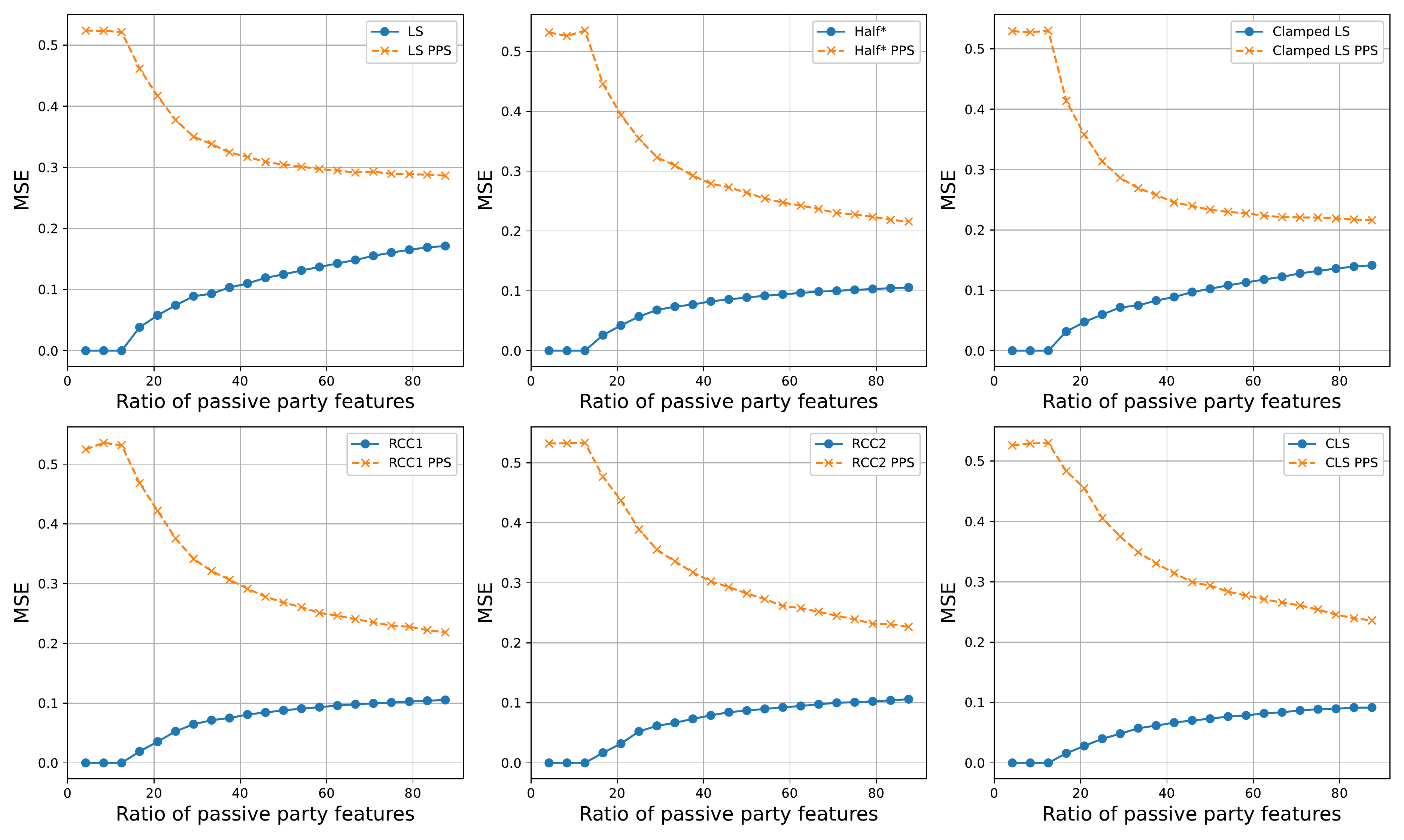}} % importing figure
 \caption{MSE per feature before and after applying PPS (Robot).}
 \label{fig8pps} % labeling to refer it inside the text
\end{figure}
\begin{figure}[ht]
 \centering % centering figure
 \scalebox{0.4} % rescale the figure by a factor of 0.8
 {\includegraphics{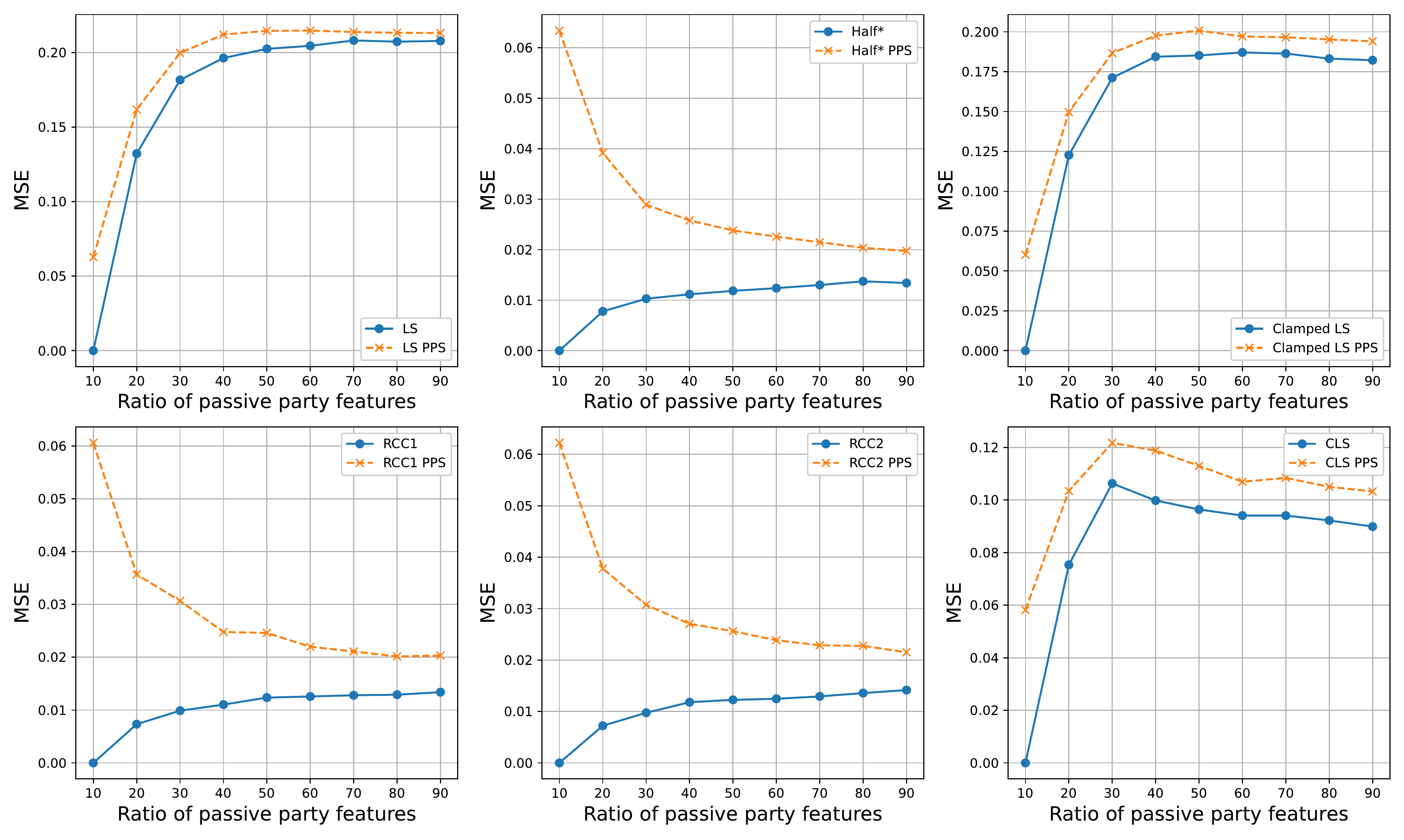}} % importing figure
 \caption{MSE per feature before and after applying PPS (Synthetic).}
 \label{fig9pps} % labeling to refer it inside the text
\end{figure}
Next, we apply the PPS technique on the datasets to obtain new training sets, which are used to train new LR models. After the model is trained, the same reconstruction attacks are applied and their performance are evaluated by the MSE per feature shown in Figures \ref{fig6pps} to \ref{fig9pps} with dashed orange lines. As it can be observed, all the reconstruction attacks have become less successful in their estimation. It is important to note that multiplying the data by $-\mathbf{I}$ and then normalizing it is equivalent to changing each feature value $x$ to $1-x$. In what follows, we consider the least square attack analytically. As given in (\ref{ineq1}), we have
\begin{equation*}
    \textnormal{MSE}(\hat{\mathbf{X}}_\textnormal{LS})=\frac{\textnormal{Tr}\left((\mathbf{I}-\mathbf{A}^+\mathbf{A})\mathbf{K}_\mathbf{0}\right)}{d}.
\end{equation*}
Let $\textnormal{MSE}(\hat{X}_{\textnormal{LS}}^\textnormal{PPS})$ denote the MSE per feature of this attack after PPS has been applied. Since we have $\mathbf{W}=\mathbf{W}_0\mathbf{H}^{-1}=-\mathbf{W_0}$, we conclude that the new matrix $\mathbf{A}$ that the adversary constructs when PPS has been applied is the negative of that in the case without PPS. As a result, we have
\begin{equation*}
    \textnormal{MSE}(\hat{X}_{\textnormal{LS}}^\textnormal{PPS})=\frac{1}{d}\mathds[\|\mathbf{X}-\mathbf{A}^+\mathbf{A}(\mathbf{1}-\mathbf{X})\|^2].
\end{equation*}
By simple calculations, we get
\begin{equation*}
    \textnormal{MSE}(\hat{X}_{\textnormal{LS}}^\textnormal{PPS})- \textnormal{MSE}(\hat{\mathbf{X}}_\textnormal{LS})=\frac{4}{d}\textnormal{Tr}(\mathbf{A}^+\mathbf{A}\mathbf{K}_{\frac{1}{2}\mathbf{1}})\geq 0.
\end{equation*}
The closer the features are to $\frac{1}{2}$, the lower $\textnormal{Tr}(\mathbf{K}_{\frac{1}{2}\mathbf{1}})$ is, and the above difference is smaller. This is consistent with the intuition since in this case $x$ and $1-x$ are not far from each other. Interestingly, we can obtain the similar performance degradation for $\textnormal{MSE}(\hat{X}_{\textnormal{half}^*}^\textnormal{PPS})$ as
\begin{equation*}
    \textnormal{MSE}(\hat{X}_{\textnormal{half}^*}^\textnormal{PPS})- \textnormal{MSE}(\hat{\mathbf{X}}_{\textnormal{half}^*})=\frac{4}{d}\textnormal{Tr}(\mathbf{A}^+\mathbf{A}\mathbf{K}_{\frac{1}{2}\mathbf{1}}).
\end{equation*}
As we know, $\mathbf{A}^+\mathbf{A}$ has $\textnormal{rank}(\mathbf{A})$ 1's and $\textnormal{nul}(\mathbf{A})$ 0's as eigenvalues. In the experimental results, the rank of $\mathbf{A}$ is $\min\{k-1,d\}$. Therefore, denoting the eigenvalues of $\mathbf{K}_{\frac{1}{2}\mathbf{1}}$ by $\gamma_1\geq\gamma_2\geq\ldots\geq\gamma_d$, by applying Von Neumann's trace inequality, we have
\begin{equation*}
    \frac{4}{d}\textnormal{Tr}(\mathbf{A}^+\mathbf{A}\mathbf{K}_{\frac{1}{2}\mathbf{1}})\leq\frac{4}{d}\sum_{i=1}^{\min\{k-1,d\}}\gamma_i.
\end{equation*}
If the data is such that the maximum eigenvalue, i.e., $\gamma_1$, scales at most sublinearly with $d$ ($\lim_{d\to\infty}\frac{\gamma_1}{d}=0$), it can be concluded that for a fixed number of classes $k$, the difference vanishes as $d$ grows\footnote{For example, if the features are independent with mean $\frac{1}{2}$, we have $\mathbf{K}_{\frac{1}{2}\mathbf{1}}$ is a diagonal matrix. Obviously, all the eigenvalues of $\mathbf{K}_{\frac{1}{2}\mathbf{1}}$ are equal to the diagonal entries, that are upper bounded by $\frac{1}{4}$.}.

To show that this PPS does not alter the confidence scores, the average KL divergence between confidence scores obtained before and after the application of PPS, averaged over $N'=20000$ samples, is illustrated in Figure \ref{fig66}. In other words, if the confidence scores of sample $i$ before and after the application of PPS are denoted by $\mathbf{c}_i$ and $\tilde{\mathbf{c}}_i$, respectively, the average KL divergence is $\frac{1}{N'}\sum_{i=1}^{N'}D(\mathbf{c}_i||\tilde{\mathbf{c}}_i)$. This value, as shown in Figure \ref{fig66}, is close to 0, which is consistent with our notion of a lossless PPS. The slight deviation from 0 is due to i) normalization of features, and ii) the step size of the gradient optimizer.

Finally, it is important to note that if the adversary is aware of $\mathbf{H}$, the privacy-preserving scheme fails. If the adversary is unaware of this transformation, still he/she could perform the attack assuming that in half of the samples, the features haven't been transformed and in the remaining ones, they have. Still, with a similar argument as in the black-box setting, it can be shown that this assumption cannot beat the scheme $\hat{\mathbf{X}}_\textnormal{half}$. There are other variants of this transform, e.g., selectively transform some features and preserve the remaining ones. Permutation can also be applied in case the features have similar alphabets, etc.

\begin{figure}[ht]
 \centering % centering figure
 \scalebox{0.4} % rescale the figure by a factor of 0.8
 {\includegraphics{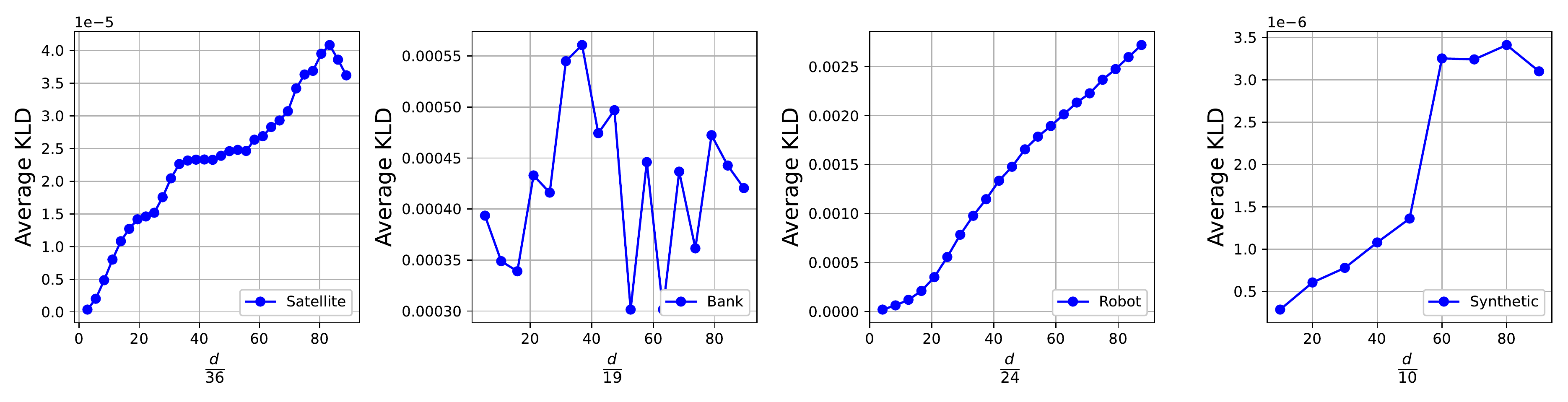}} % importing figure
 \caption{Average KL divergence between the confidence scores before and after the application of PPS.}
 \label{fig66} % labeling to refer it inside the text
\end{figure}

% to mask the passive party features and compare the achieved MSE (related to the estimation of passive party features by active party) with previously obtained MSE values when no PPS was applied. To elaborate further on how we implement the latter part, we first consider a single case. Assume active party has three unknown features, i.e., $d=3$. The passive party first does a precoding on their dataset by multiplying it by an invertible matrix. This is done as in (). Next, both parties collaborate in model training in the same way as of VFL training. Note that the invertible matrix used to precode the passive party features is not shared by CA or active party. Additionally, note that the training of the VFL model should be resumed if the set of unknown features are changed. Due to this, in our experimental results we train a new LR model once the set of unknown passive party features are changed. To smooth out the results we also do an averaging using a moving window over successive feature sets of the same size (for more details refer to \cite{mrtzvrst}). Once a LR model with the masked features are obtained, new estimation results are obtained. These results are illustrated as the upper start-orange lines in Figure \ref{fig6}. As it is observed in the Figure, for all the attack schemes, PPS deteriorates the accuracy of estimation.
% We mention here that for a scenario when only passive party feature is unknown ....
To evaluate the schemes in the second subsection of section \ref{PR}, Figure \ref{Scheme12} illustrates the privacy-preserving schemes (Schemes 1 and 2) in red and dark blue curves. The vertical axis refers to the MSE per feature for $\hat{\mathbf{X}}_{\textnormal{half}^*}$ and the horizontal axis denotes the average KL distance between the confidence scores before and after the application of these schemes. Furthermore, two other curves (denoted by "scheme 3" and "class label") are also plotted which are explained as follows. Scheme 3, shown in green, refers to the scheme in which $\tilde{\mathbf{z}}\triangleq(1-\alpha)\mathbf{z}+\alpha\mathbf{1}$ for $\alpha\in[0,1)$. Obviously, this scheme preserves the model accuracy, and by tuning $\alpha$, the trade-off in Figure \ref{Scheme12} is obtained. The curve denoted by "class label" represents the scenario where the coordinator reveals the class label. In doing so, the confidence score which is maximum is mapped to 1, and the remaining ones are mapped to 0. However, in order to obtain the trade-off in Figure \ref{Scheme12}, we let these remaining confidence scores approach 0. The smaller they get, the greater the average KL distance becomes, and in this way, the pale blue curve is produced.

\begin{figure}[ht]
 \centering % centering figure
 \scalebox{0.4} % rescale the figure by a factor of 0.8
 {\includegraphics{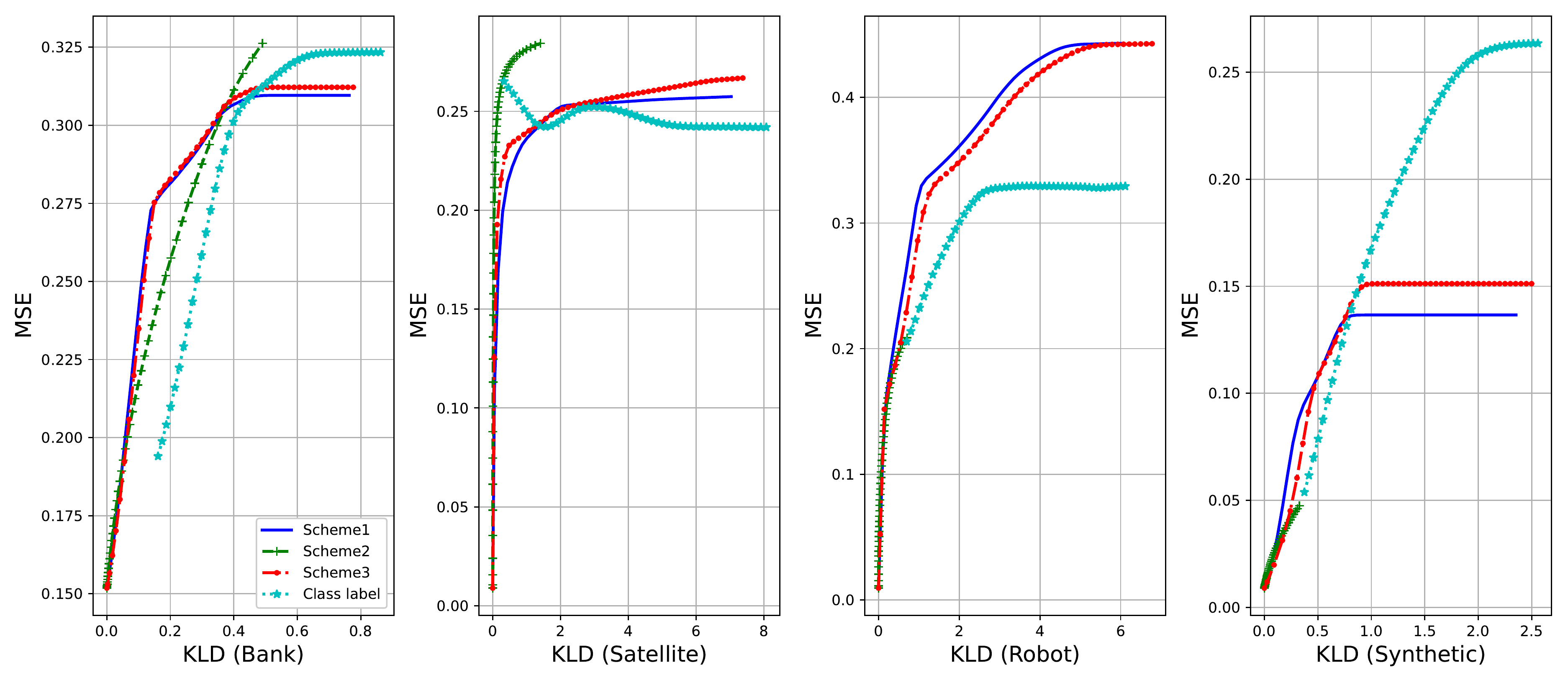}} % importing figure
 \caption{Privacy-preserving schemes that preserve the model accuracy}
 \label{Scheme12} % labeling to refer it inside the text
\end{figure}

\subsection{A note on the effect of the initial point in GIA}
As part of the experimental results, simulation of GIA was required for comparison purposes. In doing so, \cite[Algorithm 1]{Jiang} was considered, in which the initial point of the gradient-based method is set to the origin, i.e., $\mathbf{0}_d$. Knowing that the gradient-based methods start with the initial point and move in the direction opposite to the gradient of the objective function, this choice of the initial point makes sense when the features are relatively close to zero. However, when the adversary has no information about the distribution of the passive party's features, a better approach in the best-worst sense, is to take $\frac{1}{2}\mathbf{1}_d$, i.e., the Chebyshev center of $[0,1]^d$, as the initial point. This observation is shown in Figure \ref{fig3}, where the performance of GIA is illustrated for different initializations, namely, $\mathbf{0}_d$ (original in \cite{Jiang}), $\frac{1}{2}\mathbf{1}_d$ and random (uniform over $[0,1]^d$). We observe that in the Bank, Sattelite and synthetic datasets whose features have a mean close to $\frac{1}{2}$ (see Figure \ref{Fig_fig2}), initialization with $\frac{1}{2}\mathbf{1}_d$ is the best, while in the Robot dataset whose features are relatively far from $\frac{1}{2}\mathbf{1}_d$, initialization with $\frac{1}{2}\mathbf{1}_d$ still performs better for the most part. Therefore, although in the context of this paper, GIA is outperformed by the proposed attack methods, it makes sense to use the initialization $\frac{1}{2}\mathbf{1}_d$ when the adversary is blind to the underlying distribution of the passive party's features.

% Then we compare the best of GIA on each dataset with Half, Half$^*$ and RCC2. In illustrating the results for GIA approach, we have obtained both clamped and non-clamped versions. Since in all datasets, clamped versions improve upon non-clamped ones, therefore,for brevity, we have restricted the following results to the clamped ones. However, all the removed results here can be accessed in \cite{mrtzvrst}.

% In Figure \ref{fig3}, the achievable MSE (for a trained LR model) for KLD and D1 is illustrated for three different initializations, namely, zero, half and random on the three datasets. It is observed that for Bank, Robot and Satellite, respectively, half, zero and half initialization perform better than others. While this observation was not mentioned in \cite{Jiang} (their GIA approach is always initiated by zero), it reveals the importance of initial point in estimating the unknown features of passive party. Additionally, in the absence of any information on the passive party features, it seem half initialization gives a better chance to the active party in achieving lower MSE.
\begin{figure}[ht]
 \centering % centering figure
 \scalebox{0.4} % rescale the figure by a factor of 0.8
 {\includegraphics{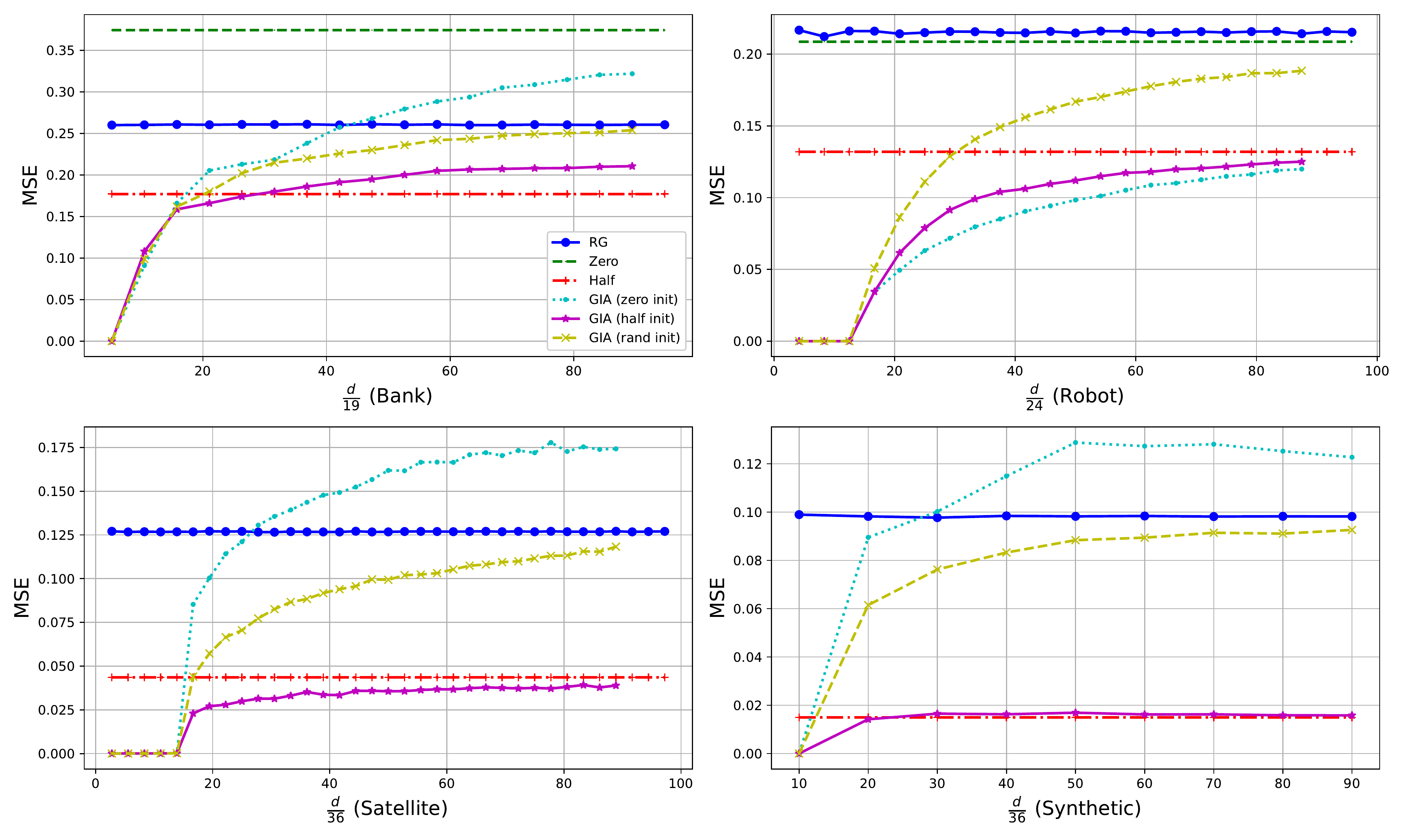}} % importing figure
 \caption{MSE per feature in GIA for different initialization.}
 \label{fig3} % labeling to refer it inside the text
\end{figure}

\section{conclusions}\label{conc}
  In this paper, a vertical federated learning setting is considered, where an active party, having access to true class labels, wishes to build a classification model by utilizing more features from a passive party, that has no access to the labels, to improve the model accuracy. The model under consideration is logistic regression. In the prediction phase, based on a classical notion of the centre of the set of feasible solutions, several inference attack methods are proposed that the active party can apply to reconstruct the sensitive information of the passive party. These attacks are shown to outperform those proposed in the literature. Moreover, several theoretical performance guarantees, in terms of upper and lower bounds, are provided for the aforementioned attacks. Subsequently, we show that in the black-box setting, knowing only the signs of the parameters involved is sufficient for the adversary to fully reconstruct the passive party's features in certain scenarios. As a defense mechanism, two privacy-preserving schemes are proposed that worsen the adversary's reconstruction attacks, while preserving the confidence scores revealed to the active party. Finally, experimental results demonstrate the effectiveness of the proposed attacks and the privacy-preserving schemes.

  Two complementary points are worth mentioning here. The first one is that in the proposed inference attacks, no knowledge about the nature of the features was used. In other words, all the features were regarded as real numbers in the interval $[0,1]$. However, in practice, features can have sparse values. For example, if a feature denotes the marital status, it has two values, and after normalizing to the range $[0,1]$, it takes either 0 or 1. Now, assume as an example that the equation under consideration is $x_1+x_2=0.5$. This equation has infinite solutions in $[0,1]^2$, while it has only one solution in $[0:1]\times[0,1]$, which is $(0,0.5)$. Hence, all the attack methods can further be improved by utilizing this knowledge of the adversary.

  We proceed with this example to mention the second point, which is the limitation of the mean square error (MSE) as a performance metric. Assume that the passive party holds features corresponding to "marital status" and "salary", and the adversary is aware of this, though it doesn't know the ordering, i.e., whether $x_1$ denotes the salary or $x_2$. Assume that in this context $x_1$ and $x_2$ denote "marital status" and "salary", respectively. Consider $N$ prediction records with true values $\mathbf{X}_1=(0,0.18)^T,\mathbf{X}_2=(1,0.66)^T,\mathbf{X}_3=(0,0.44)^T,\ldots,\mathbf{X}_N=(1,0.77)^T$. Assume that what the adversary reconstructs is a permuted version of the true values, i.e., $\hat{\mathbf{X}}_1=(0.18,0)^T,\hat{\mathbf{X}}_2=(0.66,1)^T,\hat{\mathbf{X}}_3=(0.44,0)^T,\ldots,\hat{\mathbf{X}}_N=(0.77,1)^T$. It is obvious that MSE$\neq0$, meaning that the adversary cannot perfectly reconstruct the features. However, knowing the set of features that have been estimated, upon observing the estimates, an intelligent adversary decides that the second element in each estimate refers to marital status, as it takes only 0 and 1, and the first element in each estimate refers to salary. Therefore, although MSE of this reconstruction is not zero, the adversary is able to perfectly reconstruct the features. This shows the limitation of the popular MSE metric by which the attacks are evaluated, and hence, needs further investigation, esp. when the support of the features are disparate/distinguishable.
  For instance, assuming that the alphabets of features are finite, which is the case in many practical scenarios,  the conditional entropy of the true values conditioned on the reconstructions better captures the adversary's performance in this context, i.e., here, we have $H(\mathbf{X}|\hat{\mathbf{X}})=0$.

  %The final point is that in the design of the privacy-preserving scheme, we have assumed that the adversary is unaware of the linear transform the passive party implements prior to the training. Needless to say that if the adversary is informed about the transform, i.e., matrix $\mathbf{H}$, the privacy-preserving scheme fails to protect the passive party's features, since the adversary can construct the inverse transform, i.e., $\mathbf{H}^{-1}$.
\bibliography{REFERENCE}
\bibliographystyle{IEEEtran}
\end{document}